\documentclass[11pt]{article}

\newcommand*{\ARXIV}{}%

\usepackage{amsmath,amsthm,amssymb,amsfonts,cancel}
\usepackage{thmtools,enumitem,subfigure}
\usepackage{algorithm}
\usepackage[noend]{algpseudocode}
\usepackage{mathtools}
\usepackage{dsfont}
\usepackage{tcolorbox}
\usepackage{bbm}
\usepackage{todonotes}
\usepackage{tensor}
\usepackage{DejaVuSans}
\usepackage{caption}
\usepackage{bm}
\usepackage{caption}
\usepackage{titletoc}
\usepackage{hyperref}
\usepackage[page,header]{appendix}
\usepackage{placeins}
\usepackage{wrapfig}
\usepackage{xspace}
\usepackage[capitalize]{cleveref}

\newtheorem{theorem}{Theorem}
\numberwithin{theorem}{section}

\newtheorem{lemma}[theorem]{Lemma}

\newtheorem{corollary}[theorem]{Corollary}

\makeatletter
\let\save@mathaccent\mathaccent
\newcommand*\if@single[3]{%
  \setbox0\hbox{${\mathaccent"0362{#1}}^H$}%
  \setbox2\hbox{${\mathaccent"0362{\kern0pt#1}}^H$}%
  \ifdim\ht0=\ht2 #3\else #2\fi
  }
\newcommand*\rel@kern[1]{\kern#1\dimexpr\macc@kerna}
\newcommand*\widebar[1]{\@ifnextchar^{{\wide@bar{#1}{0}}}{\wide@bar{#1}{1}}}
\newcommand*\wide@bar[2]{\if@single{#1}{\wide@bar@{#1}{#2}{1}}{\wide@bar@{#1}{#2}{2}}}
\newcommand*\wide@bar@[3]{%
  \begingroup
  \def\mathaccent##1##2{%
    \let\mathaccent\save@mathaccent
    \if#32 \let\macc@nucleus\first@char \fi
    \setbox\z@\hbox{$\macc@style{\macc@nucleus}_{}$}%
    \setbox\tw@\hbox{$\macc@style{\macc@nucleus}{}_{}$}%
    \dimen@\wd\tw@
    \advance\dimen@-\wd\z@
    \divide\dimen@ 3
    \@tempdima\wd\tw@
    \advance\@tempdima-\scriptspace
    \divide\@tempdima 10
    \advance\dimen@-\@tempdima
    \ifdim\dimen@>\z@ \dimen@0pt\fi
    \rel@kern{0.6}\kern-\dimen@
    \if#31
      \overline{\rel@kern{-0.6}\kern\dimen@\macc@nucleus\rel@kern{0.4}\kern\dimen@}%
      \advance\dimen@0.4\dimexpr\macc@kerna
      \let\final@kern#2%
      \ifdim\dimen@<\z@ \let\final@kern1\fi
      \if\final@kern1 \kern-\dimen@\fi
    \else
      \overline{\rel@kern{-0.6}\kern\dimen@#1}%
    \fi
  }%
  \macc@depth\@ne
  \let\math@bgroup\@empty \let\math@egroup\macc@set@skewchar
  \mathsurround\z@ \frozen@everymath{\mathgroup\macc@group\relax}%
  \macc@set@skewchar\relax
  \let\mathaccentV\macc@nested@a
  \if#31
    \macc@nested@a\relax111{#1}%
  \else
    \def\gobble@till@marker##1\endmarker{}%
    \futurelet\first@char\gobble@till@marker#1\endmarker
    \ifcat\noexpand\first@char A\else
      \def\first@char{}%
    \fi
    \macc@nested@a\relax111{\first@char}%
  \fi
  \endgroup
}
\makeatother

\DeclareMathAlphabet{\mathpzc}{OT1}{pzc}{m}{it}

\newcommand{\NN}{\mathbb{N}}
\newcommand{\RR}{\mathbb{R}}
\renewcommand{\S}{\mathcal{S}}
\renewcommand{\P}{\mathcal{P}}
\newcommand{\A}{\mathcal{A}}

\newcommand{\Ind}[1]{\mathds{1}_{\left[ #1 \right]}}

\newcommand{\Exp}[1]{\mathbb E \left[ #1 \right]} 

\renewcommand{\Pr}{\mathbb{P}}
\newcommand{\pfail}{\delta}
\newcommand{\relevant}{\textsc{Relevant}}
\newcommand{\Qhat}[2]{\ensuremath{\overline{\mathbf{Q}}_{#1}^{#2}}}

\newcommand{\Vhat}[2]{\overline{\mathbf{V}}_{#1}^{#2}}
\newcommand{\Vtilde}[2]{\widetilde{\mathbf{V}}_{#1}^{#2}}

\newcommand{\X}{\mathcal{X}}

\newcommand{\rbar}[2]{\overline{\mathbf{r}}_{#1}^{#2}}
\newcommand{\rhat}[2]{\hat{\mathbf{r}}_{#1}^{#2}}
\newcommand{\That}[2]{\hat{\mathbf{T}}_{#1}^{#2}}
\newcommand{\Tbar}[2]{\overline{\mathbf{T}}_{#1}^{#2}}
\newcommand{\rbonus}[2]{\textsc{Rucb}_{#1}^{#2}}
\newcommand{\tbonus}[2]{\textsc{Tucb}_{#1}^{#2}}

\mathchardef\mhyphen="2D 
\DeclareMathOperator*{\argmin}{\arg\!\min}
\DeclareMathOperator*{\argmax}{\arg\!\max}
\DeclarePairedDelimiter{\norm}{\lVert}{\rVert}

\let\originalleft\left
\let\originalright\right
\renewcommand{\left}{\mathopen{}\mathclose\bgroup\originalleft}
\renewcommand{\right}{\aftergroup\egroup\originalright}

\newtheorem{assumption}{Assumption}

\usepackage[suppress]{color-edits}
\addauthor{sb}{magenta}
\addauthor{srs}{blue}
\addauthor{cy}{purple}
\addauthor{tianyu}{red}
\addauthor{srstodo}{orange}
\addauthor{gj}{pink}

\usepackage{tikz}

\newcommand{ \B }{ \mathcal{B} } 
\newcommand{ \D }{ \mathcal{D} } 
\newcommand{ \F }{ \mathcal{F} } 
\newcommand{ \E }{ \mathbb{E} } 

\newcommand{\lev}[1]{\ensuremath{\ell(#1)}}

\newcommand{\Pkh}[1][k]{\ensuremath{\mathcal{P}^{#1}_h}}
\newcommand{\nplus}[1]{\ensuremath{n_{+}(#1)}}
\newcommand{\Expk}[2]{\mathbb{E}^{#2}\left[ #1 \right]} 

\newcommand{\frall}{\ensuremath{\,\forall\,}}

\newcommand{\AdaMB}{\textproc{AdaMB}\xspace}

\newcommand{\dyad}[1]{\ensuremath{\Box_{#1}}}

\newcommand{\gam}{\ensuremath{\gamma}}

\usepackage{hyperref}       
\usepackage{url}            
\usepackage{booktabs}       
\usepackage{amsfonts}       
\usepackage{nicefrac}       
\usepackage{microtype}      
\usepackage[margin=1in]{geometry}

\begin{document}
	\title{Adaptive Discretization for Model-Based Reinforcement Learning}
	\author{
	    Sean R. Sinclair \\
	    Cornell University \\
	    \texttt{srs429@cornell.edu}
	    \and
	    Tianyu Wang \\
	    Duke University \\
	    \texttt{tianyu@cs.duke.edu}
	    \and
        Gauri Jain \\
        Cornell University \\
        \texttt{gauri.g.jain@gmail.com}
	    \and
        Siddhartha Banerjee \\
        Cornell University \\
        \texttt{sbanerjee@cornell.edu}
	    \and
	    Christina Lee Yu \\
	    Cornell University \\
	    \texttt{cleeyu@cornell.edu}}
	\date{}
	\maketitle

	\begin{abstract}
	We introduce the technique of adaptive discretization to design an efficient model-based episodic reinforcement learning algorithm in large (potentially continuous) state-action spaces. 
	Our algorithm is based on optimistic one-step value iteration extended to maintain an adaptive discretization of the space. 
    From a theoretical perspective we provide worst-case regret bounds for our algorithm which are competitive compared to the state-of-the-art model-based algorithms. Moreover, our bounds are obtained via a modular proof technique which can potentially extend to incorporate additional structure on the problem. 
    
    From an implementation standpoint, our algorithm has much lower storage and computational requirements due to maintaining a more efficient partition of the state and action spaces.
    We illustrate this via experiments on several canonical control problems, which shows that our algorithm empirically performs significantly better than fixed discretization in terms of both faster convergence and lower memory usage. 
    Interestingly, we observe empirically that while fixed-discretization model-based algorithms vastly outperform their model-free counterparts, the two achieve comparable performance with adaptive discretization.~\footnote{The code for the experiments are available at \url{https://github.com/seanrsinclair/AdaptiveQLearning}.}
	\end{abstract}
\newpage
	\setcounter{tocdepth}{2}
	\tableofcontents
	\newpage
	
    \allowdisplaybreaks
    \section{Introduction}
\label{sec:introduction}

Reinforcement learning (RL) is a paradigm modeling an agent's interactions with an unknown environment with the goal of maximizing their cumulative reward throughout the trajectory \cite{sutton2018reinforcement}.  In online settings the dynamics of the system are unknown and the agent must learn the optimal policy only through interacting with the environment.  This requires the agent to navigate the \textit{exploration exploitation trade-off}, between exploring unseen parts of the system and exploiting historical high-reward decisions.  Most algorithms for learning the optimal policy in these online settings can be classified as either \textit{model-free} or \textit{model-based}.  Model-free algorithms construct estimates for the $Q$-function of the optimal policy, the expected sum of rewards obtained from playing a specific action and following the optimal policy thereafter, and create upper-confidence bounds on this quantity \cite{Sinclair_2019, jin_2018}.  In contrast, model-based algorithms instead estimate unknown system parameters, namely the average reward function and the dynamics of the system, and use this to learn the optimal policy based on full or one-step planning~\cite{azar2017minimax,efroni2019tight}.

\srsedit{RL has received a lot of interest in the design of algorithms for large-scale systems using parametric models and function approximation.  For example, the AlphaGo Zero algorithm that mastered Chess and Go from scratch trained their algorithm over 72 hours using 4 TPUs and 64 GPUs~\cite{silver2017mastering}.  These results show the intrinsic power of RL in learning complex control policies, but are computationally infeasible for applying algorithms to RL tasks in computing systems or operations research. The limiting factor is implementing regression oracles or gradient steps on computing hardware.  For example, RL approaches have received much interest in designing controllers for memory systems~\cite{alizadeh2010dctcp} or resource allocation in cloud-based computing~\cite{10.1145/1394608.1382172}.  Common to these examples are computation and storage limitations on the devices used for the controller, requiring algorithms to compete on three major facets: efficient learning, low computation, and low storage requirements.}

Motivated by these requirements we consider discretization techniques which map the continuous problem to a discrete one as these algorithms are based on simple primitives easy to implement in hardware (and has been tested heuristically in practice~\cite{pyeatt2001decision,7929968}).  A challenge is picking a discretization to manage the trade-off between the discretization error and the errors accumulated from solving the discrete problem. As a fixed discretization wastes computation and memory by forcing the algorithm to explore unnecessary parts of the space, we develop an adaptive discretization of the space, where the discretization is only refined on an \textit{as-needed} basis.  This approach reduces unnecessary exploration, computation, and memory by only keeping a fine-discretization across important parts of the space~\cite{Sinclair_2019}.

Adaptive discretization techniques have been successfully applied to multi-armed bandits~\cite{slivkins_2014} and model-free RL~\cite{Sinclair_2019}.  The key idea is to maintain a non-uniform partition of the space which is refined based on the density of samples.  These techniques do not, however, directly extend to model-based RL, where the main additional ingredient lies in maintaining transition probability estimates and incorporating these in decision-making. Doing so is easy in tabular RL and $\epsilon$-net based policies, as simple transition counts concentrate well enough to get good regret. This is much less straightforward when the underlying discretization changes in an online, data-dependent way.

\begin{table}[!tb]
\setlength\tabcolsep{0pt} 
\centering
\begin{tabular*}{\columnwidth}{@{\extracolsep{\fill}}rccc}
\toprule
  Algorithm  & Regret & Time Complexity & Space Complexity \\
\midrule
  \textsc{AdaMB} (Alg.~\ref{alg:brief}) $(d_\S > 2)$ & $H^{1+\frac{1}{d+1}}K^{1-\frac{1}{d+d_\S}}\;\;\;$ & $HK^{1+\frac{d_\S}{d+d_\S}}$ & 
  $HK$ \\
    \hfill $(d_\S \leq 2)$& $H^{1+\frac{1}{d+1}}K^{1-\frac{1}{d+d_\S+2}}$ & $HK^{1+\frac{d_\S}{d+d_\S + 2}}$ & 
  $HK^{1-\frac{2}{d+d_\S + 2}}$ \\
  \textsc{Adaptive Q-Learning}~\cite{Sinclair_2019} & $H^{5/2}K^{1-\frac{1}{d+2}}$ & $HK\log_d(K)$  & $HK^{1-\frac{2}{d+2}}$\\
  \textsc{Kernel UCBVI}~\cite{domingues2020regret} & $H^3\;\;\;\,K^{1-\frac{1}{2d+1}}$ & $HAK^2$ & $HK$ \\
  \textsc{Net-Based $Q$-Learning}~\cite{song2019efficient} & $H^{5/2}K^{1-\frac{1}{d+2}}$ & $HK^2$ & $HK$ \\
\midrule
  \textsc{Lower-Bounds}~\cite{slivkins_2014} & $H\;\;\;\;K^{1-\frac{1}{d+2}}$ & N/A & N/A \\
\bottomrule
\end{tabular*}
\caption{\em Comparison of our bounds with several state-of-the-art bounds for RL in continuous settings.  Here, $d$ is the covering dimension of the state-action space, $d_\S$ is the covering dimension of the state space, $H$ is the horizon of the MDP, and $K$ is the total number of episodes.  Implementing \textsc{Kernel UCBVI}~\cite{domingues2020regret} is unclear under general action spaces, so we specialize the time complexity under a finite set of actions of size $A$.  As running UCBVI with a fixed discretization is a natural approach to this problem, we include a short discussion of this algorithm in \cref{app:implementation_run_time}.  Since the results are informal, we do not include them in the table here.  We include `N/A' under the time and space complexity lower bound as there is no prior work in this domain to our knowledge.} 
\label{tab:comparison_of_bounds}
\end{table}

\ifdefined\ARXIV
    \subsection{Our Contributions}
\else
   \textbf{Our Contributions.}
\fi We design and analyze a \textit{model-based} RL algorithm, \AdaMB, that discretizes the state-action space in a data-driven way so as to minimize regret. \AdaMB requires the underlying state and action spaces to be embedded in compact metric spaces, and the reward function and transition kernel to be Lipschitz continuous with respect to this metric. This encompasses discrete and continuous state-action spaces with mild assumptions on the transition kernel, and deterministic systems with Lipschitz continuous transitions. Our algorithm only requires access to the metric, unlike prior algorithms which require access to simulation oracles~\cite{kakade2003exploration}, strong parametric assumptions~\cite{jin2019provably}, or impose additional assumptions on the action space to be computationally efficient~\cite{domingues2020regret}.

Our policy achieves near-optimal dependence of the regret on the covering dimension of the metric space when compared to other model-based algorithms.  In particular, we show that for a $H$-step MDP played over $K$ episodes, our algorithm achieves a regret bound 
\begin{align*}
        R(K) & \lesssim \begin{cases}
            H^{1+\frac{1}{d+1}}K^{\frac{d+d_\S - 1}{d+d_\S}} \quad d_S > 2 \\
            H^{1 + \frac{1}{d+1}}K^{\frac{d+d_\S + 1}{d+d_\S + 2}} \quad d_S \leq 2
                \end{cases}
\end{align*}
where $d_\S$ and $d_\A$ are the covering dimensions of the state and action space respectively, and $d = d_\S + d_\A$.
As~\cref{tab:comparison_of_bounds} illustrates, our bounds are uniformly better (in terms of dependence on $K$ and $H$, in all dimensions) than the best existing bounds for model-based RL in continuous-spaces~\cite{lakshmanan2015improved,domingues2020regret}.  In addition to having lower regret, \AdaMB is also simple and practical to implement, with low query complexity and storage requirements (see~\cref{tab:comparison_of_bounds}) compared to other model-based techniques.

To highlight this, we complement our theory with experiments comparing model-free and model-based algorithms, using both fixed and adaptive discretization. 
Our experiments show that with a fixed discretization, model-based algorithms outperform model-free ones; however, when using an adaptive partition of the space, model-based and model-free algorithms perform similarly.  
This provides an interesting contrast between practice (where model-based algorithms are thought to perform much better) and theory (where regret bounds in continuous settings are currently worse for model-based compared to model-free algorithms), and suggests more investigation is required for ranking the two approaches.

    \subsection{Related Work}
\label{sec:related_work}

There is an extensive literature on model-based reinforcement learning; below, we highlight the work which is closest to ours, but for more extensive references, see~\cite{sutton2018reinforcement} for RL, and~\cite{bubeck2012regret,aleks2019introduction} for bandits.

\medskip

\noindent \textbf{Tabular RL}:  There is a long line of research on the sample complexity and regret for RL in tabular settings.  In particular, the first asymptotically tight regret bound for tabular model-based algorithms with non-stationary dynamics of $O(H^{3/2} \sqrt{SAK})$ where $S,A$ are the size of state/action spaces respectively was established in~\cite{azar2017minimax}.
These bounds were matched (in terms of $K$) using an `asynchronous value-iteration' (or one-step planning) approach~\cite{azar2013minimax,efroni2019tight}, which is simpler to implement.  Our work extends this latter approach to continuous spaces via adaptive discretization.
More recently, analysis was extended to develop instance-dependent instead of worst-case guarantees~\cite{zanette2019tighter,simchowitz2019}
There has also been similar regret analysis for model-free algorithms~\cite{jin_2018}.

\medskip

\noindent\textbf{Parametric Algorithms}:  For RL in continuous spaces, several recent works have focused on the use of linear function approximation~\cite{jin2019provably, du2019provably, zanette2019limiting, wang2019optimism,wang2020provably,osband2014model}.  These works assume that the controller has a feature-extractor under which the process is well-approximated via a linear model. While the resulting algorithms can be computationally efficient, they incur linear loss when the underlying process does not meet their strict parametric assumptions. Other work has extended this approach to problems with bounded eluder dimension~\cite{wang2020provably,russo2013eluder}.

\medskip

\noindent\textbf{Nonparametric Algorithms}: In contrast, nonparametric algorithms only require mild local assumptions on the underlying process, most commonly, that the $Q$-function is Lipschitz continuous with respect to a given metric.  For example,~\cite{yang2019learning} and~\cite{shah2018q} consider nearest-neighbour methods for deterministic, infinite horizon discounted settings.  
Others assume access to a generative model~\cite{kakade2003exploration,henaff2019explicit}. 

The works closest to ours concerns algorithms with provable guarantees for continuous state-action settings (see also~\cref{tab:comparison_of_bounds}).
In model-free settings, tabular algorithms have been adapted to continuous state-action spaces via fixed discretization (i.e., $\epsilon$-nets)~\cite{song2019efficient}. In model-based settings, researchers have tackled continuous spaces via kernel methods, based on either a fixed discretization of the space \cite{lakshmanan2015improved}, or more recently, without resorting to discretization~\cite{domingues2020regret}. While the latter does learn a data-driven representation of the space via kernels, it requires solving a complex optimization problem at each step, and hence is efficient mainly for finite action sets (more discussion on this is in \cref{sec:main_results}).
Finally, adaptive discretization has been successfully implemented in model-free settings~\cite{Sinclair_2019, cao2020provably}, and this provides a good benchmark for our algorithm, and for comparing model-free and model-based algorithms.

\medskip

\srsedit{\noindent\textbf{Discretization Based Approaches}: Discretization-based approaches to reinforcement learning have been investigated heuristically through many different settings.  One line of work investigates adaptive basis functions, where the parameters of the functional model (e.g. neural network) are learned online while also adapting the basis functions as well~\cite{keller2006automatic,menache2005basis,whiteson2006evolutionary}.  Similar techniques are done with soft state aggregation~\cite{singh1995reinforcement}.  Most similar to our algorithm, though, are tree based partitioning rules, which store a hierarchical tree based partition of the state and action space (much like \AdaMB) which is refined over time~\cite{pyeatt2001decision,7929968}.  These were tested heuristically with various splitting rules (e.g. Gini index, etc), where instead we split based off the metric and level of uncertainty in the estimates.}

\ifdefined\ARXIV

\medskip
 
\noindent \textbf{Practical RL}: Reinforcement learning policies have enjoyed remarkable success in recent years, in particular in the context of large-scale game playing.  These results, however, mask the high underlying costs in terms of computational resources and training time that the demonstrations requires \cite{silver2017mastering,mnih2016asynchronous,mnih2013playing,silver2016mastering}.  For example, the AlphaGo Zero algorithm that mastered Chess and Go from scratch trained their algorithm over 72 hours using 4 TPUs and 64 GPUs.  These results, while highlighting the intrinsic power in reinforcement learning algorithms, are computationally infeasible for applying algorithms to RL tasks in computing systems.  As an example, RL approaches have received much interest in several of the following problems:
    \begin{itemize}
        \item \textit{Memory Management}: Many computing systems have two sources of memory; on-chip memory which is fast but limited, and off-chip memory which has low bandwidth and suffer from high latency.  Designing memory controllers for these system require a scheduling policy to adapt to changes in workload and memory reference streams, ensuring consistency in the memory, and controlling for long-term consequences of scheduling decisions \cite{alizadeh2010dctcp,alizadeh2013pfabric,chinchali2018cellular}.
        \item \textit{Online Resource Allocation}: Cloud-based clusters for high performance computing must decide how to allocate computing resources to different users or tasks with highly variable demand.  Controllers for these algorithms must make decisions online to manage the trade-offs between computation cost, server costs, and delay in job-completions.  Recent work has studied RL algorithms for such problems \cite{10.1145/1394608.1382172,lykouris2018competitive,nishtala2013scaling}.
    \end{itemize}
    Common to all of these examples are computation and storage limitations on the devices used for the controller.
    \begin{itemize}
        \item \textit{Limited Memory}: On chip memory is expensive and off-chip memory access has low-bandwidth.  As any reinforcement learning algorithm requires memory to store estimates of relevant quantities - RL algorithms for computing systems must manage their computational requirements.
        \item \textit{Power Consumption}: Many applications require low-power consumption for executing RL policies on general computing platforms. 
        \item \textit{Latency Requirements}: Many problems for computing systems (e.g. memory management) have strict latency quality of service requirements that limits reinforcement learning algorithms to execute their policy quickly.
    \end{itemize}
    
    A common technique to these problems is cerebellar model articulation controllers (CMACs) which has been used in optimizing controllers for dynamic RAM access~\cite{10.1145/1394608.1382172,lykouris2018competitive,nishtala2013scaling}.  This technique uses a random-discretizations of the space at various levels of coarseness.  Our algorithm is motivated by this approach, taking a first step towards designing efficient reinforcement learning algorithms for continuous spaces, where efficient means both low-regret, but also low storage and computation complexity (see Table~\ref{tab:comparison_of_bounds}).

\subsection{Outline of Paper}

Section~\ref{sec:preliminary} present preliminaries for the model.  Our algorithm, \AdaMB, is explained in Section~\ref{sec:algorithm} with the regret bound and proof sketch given in Section~\ref{sec:main_results} and Section~\ref{sec:main_proof} respectively.  Sections~\ref{app:concentration},~\ref{app:decomp}, and~\ref{app:lp_bound} give some of the proof details.  Lastly, Section~\ref{sec:experiments} presents numerical experiments of the algorithm.  All technical proofs and experiment details are deferred to the appendix.

\else

\fi

    \section{Preliminaries}
\label{sec:preliminary}

\ifdefined\ARXIV
    \subsection{MDP and Policies}
\else
   \textbf{MDP and Policies.}
\fi We consider an agent interacting with an underlying finite-horizon Markov Decision Processes (MDP) over $K$ sequential episodes, denoted $[K] = \{1, \ldots, K\}$.  The underlying MDP is given by a five-tuple $(\S, \A, H, T, R)$ where horizon $H$ is the number of steps (indexed $[H] = \{1,2,\ldots,H\}$) in each episode, and $(\S,\A)$ denotes the set of states and actions in each step. When needed for exposition, we use $\S_h,\A_h$ to explicitly denote state/action sets at step $h$.  When the step $h$ is clear we omit the subscript for readability.

Let $\Delta(\X)$ denote the set of probability measures over a set $\X$. 
State transitions are governed by a collection of transition kernels $T = \{T_h(\cdot \mid x,a)\}_{h \in [H], x \in \S, a \in \A}$, where $T_h(\cdot \mid x, a) \in \Delta(\S_{h+1})$ gives the distribution over states in $\S_{h+1}$ if action $a$ is taken in state $x$ at step $h$. The instantaneous rewards are bounded in $[0,1]$, and their distributions are specified by a collection of parameterized distributions $R = \{R_h\}_{h \in [H]}$, $R_h : \S_h \times \A_h \rightarrow \Delta([0, 1])$. We denote $r_h(x,a) = \mathbb{E}_{r \sim R_h(x,a)}[r]$.

A policy $\pi$ is a sequence of functions $\{ \pi_h \mid h \in [H] \}$ where each $\pi_h : \S_h \rightarrow \A_h$ is a mapping from a given state $x \in \S_h$ to an action $a \in \A_h$.
At the beginning of each episode $k$, the agent fixes a policy $\pi^k$ for the entire episode, and is given an initial (arbitrary) state $X_1^k \in \S_1$. In each step $h \in [H]$, the agent receives the state $X_h^k$, picks an action $A_h^k = \pi^k_h(X_h^k)$, receives reward $R_h^k \sim R_h(X_h^k,A_h^k)$, and transitions to a random state $X_{h+1}^k \sim T_h\left(\cdot \mid X_h^k, \pi^k_h(X_h^k)\right)$.  
This continues until the final transition to state $X_{H+1}^k$, at which point the agent chooses policy $\pi^{k+1}$ for the next episode after incorporating observed rewards and transitions in episode $k$, and the process is repeated.


\ifdefined\ARXIV
    \subsection{Value Function and Bellman Equations}
\else
   \textbf{Value Function and Bellman Equations.}
\fi  For any policy $\pi$, let $A^{\pi}_h$ denote the (random) action taken in step $h$ under $\pi$, i.e., $A^{\pi}_h = \pi_h(X_h^k)$. We define $V_h^\pi: \S \rightarrow \RR$ to denote the \emph{policy value function} at step $h$ under policy $\pi$, i.e., the expected sum of future rewards under policy $\pi$ starting from $X_h = x$ in step $h$ until the end of the episode. Formally,
\begin{align*}
    V_h^\pi(x) := \Exp{\textstyle\sum_{h'=h}^H R_{h'} ~\Big|~ X_h = x} ~~\text{ for }~~ R_{h'} \sim R_h(X_{h'},A_{h'}^{\pi}).
\end{align*}
We define the state-action value function (or $Q$-function) $Q_h^\pi : \S \times \A \rightarrow \RR$ at step $h$ as the sum of the expected rewards received after taking action $A_h = a$ at step $h$ from state $X_h = x$, and then following policy $\pi$ in all subsequent steps of the episode. Formally,
\begin{align*}
    Q_h^\pi(x,a) := r_h(x,a) + \Exp{\textstyle\sum_{h'=h+1}^H R_{h'} ~\Big|~ X_{h+1} \sim T_h\left(\cdot \mid x,a\right)} ~~\text{ for }~~ R_{h'} \sim R_{h'}(X_{h'},A_{h'}^{\pi}).
\end{align*}
Under suitable assumptions on $\S \times \A$ and reward functions~\cite{puterman_1994}, there exists an optimal policy $\pi^\star$ which gives the optimal value $V_h^\star(x) = \sup_\pi V_h^\pi(x)$ for all $x \in \S$ and $h \in [H]$. For ease of notation we denote $Q^\star = Q^{\pi^\star}$. 
The Bellman equations~\cite{puterman_1994} state that,
\begin{align}
\label{eqn:bellman_equation}
V_{h}^{\pi}(x) & =Q_{h}^{\pi}\left(x, \pi_{h}(x)\right) &\frall x\in \S\nonumber\\ 
Q_{h}^{\pi}(x, a) &= r_{h}(x, a) + \Exp{V_{h+1}^{\pi}(X_{h+1}) \mid X_h = x, A_h = a} &\frall (x,a)\in \S\times\A \\ 
V_{H+1}^{\pi}(x) & =0 &\frall x \in \S \nonumber.
\end{align}
For the optimal policy $\pi^{\star}$, it additionally holds that $V_h^\star(x) = \max_{a \in \A} Q_h^\star(x,a).$

In each episode $k\in[K]$ the agent selects a policy $\pi^k$, and is given an arbitrary starting state $X_1^k$. 
The goal is to maximize the total expected reward $\sum_{k=1}^K V_1^{\pi^k}(X_1^k)$. 
We benchmark the agent on their regret: the additive loss over all episodes the agent experiences using their policy instead of the optimal one.  In particular, the \textit{regret} $R(K)$ is defined as:
\begin{align}
\label{equation:regret}
R(K) = \textstyle\sum_{k=1}^K \left( V_1^\star(X_1^k) - V_1^{\pi^k}(X_1^k) \right).
\end{align}
Our goal is to show that the regret $R(K)$ is sublinear with respect to $K$.

\ifdefined\ARXIV
    \subsection{Metric Space and Lipschitz Assumptions}
\else
   \textbf{Metric Space and Lipschitz Assumptions.}
\fi  We assume the state space $\S$ and the action space $\A$ are each separable compact metric spaces, with metrics $\D_\S$ and $\D_\A$, and covering dimensions $d_\S$ and $d_\A$ respectively.  This imposes a metric structure $\D$ on $\S \times \A$ via the product metric, or any sub-additive metric such that \[\D((x,a), (x',a')) \leq \D_\S(x,x') + \D_\A(a,a').\]  This also ensures that the covering dimension of $\S \times \A$ is at most $d = d_\S + d_\A$.  \srsedit{We assume that the algorithm has oracle access to the metrics $\D_\S$ and $\D_\A$ through several queries, which are explained in more detail in \cref{app:implementation_run_time}.}  \srsedit{We also need that $T_h(\cdot \mid x,a)$ is Borel with respect to the metric $\D_\S$ for any $(x,a) \in \S \times \A$.}

We assume w.l.o.g. that $\S\times\A$ has diameter $1$, and we denote the diameter of $\S$ as $\D(\S) = \sup_{a\in\A,(x,y)\in\S^2}\D((x,a),(y,a)) \leq 1$.  For more information on metrics and covering dimension, see \cite{slivkins_2014,Kleinberg:2019:BEM:3338848.3299873, Sinclair_2019} for a summary.

To motivate the discretization approach, we also assume non-parametric Lipschitz structure on the transitions and rewards of the underlying process~\cite{Sinclair_2019}.
\begin{assumption}[Lipschitz Rewards and Transitions]
\label{assumption:Lipschitz}
    For every $x,x',h \in \S \times \S \times [H]$ and $a, a' \in \A \times \A$, the average reward function $r_h(x,a)$ is Lipschitz continuous with respect to $\D$, i.e.:
    \begin{align*}
        |r_h(x,a) - r_h(x', a')| & \leq L_r \D((x,a), (x',a'))
    \end{align*}
    For every $(x,x',h) \in \S \times \S \times [H]$ and $(a, a') \in \A \times \A$, the transition kernels $T_h(x' \mid x,a)$ are Lipschitz continuous in the $1$-Wasserstein metric $d_W$ with respect to $\D$, i.e.:
    \begin{align*}
    d_W(T_h(\cdot \mid x,a), T_h(\cdot \mid x',a')) & \leq L_T \D((x,a),(x',a')).
    \end{align*}
    We further assume that $Q_h^\star$ and $V_h^\star$ are also $L_V$-Lipschitz continuous for some constant $L_V$.  
\end{assumption}
\noindent See \cite{Sinclair_2019,domingues2020regret} for conditions that relate $L_V$ to $L_r$ and $L_T$.

\ifdefined\ARXIV
The next assumption is similar to previous literature for algorithms in general metric spaces \cite{Kleinberg:2019:BEM:3338848.3299873,slivkins_2014,Sinclair_2019}.  This assumes access to the similarity metrics $\D$, $\D_\S$, and $\D_\A$.  Learning the metric (or picking the metric) is important in practice, but beyond the scope of this paper \cite{wanigasekara2019nonparametric}.

\begin{assumption}
The agent has oracle access to the similarity metrics $\D$, $\D_\S$, and $\D_\A$ via several queries that are used by the algorithm.
\end{assumption}

In particular, \AdaMB (\cref{alg:brief}) requires access to several covering and packing oracles that are used throughout the algorithm.  For more details on the assumptions required and implementing the algorithm in practice, see \cref{app:implementation_run_time}.
\else

\fi

    \section{Algorithm}
\label{sec:algorithm}

\begin{algorithm*}[t]
	\begin{algorithmic}[1]
		\Procedure{AdaMB}{$\S, \A, \D, H, K, \pfail$}
			\State Initialize partitions $\P_h^0 = \S\times\A$ for $h\in[H]$, estimates $\Qhat{h}{0}(\cdot) = \Vhat{h}{k}(\cdot) = H-h+1$
			\For{each episode $k \gets 1, \ldots K$}
				\State Receive starting state $X_1^k$
				\For{each step $h \gets 1, \ldots, H$}
					\State Observe $X_h^k$ and determine $\relevant_h^k(X_h^k) = \{B\in\Pkh[k-1]|X_h^k\in B\}$
					\State Greedy selection rule: 
					pick $B_h^k = \argmax_{B \in \text{RELEVANT}_h^k(X_h^k)} \Qhat{h}{k-1}(B)$
                    \State Play action $A_h^k = \tilde{a}(B_h^k)$ associated with ball $B_h^k$; receive $R_h^k$ and transition to  $X_{h+1}^k$
					\State Update counts for $n_h^k(B_h^k), \rbar{h}{k}(B_h^k),$ and $\Tbar{h}{k}(\cdot \mid B_h^k)$
					\If{$n_h^k(B_h^k)+1 \geq \nplus{B_h^k}$}
					  \textproc{Refine Partition}$(B_h^k)$
					\EndIf 
				\EndFor
				\textproc{Compute Estimates}$(R_h^k, B_h^k)_{h=1}^H$
			\EndFor
		\EndProcedure
		
		\Procedure{Refine Partition}{$B$, $h$, $k$}
		    \State Construct $\P(B) = \{B_1, \ldots, B_{2^{d}}\}$ a $2^{-(\lev{B}+1)}$-dyadic partition of $B$
		    \State Update  $\Pkh=\Pkh[k-1]\cup\P(B) \setminus B$ 
		    \State For each $B_i$, initialize $n_h^k(B_i)=n_h^k(B)$, $\rbar{h}{k}(B_i) = \rbar{h}{k}(B)$ and $\Tbar{h}{k}(\cdot \mid B_i) \sim \Tbar{h}{k}(\cdot \mid B)$
		 \EndProcedure
		\Procedure{Compute Estimates}{$(B_h^k,R_h^k,X_{h+1}^k)_{h=1}^H$}
		\For{each $h \gets 1, \ldots H$ and $B \in \P_h^k$}
		 : Update $\Qhat{h}{k}(B)$ and $\Vhat{h}{k}(\cdot)$ via~\cref{eq:q_update} and~\cref{eq:def-V}
		\EndFor
		\EndProcedure
	\end{algorithmic}
	\caption{Model-Based Reinforcement Learning with Adaptive Partitioning (\AdaMB)}
	\label{alg:brief}
\end{algorithm*}

We now present our \emph{\textbf{M}odel-\textbf{B}ased RL with \textbf{Ada}ptive Partitioning} algorithm, which we refer to as \AdaMB. At a high level, \AdaMB maintains an adaptive partition of $\S_h\times\A_h$ for each step $h$, and uses \emph{optimistic value-iteration} over this partition. 
It takes as input the number of episodes $K$, metric $\D$ over $\S\times\A$, and the Lipschitz constants. 
It maintains \emph{optimistic estimates} for $r_h(x,a)$ and $T_h(\cdot \mid x,a)$ (i.e. high-probability uniform upper bounds $\frall h,x,a$). 
These are used for performing value iteration to obtain optimistic estimates $\Qhat{h}{}$ and $\Vhat{h}{}$ via one-step updates in \cref{eq:q_update} and \cref{eq:def-V}.  For full pseudocode of the algorithm, and a discussion on implementation details, see \cref{app:full_algo}.

\medskip

\textbf{Adaptive State-Action Partitioning}: For each step $h\in[H]$, \AdaMB maintains a partition of the space $\S_h\times\A_h$ into a collection of `balls' which is refined over episodes $k \in [K]$.  We denote  $\P_h^k$ to be the partition for step $h$ at the end of episode $k$; the initial partition is set as $\P_h^0 = \S\times\A \, \frall h\in[H]$.  
Each element $B \in \P_h^k$ is a ball of the form $B = \S(B)\times\A(B)$, where $\S(B) \subset \S$ (respectively $\A(B) \subset \A$) is the projection of ball $B$ onto its corresponding state (action) space. We let $(\tilde{x}(B), \tilde{a}(B))$ be the center of $B$ and denote $\D(B)=\sup \{\D((x,a),(y,b)) \mid (x,a),(y,b)\in B\}$ to be the diameter of a ball $B$.  The partition $\P_h^k$ can also be represented as a tree, with leaf nodes representing \emph{active} balls, and inactive \emph{parent} balls of $B\in\Pkh$ corresponding to $\{B'\in\Pkh[k']|B'\supset B, k'<k\}$; moreover, $\lev{B}$ is the depth of $B$ in the tree (with the root at level $0$). See Figure~\ref{fig:partition} for an example partition and tree generated by the algorithm.  Let 
\begin{align}
\S(\P_h^k) := \bigcup_{B \in \P_h^k ~\text{s.t.}~ \nexists B' \in \P_h^k, \S(B') \subset \S(B)} \S(B) \label{eq:induced_state_partition_def}
\end{align}
denote the partition over the state spaced induce by the current state-action partition $\P_h^k$. We can verify that the above constructed $\S(\P_h^k)$ is indeed a partition of $\S$ because the partition $\P_h^k$ is constructed according to a dyadic partitioning.

\begin{figure*}[!t]
\centering
\begin{minipage}{0.64\columnwidth}
  \centering
  \resizebox{\columnwidth}{!}{\input{parts/partitionexampledetailed}}%
  \caption*{Illustrating the state-action partitioning scheme}
\end{minipage}%
\hfill
\begin{minipage}{0.36\columnwidth}
\centering
  \includegraphics[scale=.2]{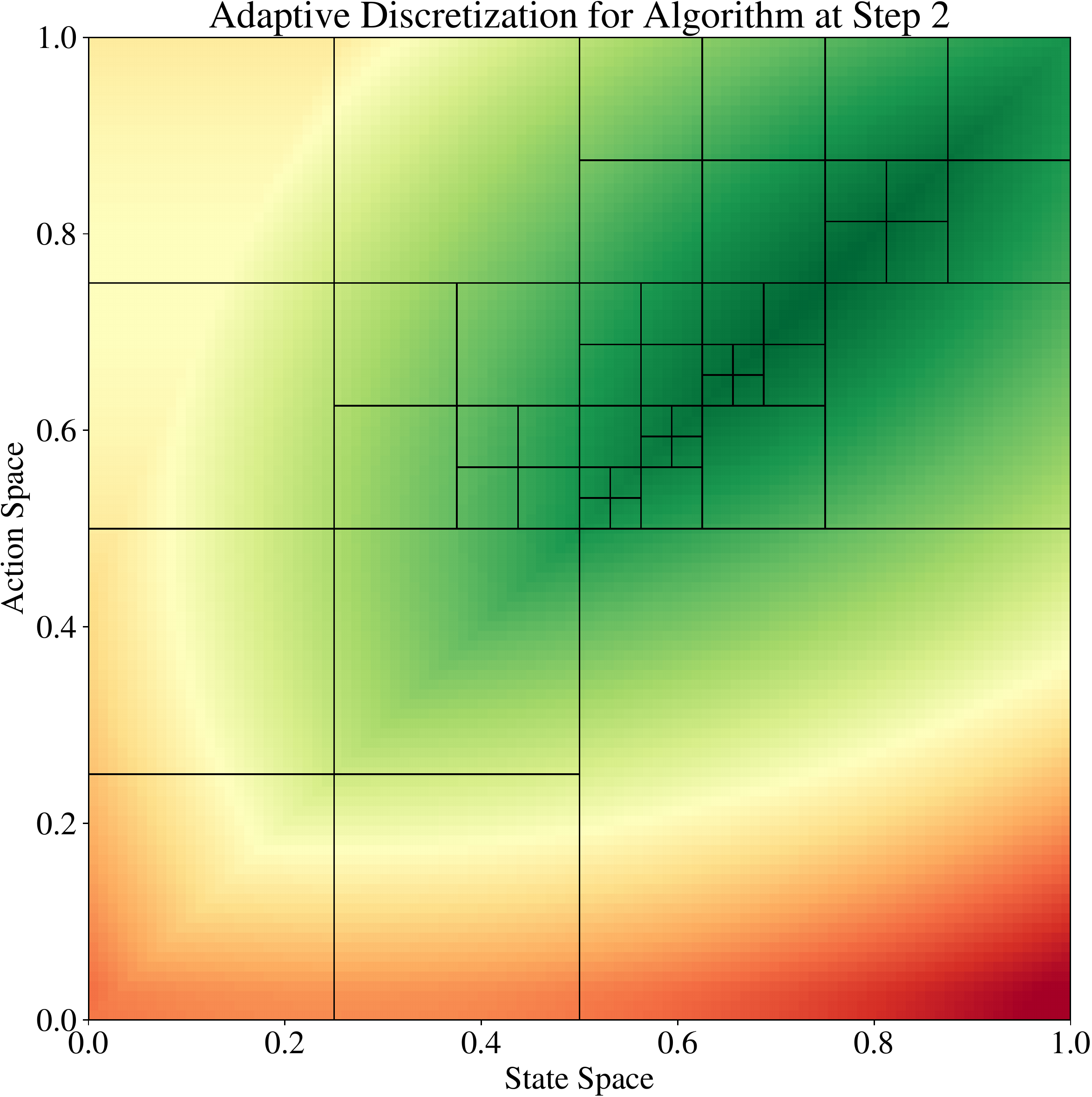}
  \caption*{Partitioning in practice}
\end{minipage}%
\caption{\em Partitioning scheme for $\S\times\A=[0,1]^2$: On the left, we illustrate our scheme. Partition $\Pkh[k-1]$ is depicted with corresponding tree (showing active balls in green, inactive parents in red). The algorithm plays ball $B_{h-1}^k$ in step $h-1$, leading to new state $X_h^k$. Since $\lev{B_{h-1}^k}=2$, we store transition estimates $\Tbar{h-1}{k}(\cdot \mid B_{h-1}^k)$ for all subsets of $\S_h$ of diameter $2^{-2}$ (depicted via dotted lines). The set of relevant balls $\relevant_h^k(X_h^k) = \{B_4,B_{21},B_{23}\}$ are highlighted in blue.\\
On the right, we show the partition $\P_{2}^K$ from one of our synthetic experiments (See `Oil Discovery' in~\cref{sec:experiments}). The colors denote the true $Q_2^\star(\cdot)$ values, with green corresponding to higher values. Note that the partition is more refined in areas which have higher $Q_2^\star(\cdot)$.}
\label{fig:partition}
\end{figure*}

While our partitioning works for any compact metric space, a canonical example to keep in mind is $\S=[0,1]^{d_\S},\A=[0,1]^{d_\A}$ with the infinity norm $\D((x,a),(x',a')) = ||(x,a)-(x',a')||_{\infty}$ (which was used in some of the simulations). 
We illustrate this in~\cref{fig:partition} for $d_\S=d_\A=1$. 
We define $\lev{B} = -\log_2(\D(B))$ to be the \emph{level} of a ball $B$, and construct $B$ as a level-$\lev{B}$ \emph{dyadic cube} in the metric-space $(\S\times\A,\D)$.
In our example of $([0,1]^2,||\cdot||_{\infty})$, a ball $B$ is an axis-aligned cube of length $2^{-\lev{B}}$ and corners in  $2^{-\lev{B}}\mathbb{Z}^{2}$, as depicted in~\cref{fig:partition}.

\medskip

At the end of each episode, for each active ball $B \in \P_h^k$ \AdaMB maintains three statistics:
\begin{itemize}
    \item $n_h^k(B)$: the number of times the ball $B$ has been \textit{selected} up to and including episode $k$.
    \item $\rhat{h}{k}(B)$: the empirical (instantaneous) reward earned from playing actions in $B$.\\
$\rbar{h}{k}(B)$: the empirical reward earned from playing actions in $B$ and its ancestors.
    \item $\{\That{h}{k}(\cdot \mid B)\}$: the empirical fractions of transitions to sets in a $2^{-\lev{B}}$-coarse partition of $\S_{h+1}$ (which we denote as $\dyad{\lev{B}}$) after playing actions in $B$.\\ $\{\Tbar{h}{k}(\cdot \mid B)\}$: the empirical fractions of transitions from playing actions in $B$ and its ancestors.
\end{itemize}

These estimates are used to construct \emph{optimistic $Q$-function estimates} $\Qhat{h}{k}(B)$ for each $B\in\Pkh$. Each ball $B\in\Pkh$ has an \emph{associated action} $\tilde{a}(B)\in\A(B)$ (we take this to be the center of the ball $\A (B)$). 

\medskip

\textbf{The \AdaMB Algorithm}: Given the above partitions and statistics, the algorithm proceeds as follows. In each episode $k$ and step $h$, \AdaMB observes state $X_h^k$, and finds all \emph{relevant} balls $\relevant_h^k(X_h^k) = \{B\in\Pkh[k-1]|X_h^k\in B\}$ (see~\cref{fig:partition}). 
It then selects an action according to a \emph{greedy selection rule}, picking $B_h^k \in \relevant_h^k(X_h^k)$ with highest $\Qhat{h}{k-1}(B)$, and plays action $\tilde{a}(B_h^k)$.  Note that the algorithm can also play any action $a$ such that $(X_h^k, a) \in B_h^k$ uniformly at random and the theory still applies.
Next, the algorithm \textit{updates counts} for $\rhat{h}{k}(B_h^k)$ and $\That{h}{k}(\cdot \mid B_h^k)$ based on the observed reward $R_h^k$ and transition to $X_{h+1}^k$.
Following this, it \textit{refines the partition} if needed. 
Finally, at the end of the episode, \AdaMB \textit{updates estimates} by solving for $\Qhat{h}{k}(\cdot)$ which are used in the next episode. We now describe the last three subroutines in more detail; see~\cref{alg:full_brief} for the full pseudocode, and \cref{app:implementation_run_time} for implementation, space, and run-time analysis.

\medskip
\textbf{Update Counts}: After playing active ball $B_h^k$ and observing $(R_h^k, X_{h+1}^k)$ for episode $k$ step $h$, \\ 
-- Increment counts and reward estimates according to
\[n_h^k(B_h^k) = n_h^{k-1}(B_h^k)+1 ~~\text{ and }~~  \rhat{h}{k}(B_h^k) = \frac{n_h^{k-1}(B_h^k)\rhat{h}{k-1}(B_h^k) + R_h^k}{n_h^{k}(B_h^k)}.\]
-- Update $\That{h}{k}(\cdot \mid B_h^k)$ as follows: For each set $A$ in a $2^{-\lev{B}}$-coarse partition of $\S_{h+1}$ denoted by $\dyad{\lev{B}}$, we set
\[\That{h}{k}(A\mid B) = \frac{n_h^{k-1}(B_h^k)\That{h}{k-1}(A\mid B)+\mathds{1}_{\{X_{h+1}^k\in A\}}}{ n_h^{k}(B_h^k)}.\]
This is maintaining an empirical estimate of the transition kernel for a ball $B$ at a level of granularity proportional to its diameter $\D(B) = 2^{-\lev{B}}.$

\medskip
\textbf{Refine Partition}: To refine the partition over episodes, we split a ball when the confidence in its estimate is smaller than its diameter. 
Formally, for any ball $B$, we define a \emph{splitting threshold} $\nplus{B} = \phi2^{\gam\lev{B}}$ , and partition $B$ once we have $n_h^k(B)+1 \geq \nplus{B}$. Note the splitting threshold grows exponentially with the level.  More concretely the splitting threshold is defined via
\begin{align*}
    n_+(B) & = \phi 2^{d_\S \ell(B)} \quad\quad d_\S > 2 \\
    n_+(B) & = \phi 2^{(d_\S + 2) \ell(B)} \quad\quad d_\S \leq 2
\end{align*}
where the difference in terms comes from the Wasserstein concentration.  This is in contrast to the splitting threshold for the model-free algorithm where $\nplus{\B} = 2^{2\lev{B}}$ \cite{Sinclair_2019}.  The $\phi$ term is chosen to minimize the dependence on $H$ in the final regret bound where $\phi = H^{(d+d_\S)/(d+1)}$.

In episode $k$ step $h$, if we need to split $B_h^k$, then we partition $\S(B_h^k)\times\A(B_h^k)$ using new balls each of diameter $\frac{1}{2}\D(B_h^k)$. 
This partition $\P(B)$ can be constructed by taking a cross product of a level $(\lev{B}+1)$-dyadic partition of $\S(B_h^k)$ and a level-$(\lev{B}+1)$ dyadic partition of $\A(B_h^k)$.  
We then remove $B$ and add $\P(B)$ to $\P_h^{k-1}$ to form the new partition $\P_h^{k}$.
In practice, each child ball can inherit all estimates from its parent, and counts for the parent ball are not updated from then on.  However, for ease of presentation and analysis we assume each child ball starts off with fresh estimates of $\rhat{h}{k}(\cdot)$ and $\That{h}{k}(\cdot)$ and use $\rbar{h}{k}(\cdot)$ and $\Tbar{h}{k}(\cdot)$ to denote the aggregate statistics.

\medskip	    
\textbf{Compute Estimates}: At the end of the episode we set
\begin{align*}
    \rbar{h}{k}(B) & =  \frac{\sum_{B' \supseteq B} \rhat{h}{k}(B') n_h^k(B')}{\sum_{B' \supseteq B} n_h^k(B')} \\
    \Tbar{h}{k}(A \mid B) & = \frac{\sum_{B' \supseteq B} \sum_{A' \in \dyad{\ell(B')}; A \subset A'}2^{-d_\S(\ell(B') - \ell(B))}n_h^k(B') \That{h}{k}(A' \mid B')}{\sum_{B' \supseteq B} n_h^k(B')}
\end{align*}
When aggregating the estimates of the transition kernel, we need to multiply by a factor to ensure we obtain a valid distribution.  This is because any ancestor $B'$ of $B$ maintain empirical estimates of the transition kernel to a level $\dyad{\lev{B'}}$.  Thus, we need to split the mass in order to construct a distribution over $\dyad{\lev{B}}$.  We also define confidence terms typically used in multi-armed bandits which are defined via:
\begin{align*}
    &\rbonus{h}{k}(B) = \sqrt{\frac{8\log(2HK^2/\delta)}{\sum_{B' \supseteq B} n_h^k(B')}} + 4 L_r \D(B)\\
    &\tbonus{h}{k}(B) 
    = \begin{cases}
    L_V \left((5L_T + 4) \D(B) + 4 \sqrt{\frac{\log(HK^2 / \delta)}{\sum_{B' \subseteq B} n_h^k(B')}} + c \left(\sum_{B' \subseteq B} n_h^k(B')\right)^{-1/d_\S}\right) &\text{ if } d_\S > 2 \\
    L_V \left((5 L_T + 6) \D(B) + 4 \sqrt{\frac{\log(HK^2/\delta)}{\sum_{B' \supseteq B} n_h^k(B')}} + c \sqrt{\frac{2^{d_\S \lev{B}}}{\sum_{B' \supseteq B} n_h^k(B')}}\right) &\text{ if } d_\S \leq 2   
    \end{cases}
\end{align*}

The difference in definitions of $\tbonus{h}{k}(\cdot)$ comes from the Wasserstein concentration in \cref{app:concentration}.  With these in place we set
\begin{align}
\label{eq:q_update}
    \Qhat{h}{k}(B) & := \begin{cases}
    & \rbar{H}{k}(B) + \rbonus{H}{k}(B) \hfill \text{ if } h = H \\
    & \rbar{h}{k}(B) + \rbonus{h}{k}(B) + \E_{A \sim \Tbar{h}{k}(\cdot \mid B)}\big[\Vhat{h+1}{k-1}(A)\big]
		+ \tbonus{h}{k}(B) \hfill \text{ if } h < H\end{cases}
\end{align}
mimicing the Bellman equations by replacing the true unknown quantities with their estimates.  The value function estimates are computed in a two-stage process.  For each ball $A \in \S(\P_h^k)$ we have that
\begin{align}
\label{eq:v_tilde_update}
    \Vtilde{h}{k}(A) := \min\{\Vtilde{h}{k-1}(A), \max_{B \in \P_h^k: \S(B) \supseteq A} \Qhat{h}{k}(B) \}.
\end{align}

For technical reasons we need to construct a Lipschitz continuous function to estimate the value function in order to show concentration of the transition kernel estimates.  For each point $x \in \S_h$ we define
\begin{align}
\label{eq:def-V}
    \Vhat{h}{k}(x) = \min_{A' \in \S(\P_h^k)} \left(\Vtilde{h}{k}(A') + L_V \D_S(x, \tilde{x}(A') \right).
\end{align}
However, as the support of $\Tbar{h}{k}(\cdot \mid B)$ is only over sets in $\dyad{\lev{B}}$ we overload notation to let $\Vhat{h}{k}(A) = \Vhat{h}{k}(\tilde{x}(A))$. We equivalently overload notation so that $x \sim \Tbar{h}{k}(\cdot \mid B)$ refers to sampling over the centers associated to balls in $\dyad{\lev{B}}$.

This corresponds to a value-iteration step, where we replace the true rewards and transitions in the Bellman Equations (\cref{eqn:bellman_equation}) with their (optimistic) estimates. We only compute one-step updates as in  \cite{efroni2019tight}, which reduces computational complexity as opposed to solving the full Bellman update.

Note that at the end of the episode, for each step $h$, we only need to update $\Qhat{h}{k}(B)$ for $B = B_h^k$ and $\Vtilde{h}{k}(A)$ for each $A \in \S(\P_h^k)$ such that $A \subseteq B_h^k$. $\Vhat{h}{k}$ is only used to compute the expectation in \cref{eq:q_update}, and thus it is only evaluated in episode $k+1$ for balls $A$ in the $2^{-\lev{B^{k+1}_{h-1}}}$-coarse partition of $\S_{h}$.
    \section{Main Results}
\label{sec:main_results}

We provide two main forms of performance guarantees, worst-case regret bounds with arbitrary starting states, which yields sample-complexity guarantees for learning a policy.

\subsection{Worst-Case Regret Guarantees}

We start with giving worst-case regret guarantees for \AdaMB.
\begin{theorem}
\label{thm:regret}
	Let $d = d_A + d_S$, then the regret of \AdaMB for any sequence of starting states $\{X_1^k\}_{k=1}^K$ is upper bounded with probability at least $1 - \delta$ by
	\begin{align*}
	R(K) & \lesssim \begin{cases}
	LH^{1+\frac{1}{d+1}}K^{\frac{d+d_\S - 1}{d+d_\S}} \quad d_S > 2 \\
	LH^{1 + \frac{1}{d+1}}K^{\frac{d+d_\S + 1}{d+d_\S + 2}} \quad d_S \leq 2
	\end{cases}
	\end{align*}
	where $L = 1 + L_r + L_V + L_V L_T$ and $\lesssim$ omits poly-logarithmic factors of $\frac{1}{\delta}, H,K,$ $d$, and any universal constants.
\end{theorem}

\noindent \textbf{Comparison to Model-Free Methods}: Previous model-free algorithms achieve worst-case bounds scaling via $H^{5/2} K^{(d+1)/(d+2)}$, which achieve the optimal worst-case dependence on the dimension $d$~\cite{Sinclair_2019}.  The bounds presented here have better dependence on the number of steps $H$.  This is expected, as current analysis for model-free and model-based algorithms under tabular settings shows that model-based algorithms achieve better dependence on $H$. However, under the Lipschitz assumptions here the constant $L$ also scales with $H$ so the true dependence is somewhat masked.  A modification of our algorithm that uses full planning instead of one-step planning will achieve linear dependence on $H$, with the negative effect of increased run-time.  When we compare the dependence on the number of episodes $K$ we see that the dependence is worse - primarily due to the additional factor of $d_\S$, the covering dimension of the state-space.  This term arises as model-based algorithms maintain an estimate of the transition kernel, whose complexity depends on $d_\S$.

\medskip

\noindent \textbf{Comparison to Model-Based Methods}:  Current state of the art model-based algorithms (\textsc{Kernel-UCBVI}) achieve regret scaling like $H^{3} K^{2d/(2d+1)}$ \cite{domingues2020regret}.  We achieve better scaling with respect to both $H$ and $K$, and our algorithm has lower time and space complexity.  However, we require additional oracle assumptions on the metric space to be able to construct packings and coverings efficiently, whereas \textsc{Kernel-UCBVI} uses the data and the metric itself.  Better dependence on $H$ and $K$ is primarily achieved by using recent work on concentration for the Wasserstein metric.  These guarantees allow us to construct tighter confidence intervals which are independent of $H$, obviating the need to construct a covering of $H$-uniformly bounded Lipschitz functions like prior work (see \cref{app:concentration}).

In addition, \textsc{Kernel-UCBVI} uses a fixed bandwidth parameter in their kernel interpolation.  We instead keep an adaptive partition of the space, helping our algorithm maintain a smaller and more efficient discretization.  This technique also lends itself to show instance dependent bounds, which we leave for future work.


\medskip

\noindent \textbf{Discussion on Instance-Specific Bounds}: The bounds presented here are worst-case, problem independent guarantees.  Recent work has shown that model-free algorithms are able to get problem dependent guarantees which depend on the zooming dimension instead of the covering dimension of the space \cite{cao2020provably}.  Extending this result to model-based algorithms will be more technical, due to requiring improved concentration guarantees for the transition kernel.  Most model-based algorithms require showing \textit{uniform concentration}, in particular that the estimate of the transition kernel concentrates well when taking expectation over any Lipschitz function.  Getting tighter bounds for model-based algorithms in continuous settings will require showing that the transition kernel is naturally estimated well in parts of the space that matter - as the state-visitation frequency is dependent on the policy used.  In \cref{sec:main_proof} we discuss the transition concentration in more details.

\ifdefined\ARXIV

\else
\begin{figure*}[!t]
\centering
  \includegraphics[width=\columnwidth]{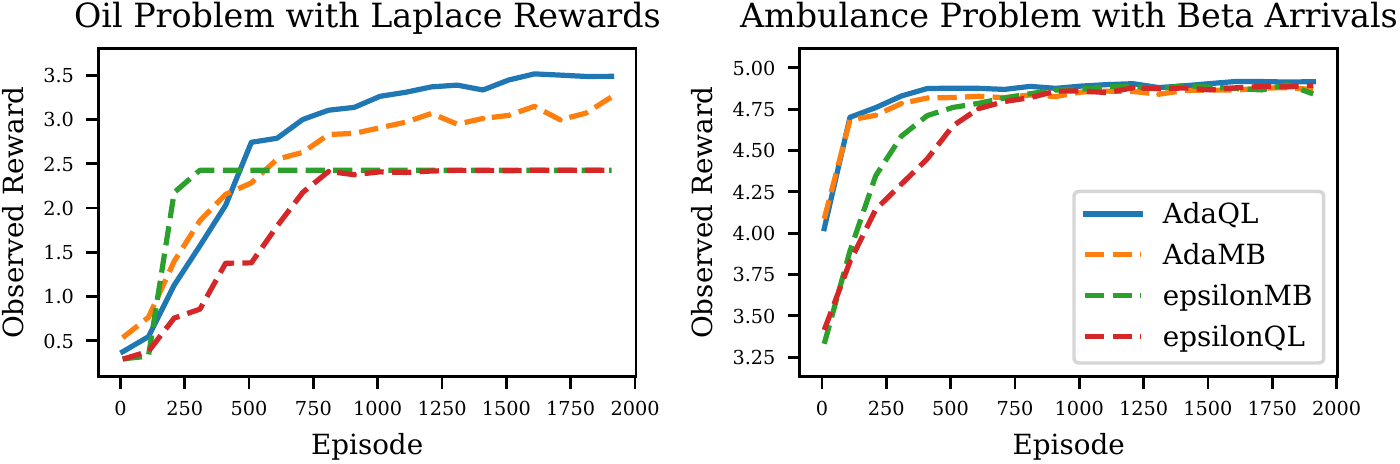}
\caption{\em Here we compare four algorithms: \textsc{epsQL} and \textsc{epsMB}, uniform discretization versions of model-free and model-based algorithms, and \textsc{adaQL} and \AdaMB, their adaptive discretization counterparts.  More simulation results are in \cref{sec:experiments}.}
\label{fig:result_main_paper}
\end{figure*}

\medskip

\noindent \textbf{Proof Sketch}:  The high level proof is divided into three sections.  First we show \textit{concentration and clean-events}, under which our estimates $\rbar{}{}$ and $\Tbar{}{}$ constitute upper bounds on the relevant quantities (\cref{app:concentration}).  Afterwards, we show a \emph{regret decomposition}, which relates the difference between the estimated value and the value accumulated by the algorithm to the bonus terms (\cref{app:decomp}).  Lastly, we use an LP-based argument to bound the \emph{worst-case size of the partition} and the \emph{sum of the bonus terms} which is used for the final regret bound (\cref{app:lp_bound}).  The full proof sketch is in \cref{sec:main_proof}.

\fi\medskip

\subsection{Policy-Identification Guarantees}

We can also adapt the algorithm to give sample complexity guarantees on learning a policy of a desired quality.  We use the PAC guarantee framework for learning RL policies \cite{watkins1989learning}.  Under this setting we assume that in each episode $k \in [K]$ the agent receives an initial state $X_1^k$ drawn from some fixed distribution, and try to find the minimum number of episodes needed to find a near-optimal policy with high probability.

Following similar arguments as in \cite{jin_2018, Sinclair_2019} it is straightforward to show that
\begin{theorem}
After running \AdaMB with a number of episodes
\begin{align*}
    K = \begin{cases}
        & \tilde{O}\left(\frac{LH^{1+\frac{1}{d+1}}}{\delta \epsilon}\right)^{d+d_\S} \quad d_\S > 2\\
        & \tilde{O}\left(\frac{LH^{1+\frac{1}{d+1}}}{\delta \epsilon}\right)^{d+d_\S+2} \quad d_S \leq 2
        \end{cases}
\end{align*}
consider a policy $\pi$ chosen uniformly at random from $\pi_1, \ldots, \pi_k$.  Then for an initial state $X$ drawn from the starting distribution, with probability at least $1 - \delta$ the policy $\pi$ obeys $$V_1^\star(X) - V_1^\pi(X) \leq \epsilon.$$
\end{theorem}
    \section{Proof Sketch}
\label{sec:main_proof}

The high level proof is divided into three sections.  First, we show \textit{concentration and clean-events}, under which our estimates constitute upper bounds on the relevant quantities.  Afterwards, we show a \textit{regret decomposition}, which relates the difference between the estimated value and the value accumulated by the algorithm with the bonus terms.  Lastly, we use an LP-based argument to bound the \textit{worst-case size of the partition} and the \textit{sum of the bonus terms} which is used for the final regret bound.  We discuss each of them briefly before giving more technical details.  As the final regret-bound is technical \srsedit{and mostly involves algebra and combining terms}, its derivation is deferred to \cref{sec:finalregret}.

\subsection{Concentration and Clean Events (\cref{app:concentration})}

\AdaMB maintains estimates $\rbar{h}{k}(B)$ and $\Tbar{h}{k}(\cdot \mid B)$ of the unknown rewards and transitions of the underlying MDP.  In order to ensure that the one-step value iteration update in Equation~\ref{eq:q_update} concentrates we need to verify that these estimates provide good approximations to their true quantities.  In particular, applying Azuma-Hoeffding's inequality shows that:

\begin{lemma}
With probability at least $1 - \delta$ we have that for any $h, k \in [H] \times [K]$ and ball $B \in \P_h^k$, and any $(x,a) \in B$, 
\begin{align*}
    \left|\rbar{h}{k}(B) - r_h(x,a) \right| \leq \rbonus{h}{k}(B).
\end{align*}
\end{lemma}

The next step is ensuring concentration of the transition estimates $\Tbar{h}{k}(\cdot \mid B)$.  As the algorithm takes expectations over Lipschitz functions with respect to these distributions, we use recent work on Wasserstein distance concentration.  This is in contrast to previous work that requires using a covering argument on the space of value functions in order to show concentration guarantees for the transition kernel \cite{dann2015sample,jin2019provably,domingues2020regret}.  In particular, we show the following:
\begin{lemma}
With probability at least $1 - 2 \delta$ we have that for any $h,k \in [H] \times [K]$ and ball $B \in \P_h^k$ with $(x,a) \in B$ that 
\begin{align*}
    d_W(\Tbar{h}{k}(\cdot \mid B), T_h(\cdot \mid x,a)) \leq \frac{1}{L_V}\tbonus{h}{k}(B)
\end{align*}
\end{lemma}

The main proof uses recent work on bounding the Wasserstein distance between an empirical measure and the true measure~\cite{weed2019sharp}.  For the case when $d_\S > 2$ the concentration inequality holds up to a level of $n^{-\frac{1}{d_\S}}$ with high probability.  We use this result by chaining the Wasserstein distance of various measures together.  Unfortunately, the scaling does not hold for the case when $d_\S \leq 2$.  In this situation we use the fact that $\Tbar{h}{k}(\cdot \mid B)$ is constructed as an empirical measure with finite support $|\dyad{\lev{B}}| = 2^{d_\S \lev{B}}$.  Although $T_h(\cdot \mid x,a)$ is a continuous distribution, we consider ``snapped'' versions of the distributions and repeat a similar argument.  This allows us to get the scaling of $\sqrt{2^{d_\S \lev{B}} / n}$ seen in the definition of $\tbonus{h}{k}(B)$.  The result from \cite{weed2019sharp} has corresponding lower bounds, showing that in the worst case scaling with respect to $d_\S$ is inevitable.  As the transition bonus terms leads to the dominating terms in the regret bounds, improving on our result necessitates creating concentration intervals around the value function instead of the model~\cite{ayoub2020model}.

The Wasserstein concentration established in the previous lemmas allows us to forgo showing uniform convergence of the transition kernels over all value functions.  Indeed, the variational definition of the Wasserstein metric between measures is 
$d_W(\mu, \nu) = \sup_{f} \int f d(\mu - \nu)$
where the supremum is taken over all $1$-Lipschitz functions.  Noting that $V_h^\star$ and $\Vhat{h}{k}(\cdot)$ are constructed to be $L_V$-Lipschitz functions we therefore get that for $V = V_h^\star$ or $V = \Vhat{h}{k}(\cdot)$:
\[ \E_{X \sim \Tbar{h}{k}(\cdot \mid B)}[V(X)] - \E_{X \sim T_h(\cdot \mid x,a)}[V(X)] \leq L_V d_W(\Tbar{h}{k}(\cdot \mid B), T_h(\cdot \mid x,a)) \leq \tbonus{h}{k}(B).\]

Getting improved bounds for model-based algorithms in continuous spaces will necessitate showing that the algorithm does not need to show uniform concentration over all value functions or all Lipschitz functions, but rather a subset that is constructed by the algorithm.

These concentration bounds allow us to now demonstrate a principle of \textit{optimism} for our value-function estimates. Formally, we show that conditioned on the concentration bounds on the rewards and transitions being valid, the estimates for $Q_h^\star$ and $V_h^\star$ constructed by \AdaMB are indeed upper bounds for the true quantities.
This follows a common approach for obtaining regret guarantees for reinforcement learning algorithms~\cite{simchowitz2019}. 
\begin{lemma}
With probability at least $1 - 3 \delta$, the following bounds are all simultaneously true for all $k,h\in[K]\times[H]$, and any partition $\P_h^k$
\begin{align*}
    \Qhat{h}{k}(B) & \geq Q_h^\star(x,a)\qquad \text{for all } B \in \P_h^k \text{, and } (x,a) \in B\\
    \Vtilde{h}{k}(A) & \geq V_h^\star(x)\qquad\;\; \text{ for all } A \in \S(\P_h^k), \text{ and } x \in A\\
    \Vhat{h}{k}(x) & \geq V_h^\star(x)\qquad\;\, \text{ for all } x \in \S
\end{align*}
\end{lemma}

\subsection{Regret Decomposition (\cref{app:decomp})}

Similar to \cite{efroni2019tight}, we use one step updates for $\Qhat{h}{k}(\cdot)$ and $\Vhat{h}{k}(\cdot)$. We thus use similar ideas to obtain the final regret decomposition, which then bounds the final regret of the algorithm by a function of the size of the partition and the sum of the bonus terms used in constructing the high probability estimates.  In particular, by expanding the update rules on $\Qhat{h}{k}(B)$ and $\Vhat{h}{k}(x)$ we can show:

\begin{lemma}
	The expected regret for \AdaMB can be decomposed as
	\begin{align*}
	\Exp{R(K)} & \lesssim \sum_{k=1}^K \sum_{h=1}^H \Exp{\Vtilde{h}{k-1}(\S(\P_h^{k-1},X_h^{k})) - \Vtilde{h}{k}(\S(\P_h^{k},X_h^k))} \\
	& + \sum_{h=1}^H \sum_{k=1}^K \Exp{2\rbonus{h}{k}(B_h^k)} + \sum_{h=1}^H \sum_{k=1}^K \Exp{2\tbonus{h}{k}(B_h^k)} + \sum_{k=1}^K \sum_{h=1}^H L_V \Exp{\D(B_h^k)}.
	\end{align*}
	where $\S(\P_h^{k-1}, X_h^k)$ is the region in $\S(\P_h^{k-1})$ containing the point $X_h^k$.
\end{lemma}

The first term in this expression arises from using one-step planning instead of full-step planning, and the rest due to the bias in the estimates for the reward and transitions.  Using the fact that the $\Vtilde{h}{k}$ are decreasing with respect to $k$ we can show that this term is upper bounded by the size of the partition.  Obtaining the final regret bound then relies on finding a bound on the size of the partition and the sum of bonus terms.

\subsection{Bounds on Size of Partition and Sums of Bonus Terms (\cref{app:lp_bound})}

We show technical lemmas that provide bounds on terms of the form $\sum_{k=1}^K \frac{1}{(n_h^k(B_h^k))^\alpha} $
almost surely based on the splitting rule used in the algorithm and the size of the resulting partition.  We believe that this is of independent interest as many optimistic regret decompositions involve bounding sums of bonus terms over a partition that arise from concentration inequalities.

We formulate these quantities as a linear program (LP) where the objective function is to maximize either the size of the partition or the sum of bonus terms associated to a valid partition (represented as a tree) constructed by the algorithm. The constraints follow from conditions on the number of samples required before a ball is split into subsequent children balls. To derive an upper bound on the value of the LP we find a tight dual feasible solution. This argument could be broadly useful and modified for problems with additional structures by including additional constraints into the LP.  In particular, we are able to show the following:
\begin{corollary}
For any $h\in[H]$, consider any sequence of partitions $\Pkh, k\in[K]$ induced under \AdaMB with splitting thresholds $\nplus{\ell}= \phi2^{\gam\ell}$.
Then, for any $h\in[H]$ we have:
\begin{itemize}
\item $|\P_h^k| \leq 4^d K^{\frac{d}{d+\gam}} \phi^{-\frac{d}{d+\gam}}$
\item For any $\alpha,\beta\geq 0$ s.t. $\alpha \leq 1$ and $\alpha\gamma-\beta\geq 1$, we have
    
\begin{align*}
\sum_{k=1}^K \frac{2^{\beta\lev{B_h^k}}}{\left(n_h^k(B_h^k)\right)^{\alpha}} = O\left(\phi^{\frac{-(d\alpha+\beta)}{d+\gam}} K^{\frac{d+(1-\alpha)\gam+\beta}{d+\gam}}\right)
\end{align*}

\item For any $\alpha,\beta\geq 0$ s.t. $\alpha \leq 1$ and $\alpha\gamma - \beta/\ell^{\star}\geq 1$ (where $\ell^{\star} =2+\frac{1}{d+\gam }\log_2\left(\frac{K}{\phi}\right)$), we have
\begin{align*}
\sum_{k = 1}^K \frac{\lev{B_h^k}^\beta}{\left(n_h^k(B_h^k)\right)^{\alpha}} = O\left(\phi^{\frac{-d\alpha}{d+\gam}} K^{\frac{d+(1-\alpha)\gam}{d+\gam}}\left(\log_2K\right)^{\beta}\right)
\end{align*}
\end{itemize}
\end{corollary}

We use this result with the regret decomposition to show the final regret bound.  The splitting threshold $\gam$ is taken in order to satisfy the requirements of the corollary.  As the dominating term arises from the concentration of the transition kernel, for the case when $d_\S > 2$ the sum is of the form when $\alpha = 1/d_\S$ and $\beta = 0$.  This gives the $K^{(d+d_\S - 1)/(d+d_\S)}$ term in the regret bound.  The case when $d_\S \leq 2$ is similar.
    \section{Concentration Bounds, Optimism, and Clean Events}
\label{app:concentration}

In this section we show that the bonus terms added on, namely $\rbonus{h}{k}(\cdot)$ and $\tbonus{h}{k}(\cdot)$, ensure that the estimated rewards and transitions are upper bounds for the true quantities with high probability. 
This follows a proof technique commonly used for multi-armed bandits and reinforcement learning, where algorithm designers ensure that relevant quantities are estimated optimistically with a bonus that declines as the number of samples increases.  

For all proofs we let $\{\F_{k}\}$ denote the filtration induced by all information available to the algorithm at the start of episode $k$, i.e. $\F_k = \sigma\left((X_h^{k'}, A_h^{k'}, B_h^{k'}, R_h^{k'})_{h \in [H], k' < k} \cup X_1^k\right)$ where we include the starting state for the episode.  With this filtration in place, all of the estimates $\Qhat{h}{k-1}$, $\Vhat{h}{k-1}$, and the policy $\pi_k$ are measurable with respect to $\F_k$.

Before stating the concentration inequalities, we first give a technical result, which we use to simplify the upper confidence terms. The proof of this result is deferred to~\cref{app:techproofs}.
\begin{lemma}
\label{lem:sum_ancestors}
For any $h, k \in [H] \times [K]$ and ball $B \in \P_h^k$ we have that 
\begin{align*}
    \frac{\sum_{B' \supseteq B} \D(B') n_h^k(B')}{\sum_{B' \supseteq B} n_h^k(B')} \leq 4 \D(B).
\end{align*}
\end{lemma}

\subsection{Concentration of Reward Estimates}
\label{ssec:concR}

We start by showing that with probability at least $1-\delta$, our reward estimate $\rbar{h}{k}(B) + \rbonus{h}{k}(B)$ is a \emph{uniform} upper bound on the true mean reward $r_h(x,a)$ for any $(x,a) \in B$.
\begin{lemma}
\label{lemma:reward_confidence}
With probability at least $1 - \delta$ we have that for any $h, k \in [H] \times [K]$ and ball $B \in \P_h^k$, and any $(x,a) \in B$, 
\begin{align*}
    \left|\rbar{h}{k}(B) - r_h(x,a) \right| \leq \rbonus{h}{k}(B),
\end{align*}
where we define $\rbonus{h}{k}(B) = \sqrt{\frac{8\log(2HK^2/\delta)}{\sum_{B' \supseteq B} n_h^k(B')}} + 4 L_r \D(B)$.
\end{lemma}
\begin{proof}
Let $h, k \in [H] \times [K]$ and $B \in \P_h^k$ be fixed and $(x,a) \in B$ be arbitrary.  First consider the left hand side of this expression,
\begin{align*}
    \left|\rbar{h}{k}(B) - r_h(x,a) \right| & = \left| \frac{\sum_{B' \supseteq B} \rhat{h}{k}(B') n_h^k(B')}{\sum_{B' \supseteq B} n_h^k(B')} - r_h(x,a) \right| \\
    & \leq \left|\frac{\sum_{B' \supseteq B} \sum_{k' \leq k} \Ind{B_h^{k'} = B'}(R_h^{k'} - r_h(X_h^{k'},A_h^{k'}))}{\sum_{B' \supseteq B} n_h^k(B')}\right| \\
    & \qquad+ \left|\frac{\sum_{B' \supseteq B} \sum_{k' \leq k} \Ind{B_h^{k'} = B'}(r_h(X_h^{k'},A_h^{k'}) - r_h(x,a))}{\sum_{B' \supseteq B} n_h^k(B')}\right|.
\end{align*}
where we use the definitions of $\rbar{h}{k}(B)$ and $\rhat{h}{k}(B)$ and the triangle inequality.

Next, using the fact that $r_h$ is Lipschitz continuous and that $(x,a)\in B \subseteq B'$ and $(X_h^{k'}, A_h^{k'}) \in B'$ have a distance bounded above by $\D(B')$, we can bound the second term by
\begin{align*}
    \left|\frac{\sum_{B' \supseteq B} \sum_{k' \leq k}\Ind{B_h^{k'} = B'} (r_h(X_h^{k'},A_h^{k'}) - r_h(x,a))}{\sum_{B' \supseteq B} n_h^k(B')}\right| & \leq \left|\frac{\sum_{B' \supseteq B} L_r\D(B')n_h^k(B')}{\sum_{B' \supseteq B} n_h^k(B')}\right|.
\end{align*}

Finally we bound the first term via the Azuma-Hoeffding inequality.  
Let $k_1, \ldots, k_t$ be the episodes in which $B$ and its ancestors were selected by the algorithm (i.e. $B_h^{k_i}$ is an ancestor of $B$); here $t = \sum_{B' \supseteq B} n_h^k(B')$.  
Under this definition the first term can be rewritten as
\begin{align*}
    \left|\frac{1}{t} \sum_{i=1}^t \left(R_h^{k_i} - r_h(X_h^{k_i}, A_h^{k_i})\right)\right|
\end{align*}
Set $Z_i = R_h^{k_i} - r_h(X_h^{k_i}, A_h^{k_i})$.  
Clearly $Z_i$ is a martingale difference sequence with respect to the filtration $\hat\F_i = \F_{k_i+1}$.  Moreover, as the sum of a martingale difference sequence is a martingale then for any $\tau \leq K$, $\sum_{i=1}^\tau Z_i$ is a martingale, where the difference in subsequent terms is bounded by $2$.  Thus by Azuma-Hoeffding's inequality we see that for a fixed $\tau \leq K$ that
\begin{align*}
    \Pr\left(\left|\frac{1}{\tau} \sum_{i=1}^\tau Z_i\right| \leq \sqrt{\frac{8\log(2HK^2/\delta)}{\tau}} \right) & \geq 1 - 2 \exp\left( -\frac{\tau \frac{8\log(2HK^2/\delta)}{\tau}}{8} \right) \\
    & = 1 - \frac{\delta}{2HK^2}.
\end{align*} 

When $\tau = t = \sum_{B' \supseteq B} n_h^k(B')$ the right hand side in the concentration is precisely 
\begin{align*}
    \sqrt{\frac{8\log(2HK^2/\delta)}{\sum_{B' \supseteq B} n_h^k(B')}}.
\end{align*}

We then take a union bound over all steps $H$ and episodes $K$ and all $K$ possible values of $\tau$. Note that we do not need to union bound over the balls $B \in \P_h^k$ as the estimate of only one ball is changed per (step, episode) pair, i.e. $\rhat{h}{k}(B)$ is changed for a single ball per episode. For all balls not selected, it inherits the concentration of the good event from the previous episode because its estimate does not change. Furthermore, even if ball $B$ is ``split'' in episode $k$, all of its children inherit the value of the parent ball, and thus also inherits the good event, so we still only need to consider the update for $B_h^k$ itself.

Combining these we have for any $(h,k) \in [H] \times [K]$ and ball $B \in \P_h^k$ such that $(x,a) \in B$
\begin{align*}
    |\rbar{h}{k}(B) - r_h(x,a)| & \leq \sqrt{\frac{8\log(2HK^2/\delta)}{\sum_{B' \supseteq B} n_h^k(B')}} + \frac{\sum_{B' \supseteq B} L_r \D(B')n_h^k(B')}{\sum_{B' \supseteq B} n_h^k(B')} \\
    & \leq \sqrt{\frac{8\log(2HK^2/\delta)}{\sum_{B' \supseteq B} n_h^k(B')}} + 4 L_r \D(B) = \rbonus{h}{k}(B) \qquad\qquad\text{(by~\cref{lem:sum_ancestors})}.
\end{align*}
\end{proof}

\subsection{Concentration of Transition Estimates}
\label{ssec:concT}

Next we show concentration of the estimate of the transition kernel.  We use recent work on bounding the Wasserstein distance between the empirical distribution and the true distribution for arbitrary measures \cite{weed2019sharp}.  The proof is split into two cases, where the cases define the relevant $\tbonus{h}{k}(\cdot)$ used. We state the result here but defer the full proof to \cref{app:techproofs}.

\begin{lemma}
\label{lemma:transition_confidence}
With probability at least $1 - 2 \delta$ we have that for any $h,k \in [H] \times [K]$ and ball $B \in \P_h^k$ with $(x,a) \in B$ that 
\begin{align*}
    d_W(\Tbar{h}{k}(\cdot \mid B), T_h(\cdot \mid x,a)) \leq \frac{1}{L_V}\tbonus{h}{k}(B)
\end{align*}
\end{lemma}

\subsection{Optimism Principle}
\label{ssec:optimism}

The concentration bounds derived in~\cref{ssec:concR,ssec:concT} allow us to now demonstrate a principle of \textit{optimism} for our value-function estimates.
\begin{lemma}
\label{lemma:optimism}
With probability at least $1 - 3 \delta$, the following bounds are all simultaneously true for all $k,h\in[K]\times[H]$, and any partition $\P_h^k$
\begin{align*}
    \Qhat{h}{k}(B) & \geq Q_h^\star(x,a)\qquad \text{for all } B \in \P_h^k \text{, and } (x,a) \in B\\
    \Vtilde{h}{k}(A) & \geq V_h^\star(x)\qquad\;\; \text{ for all } A \in \S(\P_h^k), \text{ and } x \in A\\
    \Vhat{h}{k}(x) & \geq V_h^\star(x)\qquad\;\, \text{ for all } x \in \S
\end{align*}
\end{lemma}
\begin{proof}
Recall the `good events' in~\cref{lemma:reward_confidence,lemma:transition_confidence} simultaneously hold with probability $1 - 3 \delta$.
Conditioned on this, we show the result by forwards induction on $k$ and backwards induction on $h$.

\medskip

\noindent \textbf{Base Case ($k = 0$)}:
Recall the estimates are initialized as $\Qhat{h}{k}(\cdot) = \Vhat{h}{k}(\cdot) = \Vtilde{h}{k}(\cdot) = H-h+1$. Now since all the rewards lie in $[0,1]$, we have that $Q_h^\star(\cdot, \cdot)$ and $V_h^\star(\cdot)$ are upper bounded by $H-h+1$, and so optimism holds for any $h \in [H]$.

\medskip

\noindent \textbf{Induction ($k - 1 \rightarrow k$)}: 
We first consider $h = H+1$ and then proceed via backwards induction on $h$. 
For $h = H+1$, optimism holds as all quantities are zero. Next, for any $B \in \P_h^k$ and any $(x,a) \in B$,
\begin{align*}
    \Qhat{h}{k}(B) & = \rbar{h}{k}(B) + \rbonus{h}{k}(B) + \E_{Y \sim \Tbar{h}{k}(\cdot \mid B)}[\Vhat{h+1}{k-1}(Y)] + \tbonus{h}{k}(B) \\
	& \geq r_h(x,a) + \E_{Y \sim \Tbar{h}{k}(\cdot \mid B)}[V_{h+1}^\star(Y)] + \tbonus{h}{k}(B) \text{ (induction hypothesis and~\cref{lemma:reward_confidence})}\\
	& \geq r_h(x,a) + \E_{Y \sim T_h(\cdot \mid x, a)}[V_{h+1}^\star(Y)] = Q_h^\star(x,a) \text{ (by~\cref{lemma:transition_confidence})}
\end{align*}
where we used the fact that $V_h^\star$ is $L_V$-Lipschitz continuous and that the difference in expectation over any Lipschitz function with respect to two different distributions is bounded above by the Wasserstein distance times the Lipschitz constant.

For any $A \in \S(\P_h^k)$ and any $x \in A$, if $\Vtilde{h}{k}(A) = \Vtilde{h}{k-1}(A)$ then optimism clearly holds by the induction hypothesis, and otherwise
\begin{align*}
    \Vtilde{h}{k}(A) & = \max_{B \in \P_h^k: \S(B) \supseteq A} \Qhat{h}{k}(B) \\
    & \geq \Qhat{h}{k}(B^\star) ~~\text{ for } (x, \pi_h^\star(x)) \in B^\star \\
    & \geq Q_h^\star(x, \pi_h^\star(x)) = V_h^\star(x).
\end{align*}
For $x \in A \in \S(\P_h^k)$, and for the ball $B^\star \in \P_h^k$ that satisfies $(x, \pi_h^\star(x)) \in B^\star$, it must be that $\S(B^\star) \supseteq A$ because of the construction of the induced partition $\S(\P_h^k)$ via \cref{eq:induced_state_partition_def}, the dyadic partitioning of $P_h^k$ which guarantees $\S(\P_h^k)$ is a partition, and the fact that $x \in \S(B^{\star})$.

And lastly we have that for any $x \in \S$, 
\begin{align*}
    \Vhat{h}{k}(x) & = \Vtilde{h}{k}(A) + L_V d_\S(x, \tilde{x}(A)) \quad\text{ for some ball } A \in \S(\P_h^k) \\
    & \geq V_h^\star(\tilde{x}(A)) + L_V d_\S(x, \tilde{x}(A)) \quad\text{ by optimism of }\Vtilde{h}{k}\\
    & \geq V_h^\star(x) \quad\text{ by Lipschitzness of } V_h^\star.
\end{align*}
Note that when a ball $B$ is split, it inherits all estimates from its parents, and thus it inherits the optimistic properties from its parents value functions as well.
\end{proof}

    \section{Sample-Path Regret Decomposition}
\label{app:decomp}

We next outline our sample-path regret decomposition for one-step value iteration, which uses an idea adapted from Lemma 12 in~\cite{efroni2019tight}.  We introduce the notation $\S(\P_h^k, x)$ to refer to the state-ball in $\S(\P_h^k)$ which contains the point $x$.  The proofs of both results are deferred to \cref{app:techproofs}.

We begin by showing a result on the one-step difference between the estimated value of the policy and the true value of the policy employed.  This critically uses the one-step value-iteration update in order to express the difference as a decreasing bounded process plus the sum of bonus terms.

\begin{lemma}
	\label{lemma:recursive}
	Consider any $h,k\in[H]\times[K]$, and any dyadic partition $\Pkh[k-1]$ of $\S\times\A$. Then the value update of \AdaMB in the $k$'th episode in step $h$ is upper bounded by
	\begin{align*}
	&\Vtilde{h}{k-1}(\S(\P_h^{k-1},X_h^{k}))- V_h^{\pi^k}(X_h^k) \\
	&\quad \leq 
	\sum_{h' = h}^H \Expk{\Vtilde{h'}{k-1}(\S(\P_{h'}^{k-1},X_{h'}^{k})) - \Vtilde{h'}{k}(\S(\P_{h'}^k,X_{h'}^k)) \mid X_h^k}{k-1} \\
	&\quad\quad +  \sum_{h' = h}^H \Expk{\rbar{h'}{k}(B_{h'}^k) - r_{h'}(X_{h'}^k, A_{h'}^k) + \rbonus{h'}{k}(B_{h'}^k) \mid X_h^k}{k-1} \\
	&\quad\quad + \sum_{h' = h}^H \Expk{\E_{x \sim \Tbar{h'}{k}(\cdot \mid B_{h'}^k)}[\Vhat{h+1}{k-1}(x)] - \E_{x \sim T_{h'}(\cdot \mid X_{h'}^k, A_{h'}^k)}[\Vhat{h'+1}{k-1}(x)] \mid X_h^k}{k-1} \\
	& \quad\quad + \sum_{h'=h}^H \Expk{\tbonus{h'}{k}(B_{h'}^k) \mid X_h^k}{k-1} + L_V \sum_{h'=h+1}^H \Expk{\D(B_{h'}^k) \mid X_h^k}{k-1}
	\end{align*}
\end{lemma}

The proof follows directly by expanding and substituting the various quantities.  Moreover, using this lemma, we can further decompose the expected regret using the optimism principle defined in \cref{app:concentration}.

\begin{lemma}
	\label{lemma:regret_decomposition}
	The expected regret for \AdaMB can be decomposed as
	\begin{align*}
	\Exp{R(K)} & \lesssim \sum_{k=1}^K \sum_{h=1}^H \Exp{\Vtilde{h}{k-1}(\S(\P_h^{k-1},X_h^{k})) - \Vtilde{h}{k}(\S(\P_h^{k},X_h^k))} \\
	& + \sum_{h=1}^H \sum_{k=1}^K \Exp{2\rbonus{h}{k}(B_h^k)} + \sum_{h=1}^H \sum_{k=1}^K \Exp{2\tbonus{h}{k}(B_h^k)} + \sum_{k=1}^K \sum_{h=1}^H L_V \Exp{\D(B_h^k)}.
	\end{align*}
\end{lemma}
This again follows from the definition of regret, and uses Lemma~\ref{lemma:recursive}.  The proof is provided in \cref{app:techproofs}.

Next we analyze the first term in the regret decomposition by arguing it is bounded uniformly over all sample paths.
\begin{lemma}
	\label{lem:decreasing_process}
	Under \AdaMB, along every sample trajectory we have
	\begin{align*}
	\sum_{k=1}^K \sum_{h=1}^H \Vtilde{h}{k-1}(\S(P_h^{k-1}, X_h^k)) - \Vtilde{h}{k}(\S(\P_h^k, X_h^k)) \leq H^2 \max_{h} |\S(\P_h^K)|.
	\end{align*}
\end{lemma}
\begin{proof}
	We show a somewhat stronger bound, namely, that for every $h\in[H]$ we have
	$$\sum_{k=1}^K \Vtilde{h}{k-1}(\S(P_h^{k-1}, X_h^k)) - \Vtilde{h}{k}(\S(\P_h^k, X_h^k)) \leq (H-h+1) |\S(\P_h^k)|$$
	from which the claim then follows.
	
	Recall that by definition, we have $\Vtilde{h}{k-1}(\S(P_h^{k-1}, x))$ is non-decreasing $\frall x\in\S$. Now we can write
	\begin{align*}
	\sum_{k=1}^K \Vtilde{h}{k-1}(\S(P_h^{k-1}, X_h^k)) - \Vtilde{h}{k}(\S(\P_h^k, X_h^k)) & \leq \sum_{k=1}^K \sum_{A \in \S(\P_h^K)} \Vtilde{h}{k-1}(A) - \Vtilde{h}{k}(A) 
	\end{align*}
	where for a set $A \in \S(\P_h^K)$ which is not in $\P_h^k$ we let $\Vtilde{h}{k}(A)$ be the $\Vtilde{h}{k}(\cdot)$ value of the ball in $\S(\P_h^k)$ which contains $A$ (i.e., we set $\Vtilde{h}{k-1}(A) = \Vtilde{h}{k-1}(\S(\P_h^{k-1},\tilde{x}(A)))$ and $\Vtilde{h}{k}(A) = \Vtilde{h}{k}(\S(\P_h^{k},\tilde{x}(A)))$).
	Finally, we can change the order of summations to get
	\begin{align*}
	\sum_{k=1}^K \sum_{A \in \S(\P_h^K)} \Vtilde{h}{k-1}(A) - \Vtilde{h}{k}(A) & = \sum_{A \in \S(\P_h^K)} \sum_{k=1}^K \Vtilde{h}{k-1}(A) - \Vtilde{h}{k}(A) \\
	& = \sum_{A \in \S(\P_h^K)} \Vtilde{h}{0}(A) - \Vtilde{h}{K}(A) \\
	& \leq (H-h+1) |\S(\P_h^k)|.
	\end{align*}
\end{proof}
    \section{Adversarial Bounds for Counts over Partitions}
\label{app:lp_bound}

Recall that the splitting threshold is defined to be: split a ball once we have that $n_h^k(B) + 1 \geq n_+(B)$ where $n_+(B) = \phi 2^{\gamma \ell(B)}$ for parameters $\phi$ and $\gamma$.  As the splitting threshold only depends on the level of the ball in the partition, we abuse notation and use $\nplus{\ell} = \phi2^{\gam\ell}$ to denote the threshold number of samples needed by the splitting rule to trigger splitting a ball at level $\ell$.
We first provide a general bound for counts over any partition $\Pkh$.
\begin{lemma}
\label{lem:LPbound}
Consider any partition $\Pkh$ for any $k\in[K], h\in[H]$ induced under \AdaMB with splitting thresholds $\nplus{\ell}$, and consider any `penalty' vector $\{a_{\ell}\}_{\ell\in\NN_0}$ that satisfies $a_{\ell+1} \geq a_{\ell} \geq 0$ and $2a_{\ell+1}/a_{\ell} \leq \nplus{\ell}/\nplus{\ell-1}$ for all $\ell\in\NN_0$. Define $\ell^{\star} = \inf\{\ell \mid 2^{d(\ell - 1)} \nplus{\ell-1} \geq k\}$. Then
\begin{align*}
\sum_{\ell=0}^{\infty}\sum_{B\in \Pkh : \ell(B) = \ell} a_{\ell} \leq 2^{d\ell^{\star}}a_{\ell^{\star}}     \end{align*}
\end{lemma}

\begin{proof}
For $\ell\in\NN_0$, let $x_\ell$ denote the number of active balls at level $\ell$ in $\Pkh$. Then $\sum_{B\in \Pkh:\lev{B}=\ell} a_{\ell} = \sum_{\ell\in\NN_0}a_{\ell} x_\ell$. 
Now we claim that under any partition, this sum can be upper bound via the following linear program (LP):
\begin{align*}
  \text{maximize: } & \quad \sum_{\ell=0}^{\infty} a_{\ell}x_{\ell} \\
  \text{subject to: } & \quad  \sum_{\ell}2^{-\ell d} x_{\ell} \leq 1 \; , \\
  & \quad \sum_{\ell}\nplus{\ell-1}2^{-d} x_{\ell} \leq k \; , \\
  & \quad    x_{\ell} \geq 0 \,\forall\,\ell
\end{align*}
The first constraint arises via the Kraft-McMillan inequality for  prefix-free codes (see Chapter 5 in~\cite{cover2012elements}): since each node can have at most $D = 2^d$ (where $d = d_\S + d_\A$) children by definition of the covering dimension, the partition created can be thought of as constructing a prefix-free code on a $D$-ary tree. 
The second constraint arises via a conservation argument on the number of samples; recall that $\nplus{B}$ is the minimum number of samples required before $B$ is split into $2^d$ children -- an alternate way to view this is that each ball at level $\ell$ requires a `sample cost' of $\nplus{\ell-1}/2^d$ unique samples in order to be created. The sum of this sample cost over all active balls is at most the number of samples $k$.

Next, via LP duality, we get that the optimal value for this program is upper bounded by $\alpha + \beta$ for any $\alpha$ and $\beta$ such that:
\begin{align*}
    2^{-\ell d} \alpha + n_+(\ell - 1) 2^{-d} \beta & \geq a_\ell \quad \forall \ell \in \NN_0 \\
    \alpha, \beta & \geq 0.
\end{align*}

Recall the definition of $\ell^\star = \inf\{\ell \mid 2^{d(\ell - 1)} \nplus{\ell-1} \geq k\}$ and consider 
\begin{align*}
    \hat\alpha = \frac{2^{d\ell^{\star}}a_{\ell^{\star}}}{2} \quad \hat\beta = \frac{2^{d}a_{\ell^{\star}}}{2\nplus{\ell^{\star}-1}}.
\end{align*}

We claim that this pair satisfies the constraint that $2^{-\ell d} \hat\alpha + \nplus{\ell-1}2^{-d} \hat\beta \geq a_{\ell}$ for any $\ell$, and hence by weak duality we have that
\[\sum_{B\in \Pkh:\lev{B}=\ell} a_{\ell}\leq \hat\alpha+ \hat\beta \leq 2\hat\alpha = 2^{d\ell^{\star}}a_{\ell^{\star}}.\]

To verify the constraints on $(\hat\alpha,\hat\beta)$ we check it by cases.
First note that for $\ell = \ell^\star$, we have $2^{-\ell^{\star} d} \hat\alpha + \nplus{\ell^{\star}-1}2^{-d}\hat\beta = a_{\ell^\star}$. 

Next, for any $\ell < \ell^{\star}$, note that $2^{-\ell d} \geq 2^{-(\ell^{\star}-1) d} > 2\cdot(2^{-\ell^{\star} d})$, and hence $2^{-\ell d} \hat\alpha \geq 2\cdot(2^{-\ell^{\star} d} \hat\alpha) = a_{\ell^{\star}} \geq a_{\ell}$ by construction of the penalty vector.

Similarly, for any $\ell > \ell^{\star}$, we have by assumption on the costs and $\nplus{\ell}$ that
\begin{align*}
    \frac{\nplus{\ell-1}}{a_{\ell}} \geq \frac{2^{\ell - \ell^\star}\nplus{\ell^{\star}-1}}{a_{\ell^\star}} \geq 2\frac{\nplus{\ell^\star - 1}}{a_{\ell^\star}}.
\end{align*}
Then we get by plugging in our value of $\hat\beta$ that
\begin{align*}
    \nplus{\ell - 1}2^{-d} \hat\beta & = \frac{a_{\ell^\star} \nplus{\ell - 1}}{2\nplus{\ell^\star - 1}}
     \geq a_\ell
\end{align*}This verifies the constraints for all $\ell\in\NN_0$.
\end{proof}
Note also that in the above proof, we actually use the condition $2a_{\ell+1}/a_{\ell} \leq \nplus{\ell}/\nplus{\ell-1}$ for $\ell\geq \ell^{\star}$; we use this more refined version in~\cref{lem:countbound} below.


\subsection{Worst-Case Partition Size and Sum of Bonus Terms}

One immediate corollary of~\cref{lem:LPbound} is a bound on the size of the partition $|\Pkh|$ for any $h,k$.
\begin{corollary}
\label{lem:size_partition}
For any $h$ and $k$ we have that 
\begin{align*}
    |\P_h^k| \leq 4^d\left(\frac{k}{\phi}\right)^{\frac{d}{d+\gam}}
\end{align*}
and that 
\begin{align*}
    \ell^\star \leq \frac{1}{d + \gam} \log_2(k / \phi) + 2.
\end{align*} 
\end{corollary}
\begin{proof}
Note that the size of the partition can be upper bounded by the sum where we take $a_\ell = 1$ for every $\ell$.  Clearly this satisfies the requirements of Lemma~\ref{lem:LPbound}.  Moreover, using the definition of $\ell^\star$ we have that $2^{d(\ell^\star - 2)}\nplus{\ell^\star - 2} \leq k$ as otherwise $\ell^\star - 1$ would achieve the infimum.  Taking this equation and plugging in the definition of $\nplus{\ell}$ by the splitting rule yields that 
\begin{align*}
    \ell^\star \leq \frac{1}{d+\gam} \log_2\left(\frac{k}{\phi}\right) + 2.
\end{align*}
Then by plugging this in we get that
\begin{align*}
    |\P_h^k| & \leq 2^{d\ell^\star}
     \leq 2^{\frac{d}{d+\gamma}\log_2(k / \phi) + 2d}
     = 4^d \left(\frac{k}{\phi}\right)^{d/(d+\gam)}.
\end{align*}
\end{proof}
In other words, the worst case partition size is determined by a \emph{uniform} scattering of samples, wherein the entire space is partitioned up to equal granularity (in other words, a uniform $\epsilon$-net). 

More generally, we can use \cref{lem:LPbound} to bound various functions of counts over balls in $\Pkh$. In~\cref{sec:finalregret} we use this to bound various terms in our regret expansion.
\begin{corollary}
\label{lem:countbound}
For any $h\in[H]$, consider any sequence of partitions $\Pkh, k\in[K]$ induced under \AdaMB with splitting thresholds $\nplus{\ell}= \phi2^{\gam\ell}$.
Then, for any $h\in[H]$ we have:
\begin{itemize}[nosep,leftmargin=*]
\item For any $\alpha,\beta\geq 0$ s.t. $\alpha \leq 1$ and $\alpha\gamma-\beta\geq 1$, we have
    
\begin{align*}
\sum_{k=1}^K \frac{2^{\beta\lev{B_h^k}}}{\left(n_h^k(B_h^k)\right)^{\alpha}} = O\left(\phi^{\frac{-(d\alpha+\beta)}{d+\gam}} K^{\frac{d+(1-\alpha)\gam+\beta}{d+\gam}}\right)
\end{align*}

\item For any $\alpha,\beta\geq 0$ s.t. $\alpha \leq 1$ and $\alpha\gamma - \beta/\ell^{\star}\geq 1$ (where $\ell^{\star} =2+\frac{1}{d+\gam }\log_2\left(\frac{K}{\phi}\right)$), we have
\begin{align*}
\sum_{k = 1}^K \frac{\lev{B_h^k}^\beta}{\left(n_h^k(B_h^k)\right)^{\alpha}} = O\left(\phi^{\frac{-d\alpha}{d+\gam}} K^{\frac{d+(1-\alpha)\gam}{d+\gam}}\left(\log_2K\right)^{\beta}\right)
\end{align*}
\end{itemize}
\end{corollary}

The proof of both the inequalities follows from a direct application of~\cref{lem:LPbound} (and in fact, using the same $\ell^{\star}$ as in~\cref{lem:size_partition}), after first rewriting the summation over balls in $\Pkh$ as a summation over active balls in $\Pkh[K]$. The complete proof is deferred to~\cref{app:techproofs}.

    \begin{figure}[t!]
    \centering
    \includegraphics[width=.75\textwidth]{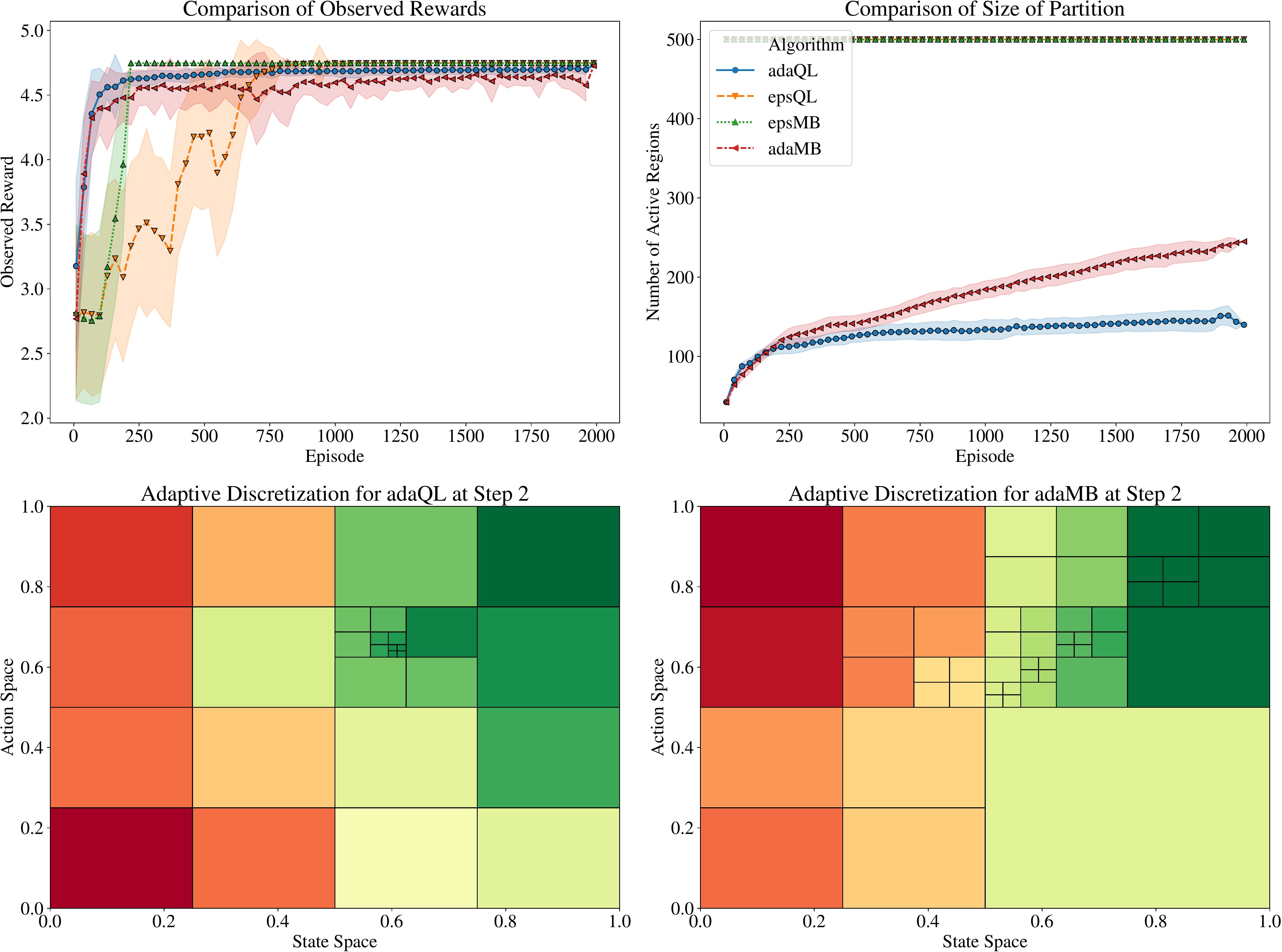}
    \caption{Comparison of the observed rewards, size of the partition, and resulting discretization for the four algorithms on the one-dimensional oil problem with no noise and survey function $f(x,a) = 1-(x-.7)^2$ and $\alpha = 1$.  The colours correspond to the estimated $\Qhat{h}{k}(B)$ values, where green corresponds to a larger estimated $Q$ value.}
    \label{fig:oil_quadratic_1}
\end{figure}

\begin{figure}[t!]
    \centering
    \includegraphics[width=.75\textwidth]{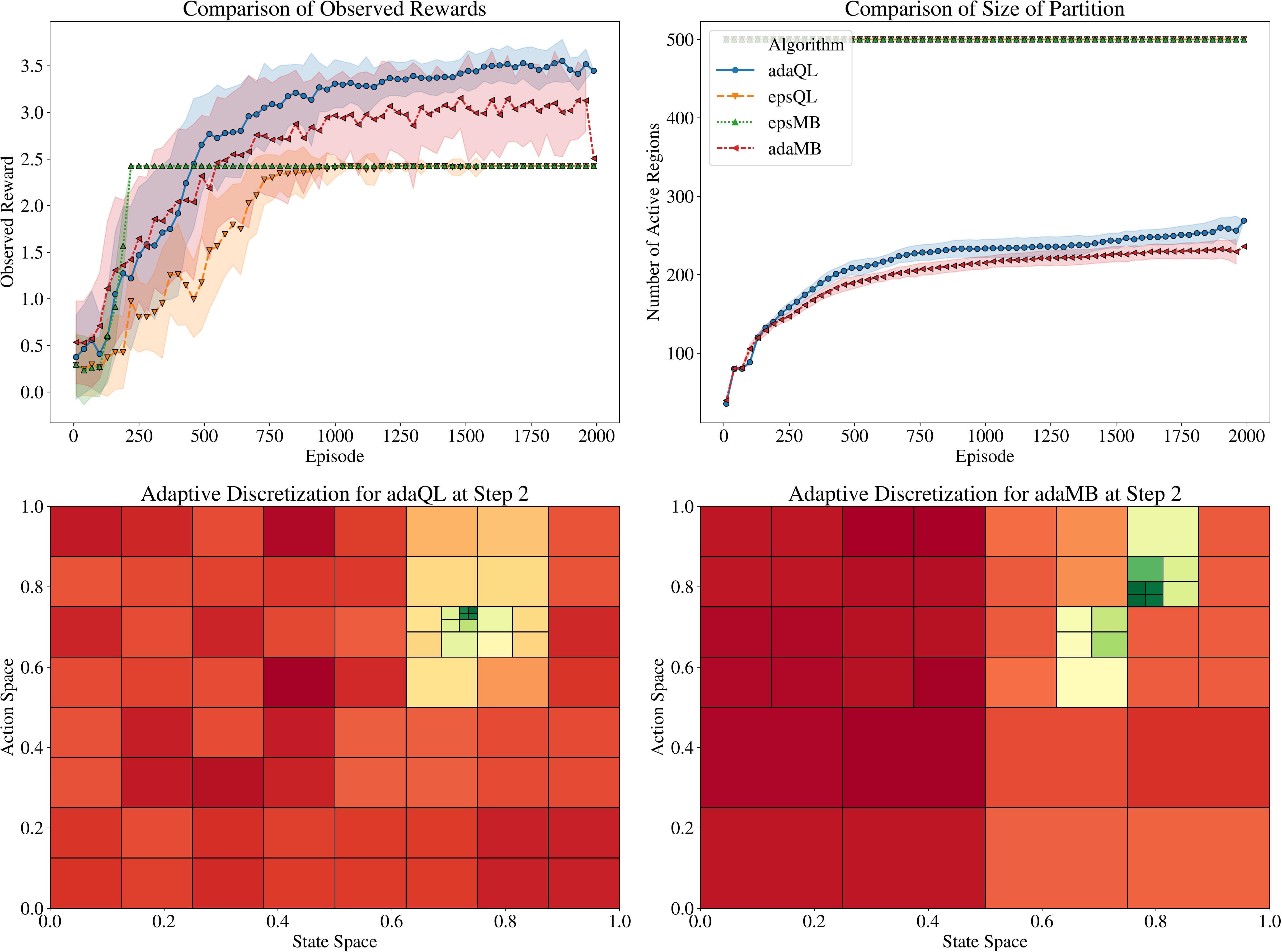}
    \caption{Comparison of the observed rewards, size of the partition, and resulting discretization for the four algorithms on the one-dimensional oil problem in the ``laplace-reward'' setting with $\alpha = 1$ and $\lambda = 10$.  The colours correspond to the estimated $\Qhat{h}{k}(B)$ values, where green corresponds to a larger estimated $Q$ value.}
    \label{fig:oil_sparse_1}
\end{figure}

\section{Experiments}
\label{sec:experiments}

In this section we give full details on the experiments and simulations performed.  For full code implementation and more results please see the Github repository at \url{https://github.com/seanrsinclair/AdaptiveQLearning}.

For the experiments we were motivated to work on ambulance routing and the oil discovery problem as efficient algorithms for reinforcement learning in operations tasks is still largely unexplored.  It is, however, a very natural objective in designing systems where agents must learn to navigate an uncertain environment to maximize their utility.  These experiments can have broader implications in planning effective public transportation, stationing medics at events, or even cache management (which technically is a discrete measurement, but is most usefully talked about in a continuous manner due to the magnitude of memory units).

The main objective for continuous space problems in reinforcement learning is to meaningfully store continuous data in a discrete manner while still producing optimal results in terms of performance and reward. We find that the oil discovery and ambulance routing problems are simple enough that we can realistically produce uniform discretization benchmarks to test our adaptive algorithm against. At the same time, they provide interesting continuous space scenarios that suggest there can be substantial improvements when using adaptive discretization in real world problems. The ambulance routing problem also allows us to naturally increase the state and action space dimensionality by adding another ambulance and consequently test our algorithms in a slightly more complex setting.  In particular, we compare \textsc{Adaptive Q-Learning}\cite{Sinclair_2019}, \textsc{Model-Free $\epsilon$-Net}\cite{song2019efficient}, \textsc{AdaMB} (Algorithm~\ref{alg:brief}), and a $\epsilon$-net variant of UCBVI~\cite{azar2017minimax}.  We refer to the simulations as \textsc{AdaQL}, \textsc{epsilonQL}, \textsc{AdaMB}, and \textsc{epsilonMB} respectively in the figures.


\subsection{Oil Discovery}

This problem, adapted from \cite{mason2012collaborative} is a continuous variant of the ``Grid World'' environment. It comprises of an agent surveying a 1D map in search of hidden ``oil deposits''.  The world is endowed with an unknown survey function which encodes the probability of observing oil at that specific location.  For agents to move to a new location they pay a cost proportional to the distance moved, and surveying the land produces noisy estimates of the true value of that location.  In addition, due to varying terrain the true location the agent moves to is perturbed as a function of the state and action.

\begin{figure}[t!]
    \centering
    \includegraphics[scale=0.75]{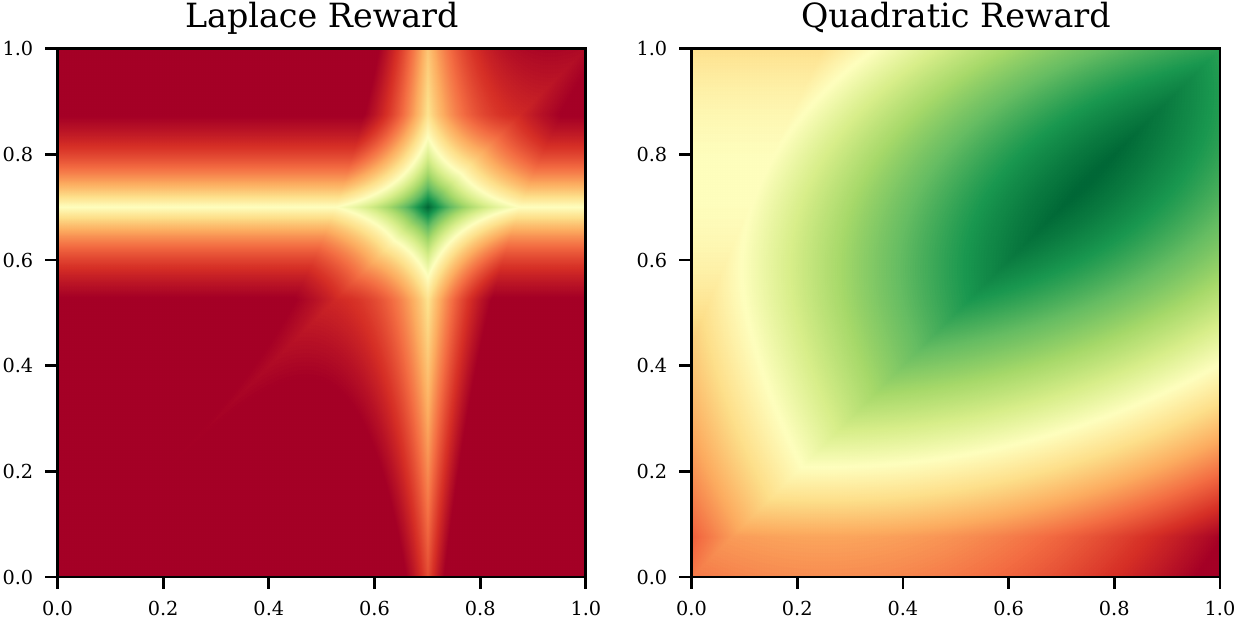}
    \caption{Plot of discretized approximation to $Q_h^\star$ for the one-dimensional oil problem in the ``laplace'' $(\lambda = 10)$ and ``quadratic'' ($\lambda = 1)$ setting.  Note that here the $x$-axis corresponds to states and the $y$-axis to actions.  The colour corresponds to the true $Q_2^\star$ value where green corresponds to a larger value.}
    \label{fig:heat_map}
\end{figure}

To formalize the problem, here the state space $\S = [0,1]$ and action space $\A = [0,1]$, where the product space is endowed with the $\ell_\infty$ metric.  The reward function is defined as 
\begin{align*}
    r_h(x,a) = \max\{ \min\{ f_h(x,a) - \alpha |x-a| + \epsilon, 1\}, 0\} 
\end{align*}
where $f_h(x,a)$ is the survey function, corresponding to the probability of observing an oil deposit at that specific location and $\alpha$ is a parameter used to govern the transportation cost and $\epsilon$ is independent Gaussian noise.  The transition function is defined as
\begin{align*}
    \Pr_h(\cdot \mid x,a) = \max\{ \min\{ \delta_{a} + N(0, \sigma_h(x,a)^2), 1\}, 0\}
\end{align*}
where again we have truncated the new state to fall within $[0,1]$ and the noise function $\sigma_h(x,a)$ allows for varying terrain in the environment leading to noisy transitions.  Clearly if we take $\sigma_h(x,a) = 0$ we recover deterministic transitions from a state $x$ taking action $a$ to the next state being $a$.

We performed three different simulations, where we took $f_h(x,a)$ and $\sigma_h(x,a)$ as follows:

\medskip

\noindent \textbf{Noiseless Setting}: $\sigma_h(x,a) = 0$ and the reward function $f_h(x,a) = 1 - \lambda(x-c)^2$ or $f_h(x,a) = 1 - e^{-\lambda |x-c|}$ where $c$ is the location of the oil deposit and $\lambda$ is a tunable parameter.

\medskip

\noindent \textbf{Sparse-Reward Setting}: $\sigma_h(x,a) = .025(x+a)^2$ and the survey function is defined via:
\begin{align*}
    f_h(x,a) & = \begin{cases}
        \frac{1}{h}\left(1-e^{-\lambda|x-.5|}\right) \qquad & h = 1 \\
        \frac{1}{h}\left(1-e^{-\lambda|x-.25|}\right) \qquad & h = 2 \\
        \frac{1}{h}\left(1-e^{-\lambda|x-.5|}\right) \qquad & h = 3 \\
        \frac{1}{h}\left(1-e^{-\lambda|x-.75|}\right) \qquad & h = 4 \\
        \frac{1}{h}\left(1-e^{-\lambda|x-1|}\right) \qquad & h = 5 \\
        \end{cases}
\end{align*}

\medskip

\noindent \textbf{Discussion.} We can see in Figure~\ref{fig:oil_sparse_1} and in Figure~\ref{fig:oil_quadratic_1} that the \textsc{epsilonQL} algorithm takes much longer to learn the optimal policy than its counterpart \textsc{epsilonMB} and both model-based algorithms. Seeing improved performance of model-based algorithms over model-free with a uniform discretization is unsurprising, as it is folklore that model-based algorithms perform better than model-free in discrete spaces.  

The two adaptive algorithms also offer a significantly smaller partition size than the corresponding uniform discretization. After comparing the adaptive algorithms' discretization of estimated $Q$-values with the true $Q_2^\star$-values in the state-action space, we find that the adaptive algorithms closely approximate the underlying $Q$ function (see Figure~\ref{fig:heat_map}). This is as the adaptive algorithms maintain a much finer partition in regions of the space where the underlying $Q^\star$ values are large, thus reducing unnecessary exploration (hence reducing the size of the partition), and allowing the algorithm to learn the optimal policy faster (low regret).  This demonstrates our algorithms' effectiveness in allocating space only to where it is advantageous to exploit more rigorously.  Interestingly, we see that the model-free algorithm is able to more closely resemble the underlying $Q^\star$ values than the model-based algorithm.  This affirms recent work showing instance-dependent bounds for model-free algorithms \cite{cao2020provably}, and our discussion on the drawback of model-based algorithms storing estimates of the transition kernel.

Moreover, in the attached github repository we include code testing the necessity of the splitting rule in the model based algorithm being of the form $n_+(B) = \phi 2^{\gam \lev{B}}$ for various forms of $\gam$.  While the theoretical results indicate that $\gam = d_\S$ is necessary for convergence, experimentally we see that $\gam = 2$ matching the model-free algorithm also suffices.

\begin{figure}[t!]
    \centering
    \includegraphics[width=.75\textwidth]{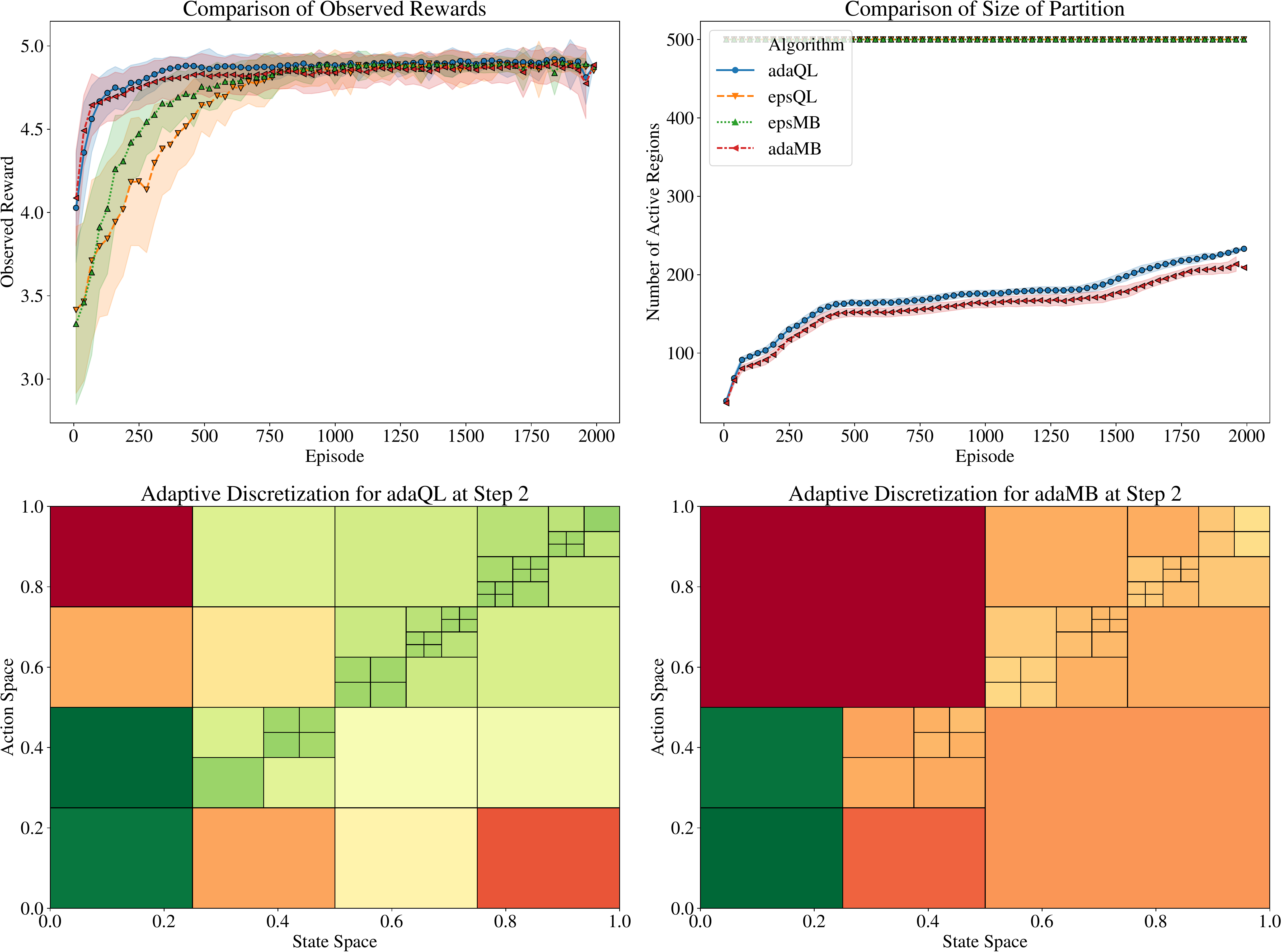}
    \caption{Comparison of the observed rewards, size of the partition, and resulting discretization for the four algorithms on the one ambulance problem with $\alpha = 1$ and arrivals $\F_h = \text{Beta}(5,2)$.  The colors correspond to the estimated $\Qhat{h}{k}(B)$ values, where green corresponds to a larger estimated $Q$ value.}
    \label{fig:single_ambulance}
\end{figure}

\begin{figure}[t!]
    \centering
    \includegraphics[width=.75\textwidth]{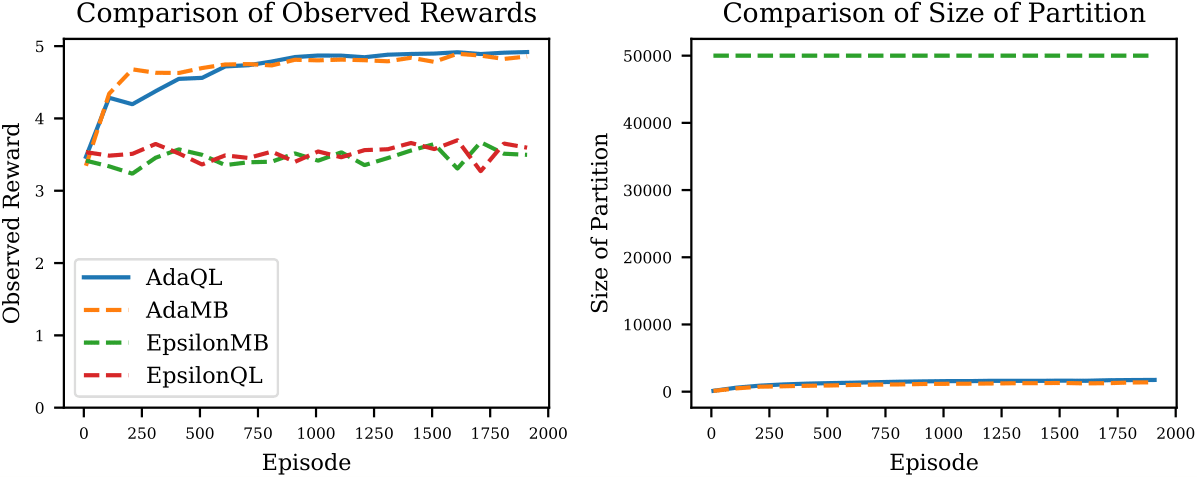}
    \caption{Comparison of the observed rewards and the size of the partition for the four algorithms on the two ambulance problem with $\alpha = 1$ and arrivals $\F_h = \text{Beta}(5,2)$.  We ommit confidence bars in this plot to help with readability.}
    \label{fig:multi_ambulance}
\end{figure}

\subsection{Ambulance Routing}

This problem is a widely studied question in operations research and control, and is closely related to the $k$-server problem.  A controller positions a fleet of $k$ ambulances over $H$ time periods, so as to minimize the transportation costs and time to respond to incoming patient requests.  In our setting, the controller first chooses locations to station the ambulances.  Next, a single request is realized drawn from a fixed $h$-dependent distribution.  Afterwards, one ambulance is chosen to travel to meet the demand, while other ambulances can re-position themselves.

Here the state space $\S = [0,1]^k$ and action space $\A = [0,1]^k$ where $k$ is the number of ambulances, and the product space is endowed with the $\ell_\infty$ metric.  The reward function and transition is defined as follows.  First, all ambulances travel from their initial state $x_i$ to their desired location $a_i$, paying a transportation cost to move the ambulance to location $a_i$.  Afterwards, a patient request location $p_h \sim \F_h$ is drawn i.i.d. from a fixed distribution $\F_h$.  The closest ambulance to $p_h$ is then selected to serve the patient, i.e. let 
\begin{align*}
    i^\star = \argmin_{i \in [k]} |a_i - p_h|
\end{align*}
denote the ambulance traveling to serve the patient.  The rewards and transitions are then defined via:
\begin{align*}
    x_i^{new} & = \begin{cases}
        a_i \qquad & i \neq i^\star \\
        p_h \qquad & i = i^\star
        \end{cases} \\
    r_h(x, a) & = 1 - \left(\frac{\alpha}{k} \norm{x - a}_1 + (1-\alpha) |a_{i^\star} - p_h| \right)
\end{align*}
where $\alpha$ serves as a tunable parameter to relate the cost of initially traveling from their current location $x$ to the desired location $a$, and the cost of traveling to serve the new patient $p_h$.  We tested values of $\alpha$ in $\{0, .25, 1\}$ where $\alpha = 1$ corresponds to only penalizing the ambulances for traveling to the initial location, $\alpha = 0$ only penalizes agents for traveling to serve the patient, and $\alpha = 0.25$ interpolates between these two settings.

For the arrival distributions, we took $\F_h = \text{Beta}(5,2)$, $\F_h = \text{Uniform}(0,1)$ and a time-varying arrival distribution:
\begin{align*}
    \F_h & = \begin{cases}
        \text{Uniform}(0,.25) \qquad & h = 1 \\
        \text{Uniform}(.25,.3) \qquad & h = 2 \\
        \text{Uniform}(.3,.5) \qquad & h = 3 \\
        \text{Uniform}(.5,.6) \qquad & h = 4 \\
        \text{Uniform}(.6,.65) \qquad & h = 5 \\
        \end{cases}
\end{align*}

\medskip

\noindent \textbf{Discussion.} In both the single ambulance case (Figure~\ref{fig:single_ambulance}) and two-ambulance (Figure~\ref{fig:multi_ambulance}) we see that the uniform discretization algorithms are outperformed by their adaptive counterparts.  Unsurprisingly, the partition size of both adaptive algorithms is significantly smaller than the epsilon algorithms, with \textsc{adaQL} being slightly more efficient.  We also see that both adaptive algorithms perform similarly in terms of rate of convergence and observed rewards for both the two and one ambulance problem.  Again, this is because the adaptive algorithms maintain a finer partition in regions of the space where the underlying $Q^\star$ values are large, thus reducing the size of the partition and leading the algorithm to learn the optimal policy faster.  When looking at the resulting discretizations in Figure~\ref{fig:single_ambulance} we observe similar results to the oil problem, where the model-free algorithm exhibits a finer partition than the model-based algorithm.
    \section{Conclusion}
\label{sec:conclusion}

We presented an algorithm using adaptive discretization for model-based online reinforcement learning based on one-step planning.  In worst case instances, we showed regret bounds for our algorithm which are competitive with other model-based algorithms in continuous settings under the assumption that the underlying dynamics of the system are Lipschitz continuous with respect to a known metric on the space.  We also provided simulations comparing model-based and model-free methods using an adaptive and fixed discretizations of the space on several canonical control problems.  Our experiments showed that adaptive partitioning empirically performs better than fixed discretizations in terms of both faster convergence and lower memory.

One future direction for the work is analyzing the discrepancy between model-based and model-free methods in continuous settings, as model-based algorithms so far have sub-optimal dependence on the dimension of the space.  Moreover, we hope to characterize problems where model-based methods using adaptive discretization are able to outperform model-free methods using a gap-dependent analysis inspired by recent gap-dependent analysis for tabular algorithms \cite{simchowitz2019}.
    \section*{Acknowledgements}
    Part of this work was done while Sean Sinclair and Christina Yu were visiting the Simons Institute for the Theory of Computing for the semester on the Theory of Reinforcement Learning. We also gratefully acknowledge funding from the NSF under grants ECCS-1847393, DMS-1839346, CCF-1948256, and CNS-1955997, the ARL under grant W911NF-17-1-0094, and the Cornell Engaged Grant: Applied Mathematics in Action.
    
    \bibliographystyle{plain}
    {\bibliography{references}}
    \appendix
    \newpage
    \section{Table of Notation}
\label{app:notation}

\renewcommand{\arraystretch}{1.2}
\begin{table*}[h!]
\begin{tabular}{l|l}
\textbf{Symbol} & \textbf{Definition} \\ \hline
\multicolumn{2}{c}{Problem setting specifications}\\
\hline
$\S,\A,H,K$  & State space, action space, steps per episode, number of episodes\\
$r_h(x,a)\,,\,T_h(\cdot \mid x,a)$ & Average reward/transition kernel for taking action $a$ in state $x$ at step $h$\\
$\pi_h,V_h^\pi(\cdot),Q_h^\pi(\cdot,\cdot)$ & Arbitrary step-$h$ policy, and Value/$Q$-function at step $h$ under $\pi$ \\
$\pi^{\star}_h, V_h^\star(\cdot),Q_h^\star(\cdot,\cdot)$ & Optimal step-$h$ policy, and corresponding Value/Q-function \\
$L_r,L_T, L_V$ & Lipschitz constants for $r$, $T$ and $V^\star$ respectively\\
$\D_\S$, $D_\A$, $\D$ & Metrics on $\S$, $\A$, and $\S \times \A$ respectively\\
\hline
\multicolumn{2}{c}{Algorithm variables and parameters}\\
\hline
$k,h$ & Index for episode, index for step in episode \\
$(X_h^k, A_h^k, R_h^k)$ & State, action, and received reward under algorithm at step h in episode k \\
$\Pkh$ & Partition tree of $\S \times \A$ for step $h$ at end of episode $k$\\
$\relevant_h^k(x)$ & Set of balls relevant for $x$ at $(k,h)$ (i.e., $\{B\in\Pkh[k-1]| (x,a)\in B \text{ for some } a \in \A\}$)\\
$\tilde{x}(B), \tilde{a}(B)$ & Associated state/action for ball $B$ (i.e., `center' of ball $B$) \\
$B_h^k$ & Ball in $\Pkh[k-1]$ selected at $(k,h)$
($\argmax_{B \in \text{RELEVANT}_h^k(X_h^k)} \Qhat{h}{k-1}(B)$)
\\
$\nplus{B}$ & Threshold number of samples after which ball $B$ is split  \\
$\Qhat{h}{k}(B)$ & $Q$-function estimates for ball $B\in\Pkh$, {at end of} episode $k$\\
$\Vtilde{h}{k}(A)$ & $V$-function estimate for a ball $A \in \S(\P_h^k)$, at end of episode $k$\\
$\Vhat{h}{k}(x)$ & $V$-function estimate for a point $x \in \S$, at end of episode $k$\\
$n_h^k(B)$ & Number of times $B$ has been chosen {by the end of episode $k$}\\
$\rhat{h}{k}(B),\That{h}{k}(\cdot \mid B)$ & Empirical rewards and transitions from ball $B\in\Pkh$ at end of episode $k$\\
$\rbar{h}{k}(B),\Tbar{h}{k}(\cdot \mid B)$ & Inherited reward/transition estimates for $B\in\Pkh$ {at end of episode $k$}\\
\hline
\multicolumn{2}{c}{Definitions used in the analysis}\\
\hline
$\Delta(\S)$ & Set of probability measures on $\S$ \\
$\dyad{\ell}$ & Set of dyadic cubes of $\S$ of diameter $2^{-\ell}$\\
$\S(\P_h^k)$ & Induced state partition from $\P_h^k$ \\
$\S(\P_h^k,x)$ & Region in $\S(\P_h^k)$ containing the point $x$ \\
$\S(B), \A(B)$ & Projection of a ball $B = B_\S \times B_\A$ to $B_\S$ and $B_\A$ accordingly\\
$\D(B)$ & The diameter of a ball $B$ \\
$\lev{B}$ & The depth in the tree of ball $B$, equivalent to $\log_2(\D(\S\times\A)/\D(B))$ \\
$R(K)$ & The regret up to episode $K$ \\
$\Exp{V_{h+1}(\hat{x}) \mid x,a}$ & $\mathbb{E}_{\hat{x} \sim \Pr_h(\cdot \mid x,a)} [V_{h+1}(\hat{x})]$\\
$\F_{k}$ & Sigma-field  generated by all information up to start of episode $k$\\
$\Expk{X}{k}$ & Expectation conditioned on information before episode $k$, i.e., $\Exp{X\mid \F_{k}}$\\
\hline
\end{tabular}
\caption{List of common notation}
\label{table:notation}
\end{table*}
    \section{Algorithm and Implementation}
\label{app:full_algo}

In this section we give the full pseudocode for implementing the algorithm, discuss the run-time and space complexity, and provide some discussion on other heuristic approaches to discretization.

\begin{algorithm*}[ht!]
	\begin{algorithmic}[1]
		\Procedure{AdaMB}{$\S, \A, \D, H, K, \pfail$}
			\State Initialize partitions $\P_h^0 = \S\times\A$ for $h\in[H]$, estimates $\Qhat{h}{0}(\cdot) = \Vhat{h}{k}(\cdot) = H-h+1$
			\For{each episode $k \gets 1, \ldots K$}
				\State Receive starting state $X_1^k$
				\For{each step $h \gets 1, \ldots, H$}
					\State Observe $X_h^k$ and determine $\relevant_h^k(X_h^k) = \{B\in\Pkh[k-1] \mid X_h^k \in B\}$
					\State Greedy selection rule: 
					pick $B_h^k = \argmax_{B \in \text{RELEVANT}_h^k(X_h^k)} \Qhat{h}{k-1}(B)$
                    \State Play action $A_h^k = \tilde{a}(B_h^k)$ associated with ball $B_h^k$; receive $R_h^k$ and transition to  $X_{h+1}^k$
					\State Update counts for $n_h^k(B_h^k), \rhat{h}{k}(B_h^k),$ and $\That{h}{k}(\cdot \mid B_h^k)$ via:
					\State $n_h^k(B_h^k) \leftarrow n_h^{k-1}(B_h^k) +1$
					\State $\rhat{h}{k}(B_h^k) \leftarrow \frac{(n_h^k(B_h^k) - 1)\rhat{h}{k}(B_h^k) + R_h^k}{n_h^k(B_h^k)}$
					\State $\That{h}{k}(A \mid B_h^k) = \frac{(n_h^k(B_h^k) - 1) \That{h}{k-1}(A \mid B_h^k) + \Ind{X_{h+1}^k \in A}}{n_h^k(B_h^k)}$ for $A \in \dyad{\lev{B_h^k}}$
					\If{$n_h^k(B_h^k) + 1\geq \nplus{B_h^k}$}
					  \textproc{Refine Partition}$(B_h^k)$
					\EndIf 
				\EndFor
				\textproc{Compute Estimates}$(B_h^k,R_h^k,X_{h+1}^k)_{h=1}^H$
			\EndFor
		\EndProcedure
		
		\Procedure{Refine Partition}{$B$, $h$, $k$}
		    \State Construct $\P(B) = \{B_1, \ldots, B_{2^{d}}\}$ a $2^{-(\lev{B}+1)}$-dyadic partition of $B$
		    \State Update  $\Pkh=\Pkh[k-1]\cup\P(B) \setminus B$ 
		    \State For each $B_i$, initialize $n_h^k(B_i)=0$, $\rhat{h}{k}(B_i) = 0$ and $\That{h}{k}(B_i) = 0$
		 \EndProcedure
		\Procedure{Compute Estimates}{$(B_h^k,R_h^k,X_{h+1}^k)_{h=1}^H$}
		\For{each $h \gets 1, \ldots H$ and $B \in \P_h^k$}
		    \State Construct $\rbar{h}{k}(B)$ and $\Tbar{h}{k}(B)$ by
		    \State $\rbar{h}{k}(B) = \frac{\sum_{B' \supseteq B} \rhat{h}{k}(B') n_h^k(B')}{\sum_{B' \supseteq B} n_h^k(B')}$
		    \State $\Tbar{h}{k}(A \mid B) = \frac{\sum_{B' \supseteq B} \sum_{A' \in \dyad{\lev{B'}}; A \subset A'}2^{-d_\S(\ell(B') - \ell(B))}n_h^k(B') \That{h}{k}(A' \mid B')}{\sum_{B' \supseteq B} n_h^k(B')}$ for $A \in \dyad{\lev{B}}$
		    \State Solve for $\Vhat{h+1}{k-1}(A)$ for every $A \in \dyad{\ell(B)}$ by
		    \[\Vhat{h+1}{k-1}(A) = \min_{A' \in \S(\P_{h+1}^{k-1})} \Vtilde{h+1}{k-1}(A') + L_V \D_\S(\tilde{x}(A), \tilde{x}(A'))\]
		    \State Set $\Qhat{h}{k}(B) = \rbar{h}{k}(B) + \rbonus{h}{k}(B) + \E_{A \sim \Tbar{h}{k}(\cdot \mid B)}[\Vhat{h+1}{k-1}(A)] + \tbonus{h}{k}(B)$
		\EndFor
		\For{each $h \gets 1, \ldots H$ and $A \in \S(\P_h^k)$}
		    \State Set $\Vtilde{h}{k}(A) = \min\{\Vtilde{h}{k-1}(A), \max_{B \in \P_h^k; \S(B) \supseteq A} \Qhat{h}{k}(B)\}$
		\EndFor
		\EndProcedure
	\end{algorithmic}
	\caption{Model-Based Reinforcement Learning with Adaptive Partitioning (\AdaMB)}
	\label{alg:full_brief}
\end{algorithm*}

\subsection{Implementation and Running Time}
\label{app:implementation_run_time}

Here we briefly discuss the oracle assumptions required for implementing the algorithm, and analyze the run-time and storage complexity.

\medskip

\noindent \textbf{Oracle Assumptions}: There are three main oracle assumptions needed to execute the algorithm.  In line 14 of \cref{alg:full_brief} we need access to a ``covering oracle'' on the metric space.  This oracle takes as input a ball $B \subset \S \times \A$ and outputs an $r$-covering of $B$.  This subroutine is easy in many metrics of interest (e.g. the Euclidean norm or any equivalent norms in $\mathbb{R}^d$) by just splitting each of the principle dimensions in half.  Second, we need to be able to compute $\S(B)$ for any $B \in \S \times \A$.  As our algorithm is maintaining a dyadic partition of the space, this subroutine is also simple to implement as each ball $B$ is of the form $\S(B) \times \S(A)$ and so the algorithm can store the two components separately.  Lastly, we require computing $\relevant_h^k(X)$.  By storing the partition as a tree, this subroutine can be implementing by traversing down the tree and checking membership at each step.  See the Github repository at \url{https://github.com/seanrsinclair/AdaptiveQLearning} for examples of implementing these methods.  \srsedit{Storing the discretization as a hash function would allow some of these access steps to be implemented in $O(1)$ time, with the downside being that splitting a region has a larger computational requirement.}

\medskip

\noindent \textbf{Storage Requirements}:  The algorithm maintains a partition $\P_h^k$ of $\S_h \times \A_h$ for every $h$, and the respective induced partition $\S(\P_h^k)$ whose size is trivially upper bounded by the size of the total partition.  Each element $B \in \P_h^k$ maintains four estimates.  The first three ($n_h^k(B)$, $\rhat{h}{k}(B)$, and $\Qhat{h}{k}(B)$) are linear with respect to the size of the partition.  The last one, $\That{h}{k}(\cdot \mid B)$ has size $|\dyad{\lev{B}}| \lesssim O(2^{d_\S \lev{B}}$).  Moreover, the algorithm also maintains estimate $\Vtilde{h}{k}(\cdot)$ over $\S(\P_h^k)$.  Clearly we have that the worst-case storage complexity arises from maintaining estimates of the transition kernels over each region in $\P_h^k$.  Thus we have that the total storage requirement of the algorithm is bounded above by \[ \sum_{h=1}^H \sum_{B \in \P_h^K} 2^{d_\S\lev{B}}.\]
Utilizing \cref{lem:LPbound} with $a_\ell = 2^{d_\S \ell}$ we find that the sum is bounded above by
\begin{align*}
     \sum_{h=1}^H \sum_{B \in \P_h^K} 2^{d_\S\lev{B}} & \leq \sum_{h=1}^H 2^{d \ell^\star} a_{\ell^\star} \\
     & \lesssim HK^{\frac{d+d_\S}{d+\gam}}.
\end{align*}
Plugging in the definition of $\gam$ from the splitting rule yields the results in \cref{tab:comparison_of_bounds}.

\medskip

\noindent \textbf{Run-Time}: We assume that the oracle access discussed occurs in constant time.  The inner loop of \cref{alg:full_brief} has four main steps.  Finding the set of relevant balls for a given state can be implemented in $\log_d(|\P_h^k|)$ time by traversing through the tree structure.  Updating the estimates and refining the partition occur in constant time by assumption on the oracle.  Lastly we need to update the estimates for $\Qhat{h}{k}$ and $\Vhat{h}{k}$.  Since the update only needs to happen for a constant number of regions (as only one ball is selected per step episode pair) the dominating term arises from computing the expectation over $\Tbar{h}{k}(\cdot \mid B_h^k)$.  Noting that the support of the distribution is $|\dyad{\lev{B_h^k}}| = 2^{d_\S \lev{B_h^k}}$ the total run-time of the algorithm is upper bounded by \[\sum_{h=1}^H \sum_{k=1}^K 2^{d_\S \lev{B_h^k}}.\]  Rewriting the sum we have
\begin{align*}
    \sum_{h=1}^H \sum_{k=1}^K 2^{d_\S \lev{B_h^k}} & \leq \sum_{h=1}^H \sum_{\ell \in \mathbb{N}} \sum_{B \in \P_h^K : \lev{B} = \ell} 2^{d_\S \ell} \sum_{k \in [K] : B_h^k = B} 1 \\
    & \lesssim \sum_{h=1}^H \sum_{\ell \in \mathbb{N}} \sum_{B \in \P_h^K : \lev{B} = \ell} 2^{d_\S \ell} n_+(B) \\
    & \lesssim \sum_{h=1}^H \sum_{\ell \in \mathbb{N}} \sum_{B \in \P_h^K : \lev{B} = \ell} 2^{d_\S \ell} \phi 2^{\gam \ell}.
\end{align*}
Utilizing \cref{lem:LPbound} with $a_\ell = 2^{(d_\S + \gam) \ell}$ we find that the sum is bounded above by
$H \phi 2^{d \ell^\star} a_{\ell^\star} \lesssim HK^{1 + \frac{d_\S}{d+\gam}}.$  Plugging in $\gam$ from the splitting rule yields the result in \cref{tab:comparison_of_bounds}.

\medskip

\srsedit{\noindent \textbf{Monotone Increasing Run-Time and Storage Complexity}: The run-time and storage complexity guarantees presented are monotonically increasing with respect to the number of episodes $K$.  However, to get sublinear minimax regret in a continuous setting for nonparametric Lipschitz models, the model complexity must grow over episodes.  In practice, one would run \AdaMB until running out of space - and our experiments show that \AdaMB uses resources (storage and computation) much better than a uniform discretization.  We are not aware of any storage-performance lower bounds, so this is an interesting future direction.}
	\section{Experiment Setup and Computing Infrastructure}

\noindent \textbf{Experiment Setup}: Each experiment was run with \srstodoedit{$200$} iterations where the relevant plots are taking the mean and a standard-normal confidence interval of the related quantities.  We picked a fixed horizon of $H = 5$ and ran it to \srstodoedit{$K = 2000$} episodes.  As each algorithm uses bonus terms of the form $c / \sqrt{t}$ where $t$ is the number of times a related region has been visited, we tuned the constant $c$ separately for each algorithm (for $c \in [.001, 10]$) and plot the results on the performance of the algorithm for the best constant $c$.

\medskip

\noindent \textbf{Fixed Discretization UCBVI}: We bench marked our adaptive algorithm against a fixed-discretization model-based algorithm \srsedit{with full and one-step planning}.  In particular, we implemented UCBVI from \cite{azar2017minimax} using a fixed discretization of the state-action space.  The algorithm takes as input a parameter $\epsilon$ and constructs an $\epsilon$-covering of $\S$ and $\A$ respectively.  It then runs the original UCBVI algorithm over this discrete set of states and actions.  The only difference is that when visiting a state $x$, as feedback to the algorithm, the agent snaps the point to its closest neighbour in the covering.

UCBVI has a regret bound of $H^{3/2}\sqrt{SAK} + H^4 S^2 A$ where $S$ and $A$ are the size of the state and action spaces.  Replacing these quantities with the size of the covering, we obtain $$H^{3/2} \sqrt{\epsilon^{-d_\S} \epsilon^{-d_\A} K} + H^4 \epsilon^{-2d_\S} \epsilon^{-d_\A}.$$
A rough calculation also shows that the discretization error is proportional to $H L K \epsilon$.  Tuning $\epsilon$ so as to balance these terms, we find that the regret of the algorithm can be upper bounded by

\[ L H^2 K^{2d/(2d + 1)}.\]

The major difference in this approach versus a uniform discretization of a model-free algorithm (e.g. \cite{song2019efficient}) is that in model-based algorithms the lower-order terms scale quadratically with the size of the state space.  In tabular settings, this term is independent of the number of episodes $K$.  However, in continuous settings the discretization depends on the number of episodes $K$ in order to balance the approximation error from discretizing the space uniformly.  See \cite{domingues2020regret} for a discussion on this dependence.

Obtaining better results for model-based algorithms with uniform discretization requires better understanding the complexity in learning the transition model, which ultimately leads to the terms which depend on the size of the state space.  The theoretical analysis of the concentration inequalities for the transitions in \cref{app:concentration} are min-max, showing that worst case dependence on the dimension of the state space is inevitable.  However, potential approaches could instead model bonuses over the value function instead of the transitions would lead to better guarantees~\cite{ayoub2020model}.  Our concentration inequalities on the transition kernels is a first-step at understanding this feature in continuous settings.

\begin{table}[!t]
\caption{Comparison of the average running time (in seconds) of the four different algorithms considered in the experimental results: \AdaMB (Algorithm~\ref{alg:brief}), \textsc{Adaptive Q-Learning}~\cite{Sinclair_2019}, \textsc{Net-Based Q-Learning}~\cite{song2019efficient}, and a \textsc{Fixed Discretization UCBVI}~\cite{azar2017minimax}.}
\label{tab:computation}
\setlength\tabcolsep{0pt} 
\centering
\begin{tabular*}{\columnwidth}{@{\extracolsep{\fill}}l|cccc}
\hline
  Problem  & \textsc{AdaMB} & \textsc{AdaQL} & \textsc{epsilonQL} & \textsc{epsilonMB}\\
\hline
\textsc{1 Ambulance\qquad} & 8.07 & 0.90 & 1.10 & 16.59 \\
\textsc{2 Ambulances} & 22.92 & 1.57 & 9.54 & 90.92 \\
\textsc{Oil Problem} & 5.63 & 1.31 & 2.21 & 20.27 \\
\hline
\end{tabular*}
\end{table}

\medskip

\noindent \textbf{Computing Infrastructure and Run-Time}: The experiments were conducted on a personal computer with an AMD Ryzen 5 3600 6-Core 3.60 GHz processor and 16.0GB of RAM. No GPUs were harmed in these experiments.  The average computation time for running a single simulation of an algorithm is listed in Table~\ref{tab:computation}.  As different hyperparameter settings result in similar run-times, we only show the three major simulations conducted \srsedit{with fixed bonus scaling $c = 1$}.  As to be expected, the adaptive algorithms ran much faster than their uniform discretization counterparts.  Moreover, the model-free methods have lower running time than the model-based algorithms.  These results mimic the run-time and space complexity discussed in \cref{tab:comparison_of_bounds}.
    \section{Regret Derivation}
\label{sec:finalregret}

In this section we combine all of the previous results to derive a final regret bound.  We first provide a bound on the expected regret for \AdaMB, before using a simple concentration inequality to obtain a high probability result.
\begin{theorem}
	\label{thm:regret_app}
	Let $d = d_A + d_S$, then the expected regret of \AdaMB for any sequence of starting states $\{X_1^k\}_{k=1}^K$ is upper bounded by
	\begin{align*}
	\Exp{R(K)} & \lesssim \begin{cases}
	LH^{1+\frac{1}{d+1}}K^{\frac{d+d_\S - 1}{d+d_\S}} \quad d_S > 2 \\
	LH^{1 + \frac{1}{d+1}}K^{1-\frac{1}{d+d_\S + 2}} \quad d_S \leq 2
	\end{cases}
	\end{align*}
	where $L = 1 + L_r + L_V + L_V L_T$ and $\lesssim$ omits poly-logarithmic factors of $\frac{1}{\delta}, H,K,$ $d$, and any universal constants.
\end{theorem}

\begin{proof}
	Using~\cref{lemma:regret_decomposition} we have that
	\begin{align*}
	\Exp{R(K)} & \leq \sum_{k=1}^K \sum_{h=1}^H \Exp{\Vtilde{h}{k-1}(\S(\P_h^{k-1},X_h^{k})) - \Vtilde{h}{k}(\S(\P_h^{k},X_h^k))} \\
	& + \sum_{h=1}^H \sum_{k=1}^K \Exp{2\rbonus{h}{k}(B_h^k)} + \sum_{h=1}^H \sum_{k=1}^K \Exp{2\tbonus{h}{k}(B_h^k)} + \sum_{k=1}^K \sum_{h=1}^H L_V \Exp{\D(B_h^k)}.
	\end{align*}
	We ignore the expectations, arguing a worst-case problem-independent bound on each of the quantities which appear in the summation.  At the moment, we leave the splitting rule defined in the algorithm description as $\nplus{\ell} = \phi2^{\gam \ell}$, where we specialize the regret bounds for the two cases at the end.  We also ignore all poly-logarithmic factors of $H$, $K$, $d$, and absolute constants in the $\lesssim$ notation.
	
	First note that via the splitting rule the algorithm maintains that for any selected ball $B$ we have that $\D(B) \leq (\phi / n_h^k(B))^{1/\gam}$.
	
	\noindent \textbf{Term One}: Using~\cref{lem:decreasing_process} we have that
	\[\sum_{k=1}^K \sum_{h=1}^H \Vtilde{h}{k-1}(\S(\P_h^{k-1}, X_h^k)) - \Vtilde{h}{k}(\S(\P_h^k, X_h^k)) \leq H^2 \max_{h} |\S(\P_h^k)|.\]
	However, using~\cref{lem:size_partition} we have that $|\S(\P_h^k)| \leq |\P_h^k| \leq 4^d \left(\frac{K}{\phi}\right)^{d / (d + \gam)}$.  Thus we can upper bound this term by $H^2 4^{d} \left(\frac{K}{\phi}\right)^{d / (d + \gam)} \lesssim H^2 K^{d/(d+\gam)} \phi^{-d/(d+\gam)}$.

	\noindent \textbf{Term Two and Four}: 
	\begin{align*}
	\sum_{h=1}^H \sum_{k=1}^K \rbonus{h}{k}(B_h^k) + L_V \D(B_h^k) & = \sum_{h=1}^H \sum_{k=1}^K \sqrt{\frac{8\log(2HK^2/\delta)}{\sum_{B' \supseteq B_h^k} n_h^k(B')}} + 4 L_r \D(B_h^k) + L_V \D(B_h^k)\\
	& \lesssim \sum_{h=1}^H \sum_{k=1}^K \sqrt{\frac{1}{n_h^k(B_h^k)}} + (L_r + L_V) \left(\frac{\phi}{n_h^k(B_h^k)}\right)^{\frac{1}{\gam}}
	\end{align*}
	where we used the definition of $\rbonus{h}{k}(B)$ and the splitting rule.
	
	Next we start by considering the case when $d_\S > 2$.
	
	\noindent \textbf{Term Three}:
	\begin{align*}
	\sum_{h=1}^H \sum_{k=1}^K \tbonus{h}{k}(B_h^k) & = \sum_{h=1}^H \sum_{k=1}^K (L_T + 1) L_V 4\D(B) + 4 L_V \sqrt{\frac{\log(HK^2 / \delta)}{\sum_{B' \subseteq B} n_h^k(B')}} \\
	&\qquad + \sum_{h=1}^H \sum_{k=1}^K L_T L_V \D(B) + c L_V \left(\sum_{B' \subseteq B} n_h^k(B')\right)^{-1/d_\S} \\
	& \lesssim \sum_{h=1}^H \sum_{k=1}^K (L_V L_T + L_V) \left(\frac{\phi }{n_h^k(B_h^k)} \right)^{\frac{1}{\gam}} + L_V \sqrt{\frac{1}{n_h^k(B_h^k)}} + L_V \left(n_h^k(B_h^k)\right)^{-1/d_\S}.
	\end{align*}
	
	where we used the definition of $\tbonus{h}{k}(B_h^k)$.  
	
	\noindent \textbf{Combining Terms}: We will take $\phi \geq 1$ in order to tune the regret bound in terms of $H$ and $\gam = d_\S$ in this situation.  Using this we find that the dominating term is of the form $(\phi / n_h^k(B_h^k))^{1 / \gam}$.  Thus we get that for $L = 1 + L_r + L_V + L_V L_T$,
	$$R(K) \lesssim H^2 \phi^{-\frac{d}{d + \gam}} K^{\frac{d}{d+\gam}} + L \phi^{\frac{1}{\gam}} \sum_{h=1}^H \sum_{k=1}^K \left(\frac{1}{n_h^k(B_h^k)} \right)^\frac{1}{\gam}.$$
	We now use~\cref{lem:countbound} for the case when $\alpha = \frac{1}{\gam}$ and $\beta = 0$.  This satisfies the required conditions of the result and we get:
	\begin{align*}
	R(K) & \lesssim H^2 \phi^{-\frac{d}{d + \gam}} K^{\frac{d}{d + \gam}} + H L \phi^{\frac{1}{\gam}} \phi^{-\frac{d}{\gam(d+\gam)}}K^{\frac{d+\gam - 1}{d + \gam}} \\
	& = H^2 \phi^{-\frac{d}{d + \gam}} K^{\frac{d}{d + \gam}} + H L \phi^{\frac{1}{d+\gam}}K^{\frac{d+\gam - 1}{d + \gam}}.
	\end{align*}
	Taking $\phi$ as $\phi = H^{\frac{d+\gam}{d+1}} \geq 1$ and plugging in $\gam = d_\S$ we see that $$R(K) \lesssim LH^{1+\frac{1}{d+1}}K^{\frac{d+d_\S - 1}{d+d_\S}}.$$

	
	
	Next we consider the case when $d_S \leq 2$.  The first two terms and the fourth term remain the same, whereby now in the third term we have:
	\begin{align*}
	\sum_{h=1}^H \sum_{k=1}^K \tbonus{h}{k}(B_h^k) & = \sum_{h=1}^H \sum_{k=1}^K L_V \left((5 L_T + 6) \D(B_h^k) + 4 \sqrt{\frac{\log(HK^2)}{\sum_{B' \supseteq B_h^k} n_h^k(B')}} + c \sqrt{\frac{2^{d_\S \lev{B_h^k}}}{\sum_{B' \supseteq B_h^k} n_h^k(B')}}\right)\\
	& \lesssim \sum_{h=1}^H \sum_{k=1}^K L_V(1 + L_T) \left(\frac{\phi }{n_h^k(B_h^k)}\right)^{1/\gam} + L_V \sqrt{\frac{1}{n_h^k(B_h^k)}} + L_V \sqrt{\frac{2^{d_\S \ell(B_h^k)}}{n_h^k(B_h^k)}}.
	\end{align*}
	
	\noindent \textbf{Combining Terms}: Again using that we take $\phi \geq 1$ we can combine terms to get:
	$$R(K) \lesssim H^2 \phi^{-\frac{d}{d+\gam}} K^{\frac{d}{d+\gam}} + L \sum_{h=1}^H \sum_{k=1}^K \left(\frac{1}{n_h^k(B_h^k)}\right)^{\frac{1}{\gam}} + L \sum_{h=1}^H \sum_{k=1}^K \sqrt{\frac{2^{d_\S \lev{B_h^k}}}{n_h^k(B_h^k)}}.$$
	
	Again using \cref{lem:countbound} for the case when $\gam = d_\S + 2$ which satisfies the requirements we get
	\begin{align*}
	R(K) & \lesssim H^2 \phi^{-\frac{d}{d+\gam}} K^{\frac{d}{d+\gam}} + L H \phi^{\frac{1}{d+\gam}} K^{\frac{d+\gam - 1}{d + \gam}} + L H \phi^{-\frac{d}{2(d+\gam)}} K^{\frac{d+\frac{1}{2}\gam + \frac{1}{2}}{d+\gam}} \\
	& \lesssim H^2 \phi^{-\frac{d}{d+\gam}} K^{\frac{d}{d+\gam}} + L H \phi^{\frac{1}{d+\gam}} K^{\frac{d+\gam - 1}{d+\gam}}.
	\end{align*}
	where we used the fact that the second term dominates the third when $\gam = d_\S + 2$.  Taking $\phi$ the same as the previous case we get:
	$$R(K) \lesssim LH^{1 + \frac{1}{d + 1}}K^{\frac{d+d_\S + 1}{d+d_\S + 2}}.$$
\end{proof}

Using this bound on the expected regret and a straightforward use of Azuma-Hoeffding's inequality we can show the following:
\begin{theorem}
	\label{thm:regret_app_prob}
	Let $d = d_A + d_S$, then the regret of \AdaMB for any sequence of starting states $\{X_1^k\}_{k=1}^K$ is upper bounded with probability at least $1 - \delta$ by
	\begin{align*}
	R(K) & \lesssim \begin{cases}
	LH^{1+\frac{1}{d+1}}K^{\frac{d+d_\S - 1}{d+d_\S}} \quad d_S > 2 \\
	LH^{1 + \frac{1}{d+1}}K^{\frac{d+d_\S + 1}{d+d_\S + 2}} \quad d_S \leq 2
	\end{cases}
	\end{align*}
	where $L = 1 + L_r + L_V + L_V L_T$ and $\lesssim$ omits poly-logarithmic factors of $\frac{1}{\delta}, H,K,$ $d$, and any universal constants.
\end{theorem}

\begin{proof}
	Let $R(K) = \sum_{k=1}^K V_1^\star(X_1^k) - V_1^{\pi^k}(X_1^k)$ be the true regret of the algorithm.  We apply Azuma-Hoeffding's inequality, where we use \cref{thm:regret_app} to find a bound on its expectation.  Keeping the same notation as before, let $Z_\tau = \sum_{k=1}^\tau V_1^\star(X_1^k) - V_1^{\pi^k} - \Exp{\sum_{k=1}^\tau V_1^\star(X_1^k) - V_1^{\pi^k}(X_1^k)}.$  Clearly we have that $Z_\tau$ is adapted to the filtration $\F_\tau$, and has finite absolute moments.  Moreover, using the fact that the value function is bounded above by $H$ then
	\begin{align*}
	|Z_\tau - Z_{\tau - 1}| & = |V_1^\star(X_1^\tau) - V_1^{\pi^\tau}(X_1^\tau) - \Exp{V_1^\star(X_1^\tau) - V_1^{\pi^\tau}(X_1^\tau)}| \\
	& \leq 4H.
	\end{align*}
	Thus we get, via a straightforward application of Azuma-Hoeffding's that with probability at least $1 - \delta$,
	\begin{align*}
	R(K) & \leq \Exp{R(K)} + \sqrt{32H^2K\log(1/\delta)} \\
	& \lesssim \begin{cases}
	LH^{1+\frac{1}{d+1}}K^{\frac{d+d_\S - 1}{d+d_\S}} \quad d_S > 2 \\
	LH^{1 + \frac{1}{d+1}}K^{1-\frac{1}{d+d_\S + 2}} \quad d_S \leq 2.
	\end{cases}
	\end{align*}  
	\end{proof}

	\section{Proofs for Technical Results}
\label{app:techproofs}

Finally we provide some additional proofs of the technical results we use in our regret analysis.

\begin{proof}[Proof of~\cref{lem:sum_ancestors}]
Recall we want to show that for any ball $B$ and $h, k \in [H] \times [K]$ we have 
\begin{align*}
    \frac{\sum_{B' \supseteq B} \D(B') n_h^k(B')}{\sum_{B' \supseteq B} n_h^k(B')} \leq 4 \D(B).
\end{align*}
First notice that the term on the left hand side can be rewritten as:
\begin{align*}
\frac{\sum_{B' \supseteq B} \D(B') n_h^k(B')}{\sum_{B' \supseteq B} n_h^k(B')} & = \frac{1}{t} \sum_{i=1}^t \D(B_h^{k_i})
\end{align*}
where $t = \sum_{B' \supseteq B} n_h^k(B')$ is the number of times $B$ or its ancestors were selected and $k_1, \ldots, k_t$ are the episodes for which they were selected.  Using the fact that $\D(B_h^{k_i})$ are decreasing over time as the partition is refined, this average can be upper bounded by only averaging over the ancestors of $B$, i.e. 
\begin{align*}
\frac{\sum_{B' \supseteq B} \D(B') n_h^k(B')}{\sum_{B' \supseteq B} n_h^k(B')} \leq \frac{\sum_{B' \supsetneq B} \D(B')n_h^k(B')}{\sum_{B'\supsetneq B} n_h^k(B')}.
\end{align*}
Using the splitting threshold $\nplus{B} = \phi2^{\gamma\lev{B}}$, we can upper bound this quantity by
\begin{align*}
    \frac{\sum_{B' \supsetneq B} n_h^k(B')\D(B')}{\sum_{B'\supsetneq B} n_h^k(B')} & = \frac{\sum_{i=0}^{\lev{B} - 1} 2^{-i} \phi 2^{\gam i}}{\sum_{i=0}^{\lev{B} - 1} \phi 2^{\gam i}} \\
    &\leq \frac{2^{(\gam-1)(\lev{B}-1)} \sum_{i=0}^{\infty} 2^{-(\gam - 1)i}}{2^{\gam (\lev{B} - 1)}} \\
    &\leq \frac{2 \cdot 2^{(\gam-1)(\lev{B}-1)}}{2^{\gam (\lev{B} - 1)}} \qquad\text{ because $2^{-(\gam-1)} \leq \tfrac12$}\\
    &= 4 \cdot 2^{-\lev{B}} = 4 \D(B).
\end{align*}
\end{proof}

\begin{proof}[Proof of~\cref{lemma:transition_confidence}, for $d_\S > 2$]

Let $h,k \in [H] \times [K]$ and $B \in \P_h^k$ be fixed and $(x,a) \in B$ be arbitrary.  
We use a combination of Proposition 10 and 20 from \cite{weed2019sharp}.  Let $P_0 = T_h(\cdot \mid x_0,a_0)$ where $(x_0, a_0) = (\tilde{x}(B), \tilde{a}(B))$ is the center of the ball $B$.  Our goal then is to come up with concentration between the one-Wasserstein metric of $\Tbar{h}{k}(\cdot \mid B)$ and $T_h(\cdot \mid x,a)$.  We break the proof down into four stages, where we show concentration between the one-Wasserstein distance of various measures. As defined, $\Tbar{h}{k}(\cdot \mid B)$ is a distribution over $\dyad{\ell(B)}$, the uniform discretization of $\S$ at over balls with diameter $2^{-\ell(B)}$.  However, we will view $\Tbar{h}{k}(\cdot \mid B)$ as a distribution over a set of finite points in $\S$, where 
\begin{align*}
    \Tbar{h}{k}(x \mid B) = \Tbar{h}{k}(A \mid B) \quad \text{ if } x = \tilde{x}(A).
\end{align*}

\noindent \textbf{Step One}:  Let $\tilde{T}_h^k(\cdot \mid B)$ be the true empirical distribution of all samples collected from $B'$ for any $B'$ which is an ancestor of $B$, i.e.
\begin{align}
    \tilde{T}_h^k(\cdot \mid B) = \frac{\sum_{B' \supseteq B} \sum_{k' \leq k} \delta_{X_{h+1}^{k'}} \Ind{B_h^{k'} = B'}}{\sum_{B' \supseteq B} n_h^k(B')}. \label{eq:Wass_step1}
\end{align}


Let $A_{h+1}^{k'}$ denote the region in $\dyad{\ell(B_h^{k'})}$ containing the point $X_{h+1}^{k'}$. Recall $\Tbar{h}{k}(\cdot \mid B)$ is the distribution defined according to:
\begin{align*}
    \Tbar{h}{k}(\cdot \mid B) = \frac{\sum_{B' \supseteq B} \sum_{k' \leq k} \Ind{B_h^{k'} = B'}\sum_{A \in \dyad{\ell(B)} : A \subseteq A_{h+1}^{k'}}2^{-d_\S(\ell(B') - \ell(B))}\delta_{\tilde{x}(A)}}{\sum_{B' \supseteq B} n_h^k(B')}.
\end{align*}

We can verify that $\sum_{A \in \dyad{\lev{B}} : A \subseteq A_{h+1}^{k'}}2^{-d_\S(\ell(B') - \ell(B))} = 1$ as the number of regions in $\dyad{\lev{B}}$ which contain any region in $\dyad{\lev{B'}}$ is exactly $2^{d_\S(\ell(B') - \ell(B))}$. Furthermore $X_{h+1}^{k'}$ and $\tilde{x}(A)$ are both contained in $A_{h+1}^{k'}$ so that
$\D_\S(X_{h+1}^{k'}, \tilde{x}(A)) \leq \D_S(A_{h+1}^{k'}) \leq \D(B_h^{k'})$, where we use the definition of $\dyad{\ell(B_h^{k'})}$ for the last inequality.
Using these observations, it follows that
\begin{align*}
    d_W(\Tbar{h}{k}(\cdot \mid B), \tilde{T}_h^k(\cdot \mid B)) & \leq \frac{\sum_{B' \supseteq B} \sum_{k' \leq k} \Ind{B_h^{k'} = B'}}{\sum_{B' \supseteq B} n_h^k(B')} \sum_{A \in \dyad{\lev{B}} : A \subseteq A_{h+1}^{k'}}2^{-d_\S(\ell(B') - \ell(B))}\D_\S(X_{h+1}^{k'}, \tilde{x}(A)) \\
    &\leq \frac{\sum_{B' \supseteq B} \sum_{k' \leq k} \Ind{B_h^{k'} = B'}}{\sum_{B' \supseteq B} n_h^k(B')}
    \sum_{A \in \dyad{\lev{B}} : A \subseteq A_{h+1}^{k'}} 2^{-d_\S(\ell(B') - \ell(B))} \D_\S(A_{h+1}^{k'}) \\
    &\leq \frac{\sum_{B' \supseteq B} \sum_{k' \leq k} \Ind{B_h^{k'} = B'} \D(B_h^{k'})}{\sum_{B' \supseteq B} n_h^k(B')} \\
    &\leq \frac{\sum_{B' \supseteq B} \D(B') n_h^k(B') }{\sum_{B' \supseteq B} n_h^k(B')}
\end{align*}

\noindent \textbf{Step Two}:
Next we bound the difference between $\tilde{T}_h^k(\cdot \mid B)$ and $\tilde{T}_h(\cdot \mid x_0, a_0)$ where $\tilde{T}_h(\cdot \mid x_0, a_0)$ is a `ghost empirical distribution' of samples whose marginal distribution is $T_h(\cdot \mid x_0, a_0)$.
By Lipschitzness of the transition kernels, for every $x,a,x_0,a_0$,
\[ d_W(T_h(\cdot \mid x, a), T_h(\cdot \mid x_0, a_0)) \leq L_T \D((x,a), (x_0,a_0)). \]
Using the coupling definition of the Wasserstein metric, there exists a family of distributions $\xi(\cdot, \cdot | x,a,x_0,a_0)$ parameterized by $x,a,x_0,a_0$ such that 
\[\E_{(Z,Y) \sim \xi(\cdot, \cdot | x,a,x_0,a_0)}[\D_\S(Z, Y)]= d_W(T_h(\cdot \mid x, a), T_h(\cdot \mid x_0, a_0)) \leq L_T \D((x,a), (x_0,a_0)),\]
whose marginals are
\[\int_\S \xi(z, y | x,a,x_0,a_0) dy = T_h(z \mid x, a)
~\text{ and }~ \int_\S \xi(z, y | x,a,x_0,a_0) dz = T_h(y \mid x_0, a_0).\]
For $(Z,Y) \sim \xi(\cdot, \cdot | x,a,x_0,a_0)$, let $\xi'(\cdot | z, x,a,x_0,a_0)$ denote the conditional distribution of $Y$ given $Z$, such that 
\begin{align}
\xi(z, y | x,a,x_0,a_0) = T_h(z \mid x, a) \xi'(y | z, x,a,x_0,a_0). \label{eq:Wass_step2_c}
\end{align}


For ease of notation let us denote
$t = \sum_{B' \supseteq B} n_h^k(B')$ and let the indexing $k_1, \ldots, k_t$ be the episodes for which $B$ or its ancestors were selected by the algorithm. For the sequence of samples $\{(X_{h}^{k_i}, A_{h}^{k_i}, X_{h+1}^{k_i})\}_{i\in[t]}$ realized by our algorithm, consider a `ghost sample' $Y_1, \ldots, Y_t$ such that $Y_i \sim \xi'(\cdot | X_{h+1}^{k_i},X_{h}^{k_i}, A_{h}^{k_i},x_0,a_0)$ for $i \in [t]$. Let $\tilde{T}_h(\cdot \mid x_0, a_0)$ denote the empirical distribution of these samples such that 
\[\tilde{T}_h(\cdot \mid x_0, a_0) = \frac{1}{t} \sum_{i=1}^t \delta_{Y_i}
~~\text{ and recall by definition }~~ \tilde{T}_h^k(\cdot \mid B) = \frac{1}{t} \sum_{i=1}^t \delta_{X_{h+1}^{k_i}}.\]

Using the definition of the Wasserstein distance we have that
\begin{align}
    d_W(\tilde{T}_h^k(\cdot \mid B), \tilde{T}_h(\cdot \mid x_0,a_0)) & \leq \frac{1}{t} \sum_{i=1}^t \D_\S( X_{h+1}^{k_i}, Y_i). \label{eq:Wass_step2_a}
\end{align}
We will use Azuma-Hoeffding's to provide a high probability bound on this term by its expectation.  For any $\tau \leq K$ define the quantity 
\begin{align*}
    Z_\tau = \sum_{i=1}^\tau \D_\S(X_{h+1}^{k_i}, Y_i) - \Exp{\D_\S(X_{h+1}^{k_i}, Y_i)}.
\end{align*}
Let $\F_i$ be the filtration containing $\F_{k_i + 1} \cup \{Y_j\}_{j \leq i}$.  It follows that $Z_\tau$ is a martingale with respect to $\F_\tau$.  The process is adapted to the filtration by construction, has finite first moment, and we have that
\begin{align*}
    \Exp{Z_\tau \mid \F_{\tau - 1}} & = Z_{\tau - 1} + \Exp{\D_\S(X_{h+1}^{\tau}, Y_\tau)} - \Exp{\D_\S(X_{h+1}^{\tau}, Y_\tau)} = Z_{\tau - 1}.
\end{align*}
Moreover, we also have the differences are bounded by
\begin{align*}
    \left|Z_\tau - Z_{\tau - 1}\right| & = \left|\D_\S(X_{h+1}^{k_\tau}, Y_\tau) - \Exp{\D_\S(X_{h+1}^{k_\tau}, Y_\tau)}\right|
    \leq 2
\end{align*}
since by assumption $\D_\S(\S) \leq 1$. By Azuma-Hoeffding's inequality, with probability at least $1 - \frac{\delta}{HK^2}$,
\begin{align}
    \frac{1}{\tau} \sum_{i=1}^\tau \D_\S(Y_i, X_{h+1}^{k_i}) & \leq \Exp{\frac{1}{\tau} \sum_{i=1}^\tau \D_\S(Y_i, X_{h+1}^{k_i})} + \sqrt{\frac{8\log(HK^2/\delta)}{\tau}}. \label{eq:Wass_step2_b}
\end{align}
Moreover, by construction of the ghost samples we have that
\begin{align*}
    \frac{1}{\tau} \sum_{i=1}^\tau \Exp{\D_\S(Y_i, X_{h+1}^{k_i})} & = \frac{1}{\tau} \sum_{i=1}^\tau \Exp{d_W(T_h(\cdot \mid X_{h}^{k_i}, A_h^{k_i}), T_h(\cdot \mid x_0, a_0))} \\
    & \leq \frac{1}{\tau} \sum_{i=1}^\tau L_V \D(B_{h+1}^{k_i})
\end{align*}
since $x_0, a_0$ is in the ball $B$ which is contained in the ball $B_{h+1}^{k_i}$. By plugging this into \cref{eq:Wass_step2_b}, taking a union bound over the number of steps $H$, the number of episodes $K$, the number of potential stopping times $K$, and combining it with \cref{eq:Wass_step2_a} and using the construction of $t$, it follows that with probability at least $1 - \delta$, for all $h,k,B$
\begin{align*}
    d_W(\tilde{T}_h^k(\cdot \mid B), \tilde{T}_h(\cdot \mid x_0, a_0)) & \leq L_T\frac{\sum_{B' \supseteq B} n_h^k(B') \D(B')}{\sum_{B' \supseteq B} n_h^k(B')} + \sqrt{\frac{8\log(HK^2/\delta)}{\sum_{B' \supseteq B} n_h^k(B')}}.
\end{align*}
Note that we do not need to union bound over all balls $B \in \P_h^k$ as the estimate of only one ball is changed per (step, episode) pair, i.e. $\That{h}{k}(B)$ and correspondingly $\Tbar{h}{k}(B)$ is changed for only a single ball $B = B_h^k$ per episode. For all balls not selected, it inherits the concentration of the good event from the previous episode because its estimate does not change. Furthermore, even if ball $B$ is ``split'' in episode $k$, all of its children inherit the value of the parent ball, and thus also inherits the good event, so we still only need to consider the update for $B_h^k$ itself.

\noindent \textbf{Step Three}: Next we bound $d_W(\tilde{T}_h(\cdot \mid x_0, a_0), T_h(\cdot \mid x_0, a_0))$. Recall $\F_i$ is the filtration containing $\F_{k_i + 1} \cup \{Y_j\}_{j \leq i}$. Note that the joint distribution over $\{(X_{h}^{k_i}, A_{h}^{k_i}, X_{h+1}^{k_i},Y_i)\}_{i\in[t]}$ is given by
\[G_t(\{(X_{h}^{k_i}, A_{h}^{k_i}, X_{h+1}^{k_i},Y_i)\}_{i\in[t]}) = \prod_{i=1}^t (P(X_{h}^{k_i}, A_h^{k_i} ~|~ \F_{i-1}) T_h(X_{h+1}^{k_i} | X_{h}^{k_i}, A_h^{k_i}) \xi'(Y_i | X_{h+1}^{k_i}, X_{h}^{k_i}, A_h^{k_i}, x_0, a_0),\]
where $P(X_{h}^{k_i}, A_h^{k_i} ~|~ \F_{i-1})$ is given by the dynamics of the MDP along with the policy that the algorithm plays. Then we have 
\begin{align*}
&\int_{\S \times \A \times \S} G_t(\{(X_{h}^{k_i}, X_{h}^{k_i}, X_{h+1}^{k_i},Y_i)\}_{i\in[t]}) dX_{h}^{k_t} dA_{h}^{k_t} dX_{h+1}^{k_t} \\
&= G_{t-1}(\{(X_{h}^{k_i}, X_{h}^{k_i}, X_{h+1}^{k_i},Y_i)\}_{i\in[t-1]}) \\
&\qquad \cdot \int_{\S \times \A} P(X_{h}^{k_t}, A_h^{k_t} ~|~ \F_{k_{t-1}}) \left(\int_{\S} \xi(X_{h+1}^{k_i}, Y_i | X_{h}^{k_i}, A_h^{k_i}, x_0, a_0) dX_{h+1}^{k_t} \right) dX_{h}^{k_t} dA_{h}^{k_t} \\
&= G_{t-1}(\{(X_{h}^{k_i}, X_{h}^{k_i}, X_{h+1}^{k_i},Y_i)\}_{i\in[t-1]}) T_h(Y_i | x_0, a_0) \int_{\S \times \A} P(X_{h}^{k_t}, A_h^{k_t} ~|~ \F_{k_{t-1}}) dX_{h}^{k_t} dA_{h}^{k_t} \\
&= G_{t-1}(\{(X_{h}^{k_i}, X_{h}^{k_i}, X_{h+1}^{k_i},Y_i)\}_{i\in[t-1]}) T_h(Y_i | x_0, a_0).
\end{align*}
By repeating this calculation, we can verify that the marginal distribution of $Y_1 \dots Y_t$ is $\prod_{i \in [t]} T_h(Y_i | x_0, a_0)$. Following Proposition 10 and 20 from \cite{weed2019sharp} for the case when $d_\S > 2$ we have that with probability at least $1 - \delta / HK^2$ for some universal constant $c$,
\begin{align*}
    d_W(\tilde{T}_h(\cdot \mid x_0, a_0), T_h(\cdot \mid x_0, a_0)) & \leq \Exp{d_W(\tilde{T}_h(\cdot \mid x_0, a_0), T_h(\cdot \mid x_0, a_0)} + \sqrt{\frac{\log(HK^2 / \delta)}{\sum_{B' \subseteq B} n_h^k(B')}} \\
    & \leq c \left(\sum_{B' \subseteq B} n_h^k(B')\right)^{-1/d_\S}+ \sqrt{\frac{\log(HK^2 / \delta)}{\sum_{B' \subseteq B} n_h^k(B')}}.
\end{align*}

\noindent \textbf{Step Four}: Using the assumption that $T_h$ is Lipschitz and $(x_0, a_0)$ and $(x,a) \in B$ we have that

$$d_W(T_h(\cdot \mid x,a), T_h(\cdot \mid x_0, a_0)) \leq L_T \D((x,a),(x_0, a_0)) \leq L_T \D(B).$$

Putting all of the pieces together we get that
\begin{align*}
&d_W(\Tbar{h}{k}(\cdot \mid B), T_h(\cdot \mid x,a)) \\
&\qquad \leq d_W(\Tbar{h}{k}(\cdot \mid B), \tilde{T}_h^k(\cdot \mid B)) + d_W(\tilde{T}_h^k(\cdot \mid B), \tilde{T}_h(\cdot \mid x_0, a_0)) \\
&\qquad\qquad + d_W(\tilde{T}_h(\cdot \mid x_0, a_0), T_h(\cdot \mid x_0, a_0)) + d_W(T_h(\cdot \mid x_0, a_0), T_h(\cdot \mid x, a)) \\
&\qquad \leq \frac{\sum_{B' \supseteq B}n_h^k(B')\D(B')}{\sum_{B' \supseteq B} n_h^k(B')} + \frac{\sum_{B' \subseteq B} L_T n_h^k(B') \D(B')}{\sum_{B' \subseteq B} n_h^k(B')} + \sqrt{\frac{8\log(HK^2/\delta)}{\sum_{B' \subseteq B} n_h^k(B')}} \\
&\qquad\qquad + L_T \D(B) + c \left(\sum_{B' \subseteq B} n_h^k(B')\right)^{-1/d_\S}+ \sqrt{\frac{\log(HK^2 / \delta)}{\sum_{B' \subseteq B} n_h^k(B')}} \\
&\qquad = (L_T + 1) \frac{\sum_{B' \supseteq B}n_h^k(B')\D(B')}{\sum_{B' \supseteq B} n_h^k(B')} + 4 \sqrt{\frac{\log(HK^2 / \delta)}{\sum_{B' \subseteq B} n_h^k(B')}} 
 + L_T \D(B) + c \left(\sum_{B' \subseteq B} n_h^k(B')\right)^{-1/d_\S} \\
&\qquad \leq (5 L_T + 4) \D(B) + 4 \sqrt{\frac{\log(HK^2 / \delta)}{\sum_{B' \subseteq B} n_h^k(B')}} 
 + c \left(\sum_{B' \subseteq B} n_h^k(B')\right)^{-1/d_\S} \text{ by Lemma~\ref{lem:sum_ancestors}}
\end{align*}
The result then follows via a union bound over $H$, $K$, the $K$ possible values of the random variable $n_h^k(B)$.  Per usual we do not need to union bound over the number of balls as the estimate of only one ball is updated per iteration.
\end{proof}

The second concentration inequality deals with the case when $d_\S \leq 2$.  
The constant $c$ in Proposition 10 from \cite{weed2019sharp} becomes very large when $d_\S \to 2$, and thus we instead use the fact that $\Tbar{h}{k}(\cdot \mid B)$ has finite support over $2^{d_\S \ell(B)}$ points and consider Wasserstein convergence of empirical distributions sampled from discrete distributions.  $T_h(\cdot \mid x,a)$ is still a (potentially) continuous distribution so we need to change Step 3 of the above argument slightly.

\begin{proof}[Proof of \cref{lemma:transition_confidence}, for $d_\S \leq 2$]
Let $h,k \in [H] \times [K]$ and $B \in \P_h^k$ be fixed with $(x,a) \in B$ arbitrary.  We use a combination of Proposition 10 and 20 from \cite{weed2019sharp} for the case when when the distributions have finite support.  As before, let $(x_0, a_0) = (\tilde{x}(B), \tilde{a}(B))$ be the center of the ball $B$.  We again break the proof down into several stages, where we show concentration between the Wasserstein distance of various measures.  In order to obtain bounds that scale with the support of $\Tbar{h}{k}(\cdot \mid B)$ we consider ``snapped'' versions of the distributions, where we snap the resulting random variable to its point in the discretization of $\dyad{\lev{B}}$. We repeat the same first two steps as Lemma~\ref{lemma:transition_confidence} which are restated again here for completeness.  

\noindent \textbf{Step One}:  Let $\tilde{T}_h^k(\cdot \mid B)$ be the true empirical distribution of all samples collected from $B'$ for any $B'$ which is an ancestor of $B$, formally defined in \cref{eq:Wass_step1}. By the same argument as Step 1 in the proof of Lemma~\ref{lemma:transition_confidence} it follows that
\[
d_W(\Tbar{h}{k}(\cdot \mid B), \tilde{T}_h^k(\cdot \mid B)) \leq \frac{\sum_{B' \supseteq B} n_h^k(B') \D(B')}{\sum_{B' \supseteq B} n_h^k(B')}
\]

\noindent \textbf{Step Two}:
Let $\tilde{T}_h(\cdot \mid x_0, a_0)$ be a `ghost empirical distribution' of samples whose marginal distribution is $T_h(\cdot \mid x_0, a_0)$. It consists of $t = \sum_{B' \supseteq B} n_h^k(B')$ samples drawn from $Y_i \sim \xi'(\cdot|X_{h+1}^{k_i}, X_h^{k_i}, A_h^{k_i},x_0,a_0)$ as constructed in \cref{eq:Wass_step2_c}. By the same argument from Step 2 of the proof of Lemma~\ref{lemma:transition_confidence}, with probability at least $1 - \delta$, for all $h,k,B$
\begin{align*}
    d_W(\tilde{T}_h^k(\cdot \mid B), \tilde{T}_h(\cdot \mid x_0, a_0)) & \leq L_T\frac{\sum_{B' \supseteq B} n_h^k(B') \D(B')}{\sum_{B' \supseteq B} n_h^k(B')} + \sqrt{\frac{8\log(HK^2/\delta)}{\sum_{B' \supseteq B} n_h^k(B')}}.
\end{align*}

\noindent \textbf{Step Three}:
Next we let $\tilde{T}_h^{\lev{B}}(\cdot \mid x_0, a_0)$ to be the snapped empirical distribution of the ghost samples $Y_1 \dots Y_t$ to their nearest point in $\dyad{\lev{B}}$. Denote $\tilde{Y}_i$ as $\tilde{x}(A_i)$ where $A_i \in \dyad{\lev{B}}$ is the region containing the point $Y_i$. It follows that:
\[ \tilde{T}_h^{\lev{B}}(\cdot \mid x_0, a_0) = \frac{1}{t} \sum_{i=1}^t \sum_{A \in \dyad{\lev{B}}} \Ind{Y_i \in A} \delta_{\tilde{x}(A)} = \frac{1}{t}\sum_{i=1}^t \delta_{\tilde{Y}_i}. \]
Since each of the points are moved by at most $\D_\S(A_i)\leq \D(B)$ by construction of $\tilde{T}_h^{\lev{B}}$ and $\dyad{\lev{B}}$, we have that $d_W(\tilde{T}_h^{\lev{B}}(\cdot \mid x_0, a_0), \tilde{T}_h(\cdot \mid x_0, a_0)) \leq \D(B)$.

Define the snapped distribution $T_h^{\lev{B}}(\cdot \mid x_0, a_0)$ according to
\[T_h^{\lev{B}}(x \mid x_0, a_0) = \sum_{A \in \dyad{\lev{B}}} \Ind{x = \tilde{x}(A)} \int_{A} T_h(y \mid x_0, a_0) dy\]
where we note that this distribution has finite support of size $2^{-d_\S \lev{B}}$ over the set $\{\tilde{x}(A)\}_{A \in \dyad{\lev{B}}}$.

By the same argument from Step 3 of the proof of Lemma~\ref{lemma:transition_confidence}, it holds that by construction, the marginal distribution of $Y_1 \dots Y_t$ denoted $f_{Y_1 \dots Y_t}$ is $\prod_{i \in [t]} T_h(Y_i | x_0, a_0)$. Furthermore, conditioned on $(Y_1 \dots Y_t)$, the snapped samples $(\tilde{Y}_1 \dots \tilde{Y}_t)$ are fully determined. Recall that $\tilde{Y}_i$ can only take values in $\{\tilde{x}(A)\}_{A \in \dyad{\lev{B}}}$. If $A_i$ refers to the set in $\dyad{\lev{B}}$ for which $\tilde{Y}_i = \tilde{x}(A_i)$, then 
\begin{align*}
P(\tilde{Y}_1 \dots \tilde{Y}_t) &= P(Y_1 \in A_1, \dots Y_t \in A_t) \\
&= \int_{A_1} \int_{A_2} \cdots \int_{A_t} f_{Y_1 \dots Y_t}(y_1 \dots y_t) dy_t \cdots dy_1 \\
&=\int_{A_1} \int_{A_2} \cdots \int_{A_t} \prod_{i \in [t]} T_h(Y_i | x_0, a_0) dy_t \cdots dy_1 \\
&= \prod_{i \in [t]} \int_{A_i} T_h(Y_i | x_0, a_0) dy_i \\
&= T_h^{\lev{B}}(\tilde{Y}_i | x_0, a_0).
\end{align*}
such that the marginal distribution of $\tilde{Y}_1 \dots \tilde{Y}_t$ is equivalent to that of a set of $t$ i.i.d. samples from $T_h^{\lev{B}}(\cdot | x_0, a_0)$.

By Proposition [13] and [20] from from \cite{weed2019sharp}, for some universal constant $c$, with probability at least $1 - \frac{\delta}{HK^2}$,
\begin{align*}
    &d_W(\tilde{T}_h^{\lev{B}}(\cdot \mid x_0, a_0), T_h^{\lev{B}}(\cdot \mid x_0, a_0)) \\
    \leq& \Exp{d_W(\tilde{T}_h^{\lev{B}}(\cdot \mid x_0, a_0), \tilde{T}_h^{\lev{B}}(\cdot \mid x_0, a_0))} + \sqrt{\frac{\log(HK^2/\delta)}{t}} \\
    \leq& c \sqrt{\frac{2^{d_\S \lev{B}}}{t}} + \sqrt{\frac{\log(HK^2/\delta)}{t}}.
\end{align*}

\noindent \textbf{Step Four}: Next we construct a coupling to show that $d_W(T_h^{\lev{B}}(\cdot \mid x_0, a_0), T_h(\cdot \mid x_0, a_0)) \leq \D(B)$.  For a coupling we define a family of distributions $\Gamma(\cdot,\cdot|x_0,a_0,\ell)$ parameterized by $x_0,a_0,\ell$ such that 
\[\Gamma(x_{snap}, x_{orig} | x_0,a_0,\ell) = T_h(x_{orig} \mid x_0, a_0) \sum_{A \in \S_{\ell}} \Ind{x_{snap} = \tilde{x}(A)} \Ind{ x_{orig} \in A}.\]
First notice that the marginals of these distributions match $T_h^{\ell}$ and $T_h$ respectively since:
\begin{align*}
    \int_\S \Gamma(x_{snap}, x \mid x_0, a_0, \ell) dx = \sum_{A \in \S_{\ell}} \Ind{x_{snap} = \tilde{x}(A)} \int_A T_h(x \mid x_0, a_0) dx = T^{\ell}_h(x_{snap} \mid x_0, a_0)
\end{align*}
and
\begin{align*}
\int_\S \Gamma(x, x_{orig} \mid x_0, a_0, \ell) dx & = \sum_{A \in \S_\ell} \Gamma(\tilde{x}(A), x_{orig} \mid x_0, a_0, \ell)
 = T_h(x_{orig} \mid x_0, a_0).
\end{align*}

Using this coupling $\Gamma$ it follows by definition of Wasserstein distance that
\begin{align*}
    d_W(T_h^{\ell}(\cdot \mid x_0, a_0), T_h(\cdot \mid x_0, a_0)) & \leq \E_{X_{snap}, X_{orig} \sim \Gamma(\cdot | x_0, a_0, \ell(B))}[\D_\S(X_{snap}, X_{orig})] \\
    &\leq \D(B)
\end{align*}
where we used that $X_{snap}$ and $X_{orig}$ have distance bounded by $\D_\S(A)$ for some $A \in \dyad{\lev{B}}$, and by construction of $\dyad{\lev{B}}$, $\D_{\S}(A) = \D(B)$.

\noindent \textbf{Step Five}: Using the assumption that $T_h$ is Lipschitz and $(x_0, a_0)$ and $(x,a) \in B$ we have that

$$d_W(T_h(\cdot \mid x,a), T_h(\cdot \mid x_0, a_0)) \leq L_T \D((x,a),(x_0, a_0)) \leq L_T \D(B).$$

Putting all of the pieces together and a union bound over $H$, $K$, the possible values of the random variables $t$, and the number of balls $B \in \P_h^K$ we get that:
\begin{align*}
&d_W(\Tbar{h}{k}(\cdot \mid B), T_h(\cdot \mid x_0, a_0)) \\
\leq& d_W(\Tbar{h}{k}(\cdot \mid B), \tilde{T}_h^k(\cdot \mid B)) + d_W(\tilde{T}_h^k(\cdot \mid B), \tilde{T}_h(\cdot \mid x_0, a_0)) + d_W(\tilde{T}_h(\cdot \mid x_0, a_0), \tilde{T}_h^{\lev{B}}(\cdot \mid x_0, a_0)) \\
&+ d_W(\tilde{T}_h^{\lev{B}}(\cdot \mid x_0, a_0), T_h(\cdot \mid x_0, a_0)) + d_W(T_h(\cdot \mid x_0, a_0), T_h(\cdot \mid x,a)) \\
\leq& \frac{\sum_{B' \supseteq B} n_h^k(B') \D(B')}{\sum_{B' \supseteq B} n_h^k(B')} + \frac{\sum_{B' \supseteq B} L_T n_h^k(B') \D(B')}{\sum_{B' \supseteq B} n_h^k(B')} + \sqrt{\frac{8 \log(HK^2/\delta)}{\sum_{B' \supseteq B} n_h^k(B')}} \\
&+ c \sqrt{\frac{2^{d_\S \lev{B}}}{\sum_{B' \supseteq B} n_h^k(B')}} + \sqrt{\frac{\log(HK^2)}{\sum_{B' \supseteq B} n_h^k(B')}}  + 2\D(B) + L_T \D(B) \\
=& (1+L_T) \frac{\sum_{B' \supseteq B} n_h^k(B') \D(B')}{\sum_{B' \supseteq B} n_h^k(B')} + 4 \sqrt{\frac{\log(HK^2/\delta)}{\sum_{B' \supseteq B} n_h^k(B')}}  + (2+L_T) \D(B) + c \sqrt{\frac{2^{d_\S \lev{B}}}{\sum_{B' \supseteq B} n_h^k(B')}} \\
\leq& (5 L_T + 4) \D(B) + 4 \sqrt{\frac{\log(HK^2/\delta)}{\sum_{B' \supseteq B} n_h^k(B')}} 
 + c \sqrt{\frac{2^{d_\S \lev{B}}}{\sum_{B' \supseteq B} n_h^k(B')}} \text{ by Lemma~\ref{lem:sum_ancestors}}.
\end{align*}
which is $\frac{1}{L_V} \tbonus{h}{k}(B)$ as needed.
\end{proof}

\begin{proof}[Proof of~\cref{lem:countbound}]
The proof of both the inequalities follows from a direct application of~\cref{lem:LPbound}, after first rewriting the summation over balls in $\Pkh$ as a summation over active balls in $\Pkh[K]$.

\noindent \textbf{First Inequality}: First, observe that we can write
\begin{align*}
    \sum_{k=1}^K \frac{2^{\beta \ell(B_h^k)}}{\left(n_h^k(B_h^k)\right)^\alpha} & = \sum_{\ell \in \NN_0} \sum_{B: \ell(B) = \ell} \sum_{k=1}^K \Ind{B_h^k = B} \frac{2^{\beta \ell(B)}}{\left(n_h^k(B)\right)^\alpha}
\end{align*}
Now, in order to use~\cref{lem:LPbound}, we first need to rewrite the summation as over `active balls' in the terminal partition $\Pkh[K]$ (i.e., balls which are yet to be split). Expanding the above, we get
\begin{align*}
    \sum_{k=1}^K \frac{2^{\beta \ell(B_h^k)}}{\left(n_h^k(B_h^k)\right)^\alpha} 
    & = \sum_{\ell \in \NN_0} \sum_{B \in \P_h^K : \ell(B) = \ell} \sum_{B' \supseteq B} 2^{d(\ell(B') - \ell(B))} \sum_{k=1}^K \Ind{B_h^k = B'} \frac{2^{\beta \ell(B')}}{\left(n_h^k(B')\right)^\alpha} \\
    & \leq \sum_{\ell \in \NN_0} \sum_{B \in \P_h^K : \ell(B) = \ell} \sum_{B' \supseteq B} 2^{d(\ell(B') - \ell(B))} 2^{\beta \ell(B')}\sum_{j=1}^{\nplus{\ell(B')}} \frac{1}{j^\alpha} \\
    & \leq \frac{\phi^{1-\alpha}}{1 - \alpha} \sum_{\ell \in \NN_0} \sum_{B \in \P_h^K : \ell(B) = \ell} \sum_{B' \supseteq B}2^{d(\ell(B') - \ell(B))} 2^{\beta \ell(B')} 2^{\gamma \ell(B')(1 - \alpha)}.
\end{align*}
where we used the fact that once a ball has been partitioned it is no longer chosen by the algorithm and an integral approximation to the sum of $1 / j^{\alpha}$ for $\alpha \leq 1$.  Next, we plug in the levels to get
\begin{align*}
\sum_{k=1}^K \frac{2^{\beta \ell(B_h^k)}}{\left(n_h^k(B_h^k)\right)^\alpha} &\leq \frac{\phi^{1-\alpha}}{1 - \alpha} \sum_{\ell \in \NN_0} \sum_{B \in \P_h^K : \ell(B) = \ell} \sum_{j=0}^{\ell}2^{d(j - \ell)} 2^{\beta j} 2^{\gamma j(1 - \alpha)} \\
& = \frac{\phi^{1-\alpha}}{1 - \alpha} \sum_{\ell \in \NN_0} \sum_{B \in \P_h^K : \ell(B) = \ell} \frac{1}{2^{d \ell}}\sum_{j=0}^{\ell}2^{j(d+\beta + \gam(1 - \alpha))} \\
& \leq \frac{\phi^{1 - \alpha}}{(2^{d+\beta+\gam(1 - \alpha)} - 1)(1 - \alpha)} \sum_{\ell \in \NN_0} \sum_{B \in \P_h^K : \ell(B) = \ell} \frac{1}{2^{d \ell}} 2^{(\ell + 1)(d+\beta + \gam(1 - \alpha))} \\
& \leq \frac{2\phi^{1 - \alpha}}{(1 - \alpha)} \sum_{\ell \in \NN_0} \sum_{B \in \P_h^K : \ell(B) = \ell} 2^{\ell(\beta + \gam(1 - \alpha))}.
\end{align*}
We set $a_\ell = 2^{\ell(\beta + \gamma(1 - \alpha))}$.  Clearly we have that $a_\ell$ are increasing with respect to $\ell$.  Moreover,
\begin{align*}
    \frac{2a_{\ell+1}}{a_\ell} & = \frac{2 \cdot 2^{(\ell + 1)(\beta + \gam(1 - \alpha))}}{2^{(\ell )(\beta + \gam(1 - \alpha))}} = 2^{1 + \beta + \gam(1 - \alpha)}.
\end{align*}
Setting this quantity to be less than $\nplus{\ell}/\nplus{\ell - 1} = 2^{\gam}$ we require that
\begin{align*}
    2^{1 + \beta + \gam(1 - \alpha)} & \leq 2^{\gam} \Leftrightarrow 
    1 + \beta - \alpha \gam \leq 0
\end{align*}
Now we can apply~\cref{lem:LPbound} to get that
\begin{align*}
    \sum_{k=1}^{K} \frac{2^{\beta \ell(B_h^k)}}{\left(n_h^k(B_h^k)\right)^{\alpha}} & \leq \frac{2\phi^{1 - \alpha}}{(1 - \alpha)} 2^{d \ell^\star} a_{\ell^\star} \\
    & = \frac{2^{2(d+\beta+\gam(1 - \alpha))}\phi^{1 - \alpha}}{(1 - \alpha)} \left(\frac{K}{\phi}\right)^{\frac{d+\beta + \gam(1 - \alpha)}{d+\gam}} \\
    & = O\left( \phi^{\frac{-(d\alpha+\beta)}{d+\gam}} K^{\frac{d+(1-\alpha)\gam+\beta}{d+\gam}} \right).
\end{align*}

\noindent \textbf{Second Inequality}: As in the previous part, we can rewrite as the summation we have
\begin{align*}
    \sum_{k=1}^K \frac{\ell(B_h^k)^\beta}{\left(n_h^k(B_h^k)\right)^\alpha} & = \sum_{\ell \in \NN_0} \sum_{B: \ell(B) = \ell} \sum_{k=1}^K \Ind{B_h^k = B} \frac{\ell(B)^\beta}{\left(n_h^k(B)\right)^\alpha}.\\
    & \leq \sum_{\ell \in \NN_0} \sum_{B \in \P_h^K : \ell(B) = \ell} \sum_{B' \supseteq B} 2^{d(\ell(B') - \ell(B))} \ell(B')^\beta \sum_{j=1}^{\nplus{\ell(B')}} \frac{1}{j^\alpha} \\
    & \leq \sum_{\ell \in \NN_0} \sum_{B \in \P_h^K : \ell(B) = \ell} \sum_{B' \supseteq B} 2^{d(\ell(B') - \ell(B))} \ell(B')^\beta \frac{\nplus{\ell(B')}^{1-\alpha}}{1 - \alpha} \\
    & = \frac{\phi^{1-\alpha}}{1-\alpha} \sum_{\ell \in \NN_0} \sum_{B \in \P_h^K : \ell(B) = \ell} \sum_{B' \supseteq B} 2^{d(\ell(B') - \ell(B))} \ell(B')^\beta 2^{\ell(B')\gam(1 - \alpha)}
\end{align*}
As before, we plug in the levels to get
\begin{align*}
\sum_{k=1}^K \frac{\ell(B_h^k)^\beta}{\left(n_h^k(B_h^k)\right)^\alpha} & =    \frac{\phi^{1-\alpha}}{1-\alpha} \sum_{\ell \in \NN_0} \sum_{B \in \P_h^K : \ell(B) = \ell} \sum_{j=0}^\ell 2^{d(j - \ell)} j^\beta 2^{j\gam(1 - \alpha)}\\
& \leq \frac{\phi^{1-\alpha}}{1 - \alpha} \sum_{\ell \in \NN_0} \sum_{B \in \P_h^K : \ell(B) = \ell} \frac{\ell^{\beta}}{2^{d \ell}} \sum_{j=0}^{\ell} 2^{j(d+\gam(1 - \alpha))} \\
& \leq \frac{2\phi^{1-\alpha}}{(1 - \alpha)} \sum_{\ell \in \NN_0} \sum_{B \in \P_h^K : \ell(B) = \ell} \ell^{\beta} 2^{\ell\gam(1 - \alpha)}.
\end{align*}
We take the term $a_\ell = \ell^{\beta} 2^{\ell\gam(1 - \alpha)}$.  Clearly we have that $a_\ell$ are increasing with respect to $\ell$.  Moreover, 
\begin{align*}
    \frac{2a_{\ell + 1}}{a_\ell} & = \left(1 + \frac{1}{\ell}\right)^\beta 2^{1+\gam (1 - \alpha)}.
\end{align*}
We require that this term is less than $\nplus{\ell + 1}/\nplus{\ell} = 2^{\gam}$ for all $\ell\geq \ell^{\star}$ (see note after~\cref{lem:LPbound}). This yields the following sufficient condition (after dividing through by $2^{\gam}$)
\begin{align*}
    \left(1 + \frac{1}{\ell}\right)^\beta 2^{1- \alpha\gamma} \leq 1\frall \ell\geq\ell^{\star}
\end{align*}
or equivalently, $\alpha\gamma - \beta\log_2(1+1/\ell^{\star}) \geq 1$. Finally note that $\log_2(1+x)\leq x/\ln 2 \leq x$ for all $x\in[0,1]$. Thus, we get that a sufficient condition is that $\alpha\gamma - \beta/\ell^{\star} \geq 1$. Assuming this holds, we get by~\cref{lem:LPbound} that
\begin{align*}
    \sum_{k=1}^{K} \frac{ \ell(B_h^k)^{\beta}}{\left(n_h^k(B_h^k)\right)^{\alpha}} 
    & \leq \left(\frac{2\phi^{1 - \alpha}}{(1 - \alpha)}\right) 2^{d \ell^\star} a_{\ell^\star} \\
    & = \left(\frac{2\phi^{1 - \alpha}}{1 - \alpha}\right) 4^{d+\gam(1 - \alpha)} \left(\frac{K}{\phi}\right)^{\frac{d+\gam(1 - \alpha)}{d+\gam}} \left(\frac{\log_2(K / \phi)}{d+\gam} + 2\right)^\beta \\
    & = O\left( \phi^{\frac{-d\alpha}{d+\gam}} K^{\frac{d+(1-\alpha)\gam}{d+\gam}} \left(\log_2K\right)^{\beta}\right).
\end{align*}
\end{proof}

\begin{proof}[Proof of \cref{lemma:recursive}]
	We use the notation $B_{h'}^k$ to denote the active ball containing the point $(X_{h'}^k, A_{h'}^k)$.  Under this we have by the update rule on $\Vtilde{h'}{k}$ that for any $h' \geq h$
	\begin{align*}
	& \Expk{\Vtilde{h'}{k}(\S(\P_{h'}^{k},X_{h'}^k)) \mid X_h^k}{k-1} \leq \Expk{\Qhat{h'}{k}(B_{h'}^k) \mid X_h^k}{k-1} \\
	& = \Expk{\rbar{h'}{k}(B_{h'}^k) + \E_{x \sim \Tbar{h'}{k}(\cdot \mid B)}[\Vhat{h'+1}{k-1}(x)] + \rbonus{h'}{k}(B_{h'}^k) + \tbonus{h'}{k}(B_{h'}^k) \mid X_h^k}{k-1} \text{ (via update rule for \Qhat{h}{k})} \\
	& = \Expk{\rbar{h'}{k}(B_{h'}^k) - r_{h'}(X_{h'}^k, A_{h'}^k) + \rbonus{h'}{k}(B_{h'}^k) \mid X_h^k}{k-1} + \Expk{r_{h'}(X_{h'}^k, A_{h'}^k) \mid X_h^k}{k-1} \\
	&\quad + \Expk{\E_{x \sim \Tbar{h'}{k}(\cdot \mid B_{h'}^k)}[\Vhat{h'+1}{k-1}(x)] - \E_{x \sim T_h(\cdot \mid X_{h'}^k, A_{h'}^k)}[\Vhat{h'+1}{k-1}(x)] + \tbonus{h'}{k}(B_{h'}^k) \mid X_h^k}{k-1} \\
	&\quad + \Expk{\E_{x \sim T_{h'}(\cdot \mid X_{h'}^k, A_{h'}^k)}[\Vhat{h'+1}{k-1}(x)] \mid X_h^k}{k-1} \\
	& = \Expk{\rbar{h'}{k}(B_{h'}^k) - r_h(X_{h'}^k, A_{h'}^k) + \rbonus{h'}{k}(B_{h'}^k) \mid X_h^k}{k-1} + \Expk{r_h(X_{h'}^k, A_{h'}^k) \mid X_h^k}{k-1} \\
	&\quad + \Expk{\E_{x \sim \Tbar{h'}{k}(\cdot \mid B_{h'}^k)}[\Vhat{h'+1}{k-1}(x)] - \E_{x \sim T_h(\cdot \mid X_{h'}^k, A_{h'}^k)}[\Vhat{h'+1}{k-1}(x)] + \tbonus{h'}{k}(B_{h'}^k) \mid X_h^k}{k-1} \\
	&\quad + \Expk{\Vhat{h'+1}{k-1}(X_{h'+1}^k) \mid X_h^k}{k-1} \hspace{4cm}  \text{( as } X_{h'+1}^k \sim T_h(\cdot \mid X_{h'}^k, A_{h'}^k))\\
	& \leq \Expk{\rbar{h'}{k}(B_{h'}^k) - r_h(X_{h'}^k, A_{h'}^k) + \rbonus{h'}{k}(B_{h'}^k) \mid X_h^k}{k-1} + \Expk{r_h(X_{h'}^k, A_{h'}^k) \mid X_h^k}{k-1} \\
	&\quad + \Expk{\E_{x \sim \Tbar{h'}{k}(\cdot \mid B_{h'}^k)}[\Vhat{h'+1}{k-1}(x)] - \E_{x \sim T_{h'}(\cdot \mid X_{h'}^k, a_{h'}^k)}[\Vhat{h'+1}{k-1}(x)] + \tbonus{h'}{k}(B_{h'}^k) \mid X_h^k}{k-1} \\
	&\quad + \Expk{\Vtilde{h'+1}{k-1}(\S(\P_{h'+1}^{k-1},X_{h'+1}^{k}))+ L_V \D(B_{h'+1}^k) \mid X_h^k}{k-1} \hspace{3.2cm}  \text{ (via update rule for $\Vhat{h}{k}$)}
	\end{align*}
	
	Taking this inequality and summing from $h' = h$ up until $H$ we find that $\sum_{h'=h}^H \Expk{r_{h'}(X_{h'}^k, A_{h'}^k) \mid X_h^k}{k-1} = V_{h}^{\pi^k}(X_h^k)$. Moreover, by changing the index in the sum and using the fact that $V_{H+1} = 0$, it follows that
	\begin{align*}
	\sum_{h'=h}^H \Expk{\Vtilde{h'+1}{k-1}(\S(\P_{h'+1}^{k-1},X_{h'+1}^k)) \mid X_h^k}{k-1} = \sum_{h'=h}^H \Expk{\Vtilde{h'}{k-1}(\S(\P_{h'}^{k-1},X_{h'}^k)) \mid X_h^k}{k-1} - \Vtilde{h}{k-1}(\S(\P_h^{k-1},X_h^k)).
	\end{align*}
	Rearranging the inequalities gives the desired results.
\end{proof}

\begin{proof}[Proof of \cref{lemma:regret_decomposition}]
	We condition on the good events from~\cref{lemma:reward_confidence,lemma:transition_confidence} by taking $\delta = 1 / HK$.  Using the definition of regret and the law of total expectation we have that:
	\begin{align*}
	\Exp{R(K)} & = \Exp{\sum_{k=1}^K V_1^\star(X_1^k) - V_1^{\pi^k}(X_1^k)} \\
	& \lesssim \Exp{\sum_{k=1}^K \Vhat{1}{k-1}(X_1^k) - V_1^{\pi^k}(X_1^k)} \hspace{2.1cm}\text{(via the optimism principle, ~\cref{lemma:optimism})}\\
	& \lesssim \Exp{\sum_{k=1}^K \Vtilde{1}{k-1}(\S(\P_1^{k-1}, X_1^k) - V_1^{\pi^k}(X_1^k) + L_V \D_\S(\S(\P_1^{k-1}, X_1^k))} \;\;\text{(update rule for } \Vtilde{h}{k})\\
	& \lesssim \Exp{\sum_{k=1}^K \Vtilde{1}{k-1}(\S(\P_1^{k-1}, X_1^k) - V_1^{\pi^k}(X_1^k) + L_V \D(B_1^k)} \\
	\end{align*}
	Next, define $\Expk{\cdot}{k-1}\triangleq\Exp{\cdot \mid\F_{k-1}}$. 
	Now using~\cref{lemma:recursive}, and the tower rule for conditional expectations, we get
	\begin{align*}
	\Exp{R(K)} & \lesssim \Exp{\sum_{k=1}^K \Expk{\Vtilde{1}{k-1}(\S(\P_1^{k-1}, X_1^k) - V_1^{\pi^k}(X_1^k)}{k-1} + L_V \D(B_1^k)}\\
	& \lesssim \sum_{k=1}^K \sum_{h=1}^H \Exp{\Expk{\Vtilde{h}{k-1}(\S(\P_h^{k-1},X_h^{k})) - \Vtilde{h}{k}(\S(\P_h^{k},X_h^k))}{k-1}}\\
	& + \sum_{h=1}^H \sum_{k=1}^K \Exp{\Expk{\rbar{h}{k}(B_h^k) - r_h(X_h^k, A_h^k) + \rbonus{h}{k}(B_h^k)}{k-1}} \\
	& + \sum_{h=1}^H \sum_{k=1}^K \Exp{\Expk{\E_{x \sim \Tbar{h}{k}(\cdot \mid B_h^k)}[\Vhat{h+1}{k-1}(x)] - \E_{x \sim T_h(\cdot \mid X_h^k, A_h^k)}[\Vhat{h+1}{k-1}(x)] + \tbonus{h}{k}(B_h^k)}{k-1}} \\
	& + \sum_{k=1}^K \sum_{h=1}^H L_V \Exp{\D(B_h^k)}\\
	& = \sum_{k=1}^K \sum_{h=1}^H \Exp{\Vtilde{h}{k-1}(\S(\P_h^{k-1},X_h^{k})) - \Vtilde{h}{k}(\S(\P_h^{k},X_h^k))} \\
	& + \sum_{h=1}^H \sum_{k=1}^K \Exp{2\rbonus{h}{k}(B_h^k)} + \sum_{h=1}^H \sum_{k=1}^K \Exp{2\tbonus{h}{k}(B_h^k)} + \sum_{k=1}^K \sum_{h=1}^H L_V \Exp{\D(B_h^k)}
	\end{align*}
	where in the last line we used the definition of the good event.
\end{proof}

\end{document}